\documentclass[11pt]{article}
\usepackage{psfig}
\usepackage{amsmath,amsthm}
\usepackage{geometry}                
\usepackage{graphicx}
\usepackage{amssymb}
\usepackage{epstopdf}
\usepackage{mathtools}
\usepackage{framed}
\usepackage[utf8]{inputenc}
\usepackage{algorithm}
\usepackage{algpseudocode}
\usepackage{subcaption}
\usepackage{booktabs}
\usepackage{bbm}
\usepackage{sidecap}

\usepackage{multirow}

\DeclareMathOperator*{\argmin}{\arg\!\min}
\DeclareMathOperator*{\argmax}{\arg\!\max}
\DeclareMathOperator*{\softmin}{\mathrm{soft}\!\min}

\newcommand{\widesim}[2][1.5]{
  \mathrel{\overset{#2}{\scalebox{#1}[1]{$\sim$}}}
}

\usepackage{epsfig}
\usepackage{url}

\usepackage{lipsum}

\newcommand\blfootnote[1]{%
  \begingroup
  \renewcommand\thefootnote{}\footnote{#1}%
  \addtocounter{footnote}{-1}%
  \endgroup
}

\makeatletter
\newcommand*{\skipnumber}[2][1]{%
   {\renewcommand*{\alglinenumber}[1]{}\State #2}%
   \addtocounter{ALG@line}{-#1}}
\makeatother

\begin{document}

\fontsize{11}{13pt}\selectfont

\title{Robustness to Adversarial Perturbations\\in Learning from Incomplete Data}

\author{
Amir~Najafi$^*$
\and
Shin-ichi~Maeda$^\dagger$
\and
Masanori~Koyama$^\dagger$
\and
Takeru~Miyato$^\dagger$
}

\date{}
\maketitle

\newtheorem{thm}{Theorem}[section]
\newtheorem{thm2}{Theorem}
\newtheorem{corl}{Corollary}
\newtheorem{note}[thm2]{Note}
\newtheorem{lemma}{Lemma}
\newtheorem{definition}{Definition}
\newtheorem{remark}{Remark}
\newtheorem{claim}{Claim}

\vspace*{-5mm}
\begin{center}
*~Computer Engineering Department\\
Sharif University of Technology, Tehran, Iran
\\[2mm]
$\dagger$~Preferred Networks Inc., Tokyo, Japan
\end{center}
\vspace*{8mm}

\begin{abstract}
 What is the role of unlabeled data in an inference problem, when the presumed underlying distribution is adversarially perturbed? To provide a concrete answer to this question, this paper unifies two major learning frameworks: Semi-Supervised Learning (SSL) and Distributionally Robust Learning (DRL). We develop a generalization theory for our framework based on a number of novel complexity measures, such as an adversarial extension of Rademacher complexity and its semi-supervised analogue. Moreover, our analysis is able to quantify the role of unlabeled data in the generalization under a more general condition compared to the existing theoretical works in SSL. Based on our framework, we also present a hybrid of DRL and EM algorithms that has a guaranteed convergence rate. When implemented with deep neural networks, our method shows a comparable performance to those of the state-of-the-art on a number of real-world benchmark datasets.
\blfootnote{E-mails: \texttt{najafy@ce.sharif.edu}, \texttt{\{ichi,masomatics,miyato\}@preferred.jp} .}
\end{abstract}


\section{Introduction}
\label{sec:intro}

Robustness to adversarial perturbations has become an essential feature in the design of modern classifiers ---in particular, of deep neural networks. This phenomenon originates from several empirical observations, such as \cite{szegedy2013intriguing} and \cite{nguyen2015deep}, which show deep networks are vulnerable to adversarial attacks in the input space. So far, plenty of novel methodologies have been introduced to compensate for this shortcoming. Adversarial Training (AT) \cite{43405}, Virtual AT \cite{miyato2018virtual} or Distillation \cite{papernot2016distillation} are just examples of some promising methods in this area. The majority of these approaches seek an effective defense against a {\it {point-wise adversary}}, who shifts input data-points toward adversarial directions, in a separate manner. However, as shown by \cite{staib2017distributionally}, a {\it {distributional adversary}} who can shift the data distribution instead of the input data-points is provably more detrimental to learning. This suggests that one can greatly improve the robustness of a classifier by improving its defense against a distributional adversary rather than a point-wise one. 
This motivation has led to the development of Distributionally Robust Learning (DRL) \cite{ben2013robust}, which has attracted intensive research interest over the last few years \cite{shafieezadeh2015distributionally,sinha2018certifying,hu2018does,esfahani2017data}. 

Despite of all the advancements in supervised or unsupervised DRL, the amount of researches tackling this problem from a semi-supervised angle is slim to none \cite{blanchet2017semi}. Motivated by this fact, we set out to propose a distributionally robust method that can handle Semi-Supervised Learning (SSL) scenarios. Our proposed method is an extension of self-learning \cite{grandvalet2005semi,zhu2006semi,chapelle2009semi}, and can cope with all existing learning frameworks, such as neural networks. Intuitively, we first try to infer soft-labels for the unlabeled data, and then search for suitable classification rules that demonstrate low sensitivity to perturbation around these soft-label distributions.

Parts of this paper can be considered as a semi-supervised extension of the general supervised DRL developed in \cite{sinha2018certifying}. Computational complexity of our method, for a moderate label-set size, is only slightly above those of its fully-supervised rivals. To optimize our model, we design a Stochastic Gradient Descent (SGD)-based algorithm with a theoretically-guaranteed convergence rate. In order to address the generalization of our framework, we introduce a set of novel complexity measures such as {\it {Adversarial Rademacher Complexity}} and {\it {Minimal Supervision Ratio}} (MSR), each of which are defined w.r.t. the hypothesis set and probability distribution that underlies input data-points. As long as the ratio of the labeled samples in a dataset (supervision ratio) exceeds MSR, true adversarial risk can be bounded. Also, one can arbitrarily decrease MSR by tuning the model parameters at the cost of increasing the generalization bound; This means our theoretical guarantees hold for all semi-supervised scenarios. We summarize the theoretical contribution of our work in Table \ref{tab:summary}.


We have also investigated the applicability of our method, denoted by SSDRL, via extensive computer experiments on datasets such as MNIST \cite{lecun1998gradient}, SVHN \cite{netzer2011reading}, and CIFAR-10 \cite{krizhevsky2009learning}. When implemented with deep neural networks, SSDRL outperforms rivals such as Pseudo-Labeling (PL) \cite{lee2013pseudo} and the supervised DRL in \cite{sinha2018certifying} (simply denoted as DRL) on all the above-mentioned datasets. In addition, SSDRL demonstrates a comparable performance to that of Virtual Adversarial Training (VAT) \cite{miyato2018virtual} on MNIST and CIFAR-10, while outperforms VAT on SVHN.

The rest of the paper is organized as follows: Section \ref{sec:notations} specifies the notations, and Section \ref{sec:priorworks} reviews the related works. The basic idea behind the proposed method is outlined in Section \ref{sec:proposed:main}, parameter optimization is described in Section \ref{sec:proposed:paramoptim} and generalization is analyzed in Section \ref{sec:proposed:general}. Section \ref{sec:exp} is devoted to experimental results. Finally, Section \ref{sec:future} concludes the paper.

\begin{SCtable}[0.73][t]
  \centering
		\caption{\label{tab:summary}Comparison between the proposed method (SSDRL) and some existing frameworks: DRL of \cite{sinha2018certifying}, Pseudo Labeling (PL) \cite{lee2013pseudo}, and Virtual Adversarial Training (VAT) \cite{miyato2018virtual}.} 
		\begin{tabular}{lcccc}
			\toprule
		     & {DRL} & {PL} & {VAT} & {SSDRL} 
		     \\
            \midrule
            Generalization Bound & \multirow{1}{*}{$\checkmark$} & \multirow{1}{*}{$\times$} & \multirow{1}{*}{$\times$} & \multirow{1}{*}{$\checkmark$}\\
             \midrule
            Convergence Guarantee & \multirow{1}{*}{$\checkmark$} & \multirow{1}{*}{$\times$} & \multirow{1}{*}{$\times$} & \multirow{1}{*}{$\checkmark$}\\
            \midrule
            Adversarial Robustness & \multirow{1}{*}{$\checkmark$} & \multirow{1}{*}{$\times$} & \multirow{1}{*}{$\checkmark$} & \multirow{1}{*}{$\checkmark$}\\
            \midrule
            Semi-Supervised Learning & \multirow{1}{*}{$\times$} & \multirow{1}{*}{$\checkmark$} & \multirow{1}{*}{$\checkmark$} & \multirow{1}{*}{$\checkmark$}\\
			\bottomrule
		\end{tabular}
\end{SCtable}


\subsection{Notations}
\label{sec:notations}

We extend the notations used in \cite{sinha2018certifying}. Assume $\mathcal{Z}$ to be an input space, $\Theta$ to be a parameter set, and $\ell:\mathcal{Z}\times\Theta\rightarrow\mathbb{R}$ a corresponding parametric loss function. Observation space $\mathcal{Z}$ can either be the feature space $\mathcal{X}$ in unsupervised scenarios, or the space of feature-label pairs, i.e., $\mathcal{Z}\triangleq\mathcal{X}\times\mathcal{Y}$, where $\mathcal{Y}$ denotes the set of labels. For simplicity, we only consider finite label-sets. By $M\left(\mathcal{Z}\right)$, we mean the set of all probability measures supported on $\mathcal{Z}$. Assume $c:\mathcal{Z}\times\mathcal{Z}\rightarrow\left[0,+\infty\right)$ to be a non-negative and lower semi-continuous function, where $c\left(\boldsymbol{z},\boldsymbol{z}\right)=0$ for all $\boldsymbol{z}\in\mathcal{Z}$. We occasionally refer to $c$ as {\it {transportation cost}}. The following definition formulates the {\it {Wasserstein}} distance $W_c\left(P,Q\right)$ between two distributions $P,Q\in M\left(\mathcal{Z}\right)$, w.r.t. $c$ \cite{shafieezadeh2015distributionally}:
\\[-3mm]
\begin{definition}[Wasserstein distance]
\label{def:wasserstein}
The {\it {Wasserstein}} distance between two distributions $P$ and $Q$ in $M\left(\mathcal{Z}\right)$, with respect to cost $c$ is defined as:
\begin{align}
\label{eq:WasserDistDef}
W_c\left(P,Q\right)&\triangleq \inf_{\mu\in M\left(\mathcal{Z}^2\right)}\int c\left(\boldsymbol{z},\boldsymbol{z}'\right)\mathrm{d}\mu\left(\boldsymbol{z},\boldsymbol{z}'\right)
\\
\mathrm{subject~to}~&\quad
\mu\left(\cdot,\mathcal{Z}\right)=P~,~\mu\left(\mathcal{Z},\cdot\right)=Q,
\nonumber
\end{align}
where $M\left(\mathcal{Z}^2\right)$ represents the set of all couplings between any two random variables supported on $\mathcal{Z}$. Also, $\mu\left(\mathcal{Z},\cdot\right)$ and $\mu\left(\cdot,\mathcal{Z}\right)$ denote the marginals of $\mu$ taken w.r.t. the first and second variables, respectively.
\end{definition}
$W_c\left(P,Q\right)$ measures the minimal cost of moving $P$ to $Q$, where the cost of moving one unit of mass from $\boldsymbol{z}$ to $\boldsymbol{z}'$ is given by $c\left(\boldsymbol{z},\boldsymbol{z}'\right)$. Also, for $\epsilon\ge0$ and an arbitrary distribution $Q\in M\left(\mathcal{Z}\right)$, we define an $\epsilon$-ambiguity set (or a {\it {Wasserstein $\epsilon$-ball}}) as 
\begin{equation}
\mathcal{B}_{\epsilon}\left(Q\right)
\triangleq
\left\{
P\in M\left(\mathcal{Z}\right)\vert~W_c\left(P,Q\right)\leq \epsilon
\right\}.
\end{equation}
Training dataset is shown by $\boldsymbol{D}\triangleq\left\{\boldsymbol{Z}_1,\ldots,\boldsymbol{Z}_n\right\}$, with samples being drawn i.i.d. from a fixed (and unknown) distribution $P_0\in M\left(\mathcal{Z}\right)$, where $n$ is the dataset size. For a dataset $\boldsymbol{D}$, let $\hat{\mathbb{P}}_{\boldsymbol{D}}\in M\left(\mathcal{Z}\right)$ be the following empirical measure:
\begin{equation}
\hat{\mathbb{P}}_{\boldsymbol{D}}\triangleq
\frac{1}{n}\sum_{i=1}^{n}\delta_{\boldsymbol{Z}_i},
\end{equation}
where $\delta_{\boldsymbol{z}}$ denotes the Dirac delta function at point $\boldsymbol{z}\in\mathcal{Z}$. Accordingly, $\mathbb{E}$ and $\hat{\mathbb{E}}_{\boldsymbol{D}}$ represent the statistical and empirical expectation operators, respectively. For a distribution $P\in M\left(\mathcal{X}\times\mathcal{Y}\right)$, $P_{\boldsymbol{X}}\left(\cdot\right)\triangleq\sum_{y\in\mathcal{Y}}P\left(\cdot,y\right)$ denotes the marginal distribution over $\mathcal{X}$, and $P_{\vert\boldsymbol{X}}\in M\left(\mathcal{Y}\right)$ is the conditional distributions over labels given feature vector $\boldsymbol{X}\in\mathcal{X}$. For the sake of simplicity in notations, for $\boldsymbol{Z}=\left(\boldsymbol{X},y\right)\in\mathcal{X}\times\mathcal{Y}$ and a function $f$, the notations $f\left(\boldsymbol{Z}\right)$ and $f\left(\boldsymbol{X},y\right)$ have been used, interchangeably.
\subsection{Background and Related Works}
\label{sec:priorworks}

DRL attempts to minimize a worst-case risk against an adversary. The adversary has a limited budget to alter the data distribution $Q\in M\left(\mathcal{Z}\right)$, in order to inflict the maximum possible damage. Here, $Q$ can either be the true measure $P_0$ or the empirical one $\hat{\mathbb{P}}_{\boldsymbol{D}}$. The mentioned learning scenario can be modeled by a game between a learner and an adversary whose stationary point is the solution of a minimax problem \cite{hu2018does}. Mathematically speaking, DRL can be formulated as \cite{shafieezadeh2015distributionally,esfahani2017data}:
\begin{equation}
\inf_{\theta\in\Theta}\sup_{P\in\mathcal{B}_{\epsilon}\left(Q\right)}\mathbb{E}_P\left\{
\ell\left(\boldsymbol{Z};\theta\right)
\right\}.
\label{eq:DROorigin}
\end{equation}
{\it {Wasserstein}} metric has been widely used to quantify the strength of adversarial attacks \cite{shafieezadeh2015distributionally,sinha2018certifying,esfahani2017data,blanchet2017semi}, thanks to (i) its fundamental relations to adversarial robustness \cite{cranko2018monge} and  (ii) its mathematically well-studied dual-form properties \cite{esfahani2017data}.
In \cite{shafieezadeh2015distributionally}, authors have reformulated DRL into a convex program for the particular case of logistic regression. Convergence and generalization analysis of DRL have been addressed in \cite{sinha2018certifying} in a general context, while the finding of a proper ambiguity set size, i.e. $\epsilon$, has been tackled in \cite{duchi2016statistics}. An interesting analysis on DRL methods with $f$-divergences is given in \cite{hu2018does}. Sample complexity of DRL has been reviewed by \cite{schmidt2018adversarially} and \cite{cullina2018pac}. We conjecture that there might be close relations between our complexity analysis in Section \ref{sec:proposed:general} and some of the results in the latter studies. However, a careful investigation regarding this issue goes beyond the scope of this paper. 

On the other hand, recent abundance of unlabeled data has made SSL methods widely popular \cite{miyato2018virtual,dai2017good}. See \cite{zhu2006semi} for a comprehensive review on classical SSL approaches. Many robust SSL algorithms have been proposed so far \cite{balsubramani2015scalable,yan2016robust}, however, their notion of robustness is mostly different from the one considered in this paper. In \cite{loog2016contrastive}, author has proposed a pessimistic SSL approach which is guaranteed to have a better, or at least equal, performance when it takes unlabeled data into account.
We show that a special case of our method reduces to an adversarial extension of \cite{loog2016contrastive}. From a theoretical perspective, guarantees on the generalization of SSL can only be provided under certain assumptions on the choice of hypothesis set and the true data distribution \cite{zhu2006semi,chapelle2009semi,singh2009unlabeled}. 
For example, in \cite{chapelle2009semi} a {\it {compatibility}} function is introduced to restrict the relation between a model set and an input data distribution. Also, author of  \cite{rigollet2007generalization} has theoretically analyzed SSL under the so-called {\it {cluster assumption}}, in order to establish an improvement guarantee for a situation where unlabeled data had been experimentally shown to be helpful. 
The fundamental reason behind such assumptions is that lack of any prior knowledge about the information-theoretic relations between a feature vector and its corresponding label, simply makes unlabeled data to be useless for classification. Not to mention that improper assumptions about the relation of feature-label pairs, for example by employing unsuitable hypothesis sets, could actually degrade the classification accuracy in semi-supervised scenarios.
In Section \ref{sec:proposed:general}, we propose a novel {\it {compatibility function}} that works under a general setting and enables us to theoretically establish a generalization bound for our method.

Finally, the only work prior to this paper that also falls in the cross section of DRL and SSL is \cite{blanchet2017semi}. However, the method in \cite{blanchet2017semi} severely restricts the support of adversarially-altered distributions, so that the adversary is left to choose from a set of delta-spikes over only labeled and augmented unlabeled samples. Thus, one cannot expect a considerable improvement in the distributional robustness in this case, because it does not let the adversary to freely perturb training data-points toward arbitrary directions.  


\section{Proposed Framework}
\label{sec:proposed}
From now on, let us assume $\mathcal{Z}\triangleq\mathcal{X}\times\mathcal{Y}$. In a semi-supervised configuration, dataset $\boldsymbol{D}$ consists of two non-overlapping parts: $\boldsymbol{D}_{\mathrm{l}}$ (labeled) and $\boldsymbol{D}_{\mathrm{ul}}$ (unlabeled). It should be noted that learner can only observe partial information about each sample in $\boldsymbol{D}_{\mathrm{ul}}$, namely, its feature vector. Let us denote $\mathcal{I}_{\mathrm{l}}$ and $\mathcal{I}_{\mathrm{ul}}$ as the index sets corresponding to the labeled and unlabeled data points, respectively. Thus, we have $\boldsymbol{D}_{\mathrm{l}}=\left\{\left(\boldsymbol{X}_i,y_i\right)\vert~i\in\mathcal{I}_{\mathrm{l}}\right\}$, and $\boldsymbol{D}_{\mathrm{ul}}=\left\{\boldsymbol{X}_i\vert~i\in\mathcal{I}_{\mathrm{ul}}\right\}$. The unknown labels of the samples in $\boldsymbol{D}_{\mathrm{ul}}$ can be thought as a set of corresponding random variables supported on $\mathcal{Y}$. DRL in \eqref{eq:DROorigin} cannot readily extend to this {\it {partially labeled}} setting, since it needs complete access to all the feature-label pairs in $\boldsymbol{D}$. In order to bypass this barrier, we need to somehow address the additional stochasticity that originates from incorporating unlabeled data in the learning procedure. The following definition can be helpful for this aim:
\\[-3mm]
\begin{definition}
The consistent set of probability distributions $\hat{\mathcal{P}}\left(\boldsymbol{D}\right)\subseteq M\left(\mathcal{Z}\right)$ with respect to a partially-labeled dataset $\boldsymbol{D}=\boldsymbol{D}_{\mathrm{l}}\cup\boldsymbol{D}_{\mathrm{ul}}$ is defined as
\begin{gather}
\hat{\mathcal{P}}\left(\boldsymbol{D}\right)\triangleq
\left\{
\left(\frac{n_{\mathrm{l}}}{n}\right)\hat{\mathbb{P}}_{\boldsymbol{D}_{\mathrm{l}}}
+
\left(\frac{n_{\mathrm{ul}}}{n}\right)\hat{\mathbb{P}}_{\boldsymbol{D}_{\mathrm{ul}}}\cdot\Omega
~\big\vert
~
\Omega\in M^{\mathcal{X}}\left(\mathcal{Y} \right)
\right\},
\nonumber
\end{gather}
where $n_{\mathrm{l}}$ and $n_{\mathrm{ul}}$ are the sizes of $\boldsymbol{D}_{\mathrm{l}}$ and $\boldsymbol{D}_{\mathrm{ul}}$, respectively. Also, $M^{\mathcal{X}}\left(\mathcal{Y} \right)$ denotes the set of all conditional distributions over $\mathcal{Y}$, given values in $\mathcal{X}$.  
\label{def:consistent}
\end{definition}
Distributions in $\hat{\mathcal{P}}\left(\boldsymbol{D}\right)$ have delta-spikes over supervised samples in $\boldsymbol{D}_{\mathrm{l}}$. However, for samples in $\boldsymbol{D}_{\mathrm{ul}}$, singularity is only over the feature vectors while conditional distributions over their corresponding labels are free to be anything, i.e., samples can even have soft-labels. Note that the empirical measure which corresponds to the true complete dataset is also somewhere inside $\hat{\mathcal{P}}\left(\boldsymbol{D}\right)$. Our aim is to choose a suitable measure from $\hat{\mathcal{P}}\left(\boldsymbol{D}\right)$, and then use it for \eqref{eq:DROorigin}.


\subsection{Self-Learning: Optimism vs. Pessimism}
\label{sec:proposed:main}

We focus on a well-known family of SSL approaches, called self-learning  \cite{grandvalet2005semi,amini2002semi}, and then combine it to the framework of DRL. Methods that are built upon self-learning, such as Expectation-Maximization (EM) algorithm \cite{basu2002semi}, aim to transfer the knowledge from labeled samples to unlabeled ones through what is called \textit{pseudo-labeling}. More precisely, a learner is (repeatedly) trained on the supervised portion of a dataset, and then employs its learned rules to assign pseudo-labels to the remaining unlabeled part. This procedure can assign either hard or soft labels to unlabeled features. This way, all these artificially-labeled unsupervised samples can also join in for the training of the learner in the final stages of learning. However, such methods are prone to over-fitting if the information flow from $\boldsymbol{D}_{\mathrm{l}}$ to $\boldsymbol{D}_{\mathrm{ul}}$ is not properly controlled. One way to overcome this issue is to use soft-labeling, which maintains a minimum level of uncertainty within the unlabeled data points. By combining the above arguments with the core idea of DRL in \eqref{eq:DROorigin}, we propose the following learning scheme:
\begin{gather}
\label{eq:main2}
\inf_{\theta\in\Theta}~\inf_{S\in\hat{\mathcal{P}}\left(\boldsymbol{D}\right)}
\left\{
\sup_{P\in\mathcal{B}_{\epsilon}\left(S\right)}
\mathbb{E}_P\left\{
\ell\left(\boldsymbol{X},y;\theta\right)
\right\}
+
\left(\frac{1-\eta}{\lambda}\right)
\hat{\mathbb{E}}_{\boldsymbol{D}_{\mathrm{ul}}}\left\{
\mathbb{H}\left(S_{\vert\boldsymbol{X}}\right)\right\}
\right\},
\end{gather}
where $\lambda$ is a user-defined parameter, $\eta\triangleq{n_{\mathrm{l}}}/{n}$ is called the {\it {supervision ratio}}, and $\mathbb{H}\left(\cdot\right)$ denotes the Shannon entropy. For now, let us assume $\lambda<0$.

Minimization over ${S\in\hat{\mathcal{P}}\left(\boldsymbol{D}\right)}$ acts as a {\it {knowledge transfer}} module and finds the optimal empirical distribution in $\hat{\mathcal{P}}\left(\boldsymbol{D}\right)$ for the model selection module, i.e. $\inf_{\theta\in\Theta}$. Again, note that distributions in $\hat{\mathcal{P}}\left(\boldsymbol{D}\right)$ differ from each other only in the way they assign labels, or soft-labels, to the unlabeled data. According to \eqref{eq:main2}, learner has obviously chosen to be {\it {optimistic}} w.r.t. the hypothesis set $\Theta$ and its corresponding loss function $\ell$. In other words, for any $\theta\in\Theta$, learner is instructed to pick the labels that are more likely to reduce the loss function $\ell\left(\cdot;\theta\right)$ for the unlabeled data. This strategy forms the core idea of self-learning. Note that a pessimistic strategy suggests the opposite, i.e. to pick the less likely labels (those with large loss values) and not to trust the loss function. At the end of this section, we explain more about the {\it {pessimistic}} learner.

The negative regularization term $\frac{1-\eta}{\lambda}
\hat{\mathbb{E}}_{\boldsymbol{D}_{\mathrm{ul}}}\left\{
\mathbb{H}\left(S_{\vert\boldsymbol{X}}\right)\right\}$ prevents hard decisions for labels and promotes soft-labeling by bounding the Shannon entropy of label-conditionals from below. A smaller $\left\vert\lambda\right\vert$ gives softer labels. In the extreme case, choosing $\lambda=-\infty$ ends up in an adversarial version of the self-training in \cite{zhu2006semi}. It should be noted that \eqref{eq:main2} considers all the data in $\boldsymbol{D}_{\mathrm{l}}\cup\boldsymbol{D}_{\mathrm{ul}}$ for the purpose of distributional robustness. In fact, the adversary has access to all the feature vectors in $\boldsymbol{D}$, which in turn forces the learner to experience adversarial attacks near all these points. This way, learner is instructed to show less sensitivity near all training data, just as one may expect from a semi-supervised DRL.

We show that \eqref{eq:main2} can be efficiently solved given that some smoothness conditions hold for $\ell$ and $c$. Before that, Theorem \ref{thm:sslDual} shows that the optimization corresponding to the knowledge transfer module has an analytic solution, which implies the computational cost of \eqref{eq:main2} is only slightly higher than those of its fully-supervised counterparts, such as \cite{sinha2018certifying}.
\\[-3mm]
\begin{definition}
\label{def:softmin}
For $\boldsymbol{q}\in\mathbb{R}^{\mathcal{Y}}$ and $\lambda\in\mathbb{R}\cup\left\{\pm\infty\right\}$, soft-minimum of $\boldsymbol{q}$ with respect to $\lambda$ is defined as
\begin{equation}
\softmin_{y\in\mathcal{Y}}^{\left(\lambda\right)}\left(\boldsymbol{q}\right)
\triangleq
\frac{1}{\lambda}\log\left(
\frac{1}{\left\vert\mathcal{Y}\right\vert}\sum_{y\in\mathcal{Y}}e^{\lambda q_y}
\right).
\end{equation}
\end{definition}
\begin{thm2}[Lagrangian-Relaxation]
Assume a continuous loss $\ell:\mathcal{Z}\times\Theta\rightarrow\mathbb{R}$ and continuous $c:\mathcal{Z}\times\mathcal{Z}\rightarrow\mathbb{R}_{\ge0}$, parameters $\epsilon\ge0$ and $\lambda\in\mathbb{R}\cup\left\{\pm\infty\right\}$, and a partially-labeled dataset $\boldsymbol{D}$ with size $n$. For $\gamma\ge0$, let us define the empirical {\it {Semi-Supervised Adversarial Risk (SSAR)}}, denoted by $\hat{R}_{\mathrm{SSAR}}\left(\theta;\boldsymbol{D}\right)$, as
\begin{equation}
\label{eq:SSLmainMin}
\hat{R}_{\mathrm{SSAR}}\left(\theta;\boldsymbol{D}\right)
~\triangleq~
\frac{1}{n}\sum_{i\in\mathcal{I}_{\mathrm{l}}}\phi_{\gamma}\left(\boldsymbol{X}_i,y_i;\theta\right)
+
\frac{1}{n}\sum_{i\in\mathcal{I}_{\mathrm{ul}}}
\softmin_{y\in\mathcal{Y}}^{\left(\lambda\right)}\left\{
\phi_{\gamma}\left(\boldsymbol{X}_i,y;\theta\right)
\right\} + \gamma\epsilon,
\end{equation}
where $\phi_{\gamma}\left(\boldsymbol{X},y;\theta\right)$, called adversarial loss, is defined as
\begin{equation}
\label{eq:phiGammaDef}
\phi_{\gamma}\left(\boldsymbol{X},y;\theta\right)
\triangleq
\sup_{\boldsymbol{z}'\in\mathcal{Z}}\ell\left(\boldsymbol{z}';\theta\right)
-\gamma c\left(\boldsymbol{z}',\left(\boldsymbol{X},y\right)\right).
\end{equation}
Let $\theta^*\in\Theta$ be a minimizer of \eqref{eq:main2} for a given set of parameters $\epsilon\ge0$ and $\lambda<0$.
Then, there exists $\gamma\ge0$ such that $\theta^*$ is also a minimizer of \eqref{eq:SSLmainMin} with the same corresponding parameters $\epsilon$ and $\lambda$.
\label{thm:sslDual}
\end{thm2}
Proof of Theorem \ref{thm:sslDual} is given in Appendix \ref{sec:appendix:thm}. Note that $\softmin$ equals to : (i) $\min$ operator for $\lambda=-\infty$, (ii) average for $\lambda=0$, and (iii) $\max$ operator for $\lambda=+\infty$. Also, $\epsilon$ and $\gamma$ are non-negative dual parameters and fixing either of them uniquely determines the other one. Due to this one-to-one relation, one can adjust $\gamma$ (for example via cross-validation), instead of $\epsilon$. See \cite{sinha2018certifying} for a similar discussion about this issue.

A more subtle look at \eqref{eq:SSLmainMin} shows that in the dual context of the proposed optimization problem, one is free to also consider positive values for $\lambda$. Choosing $\lambda>0$ promotes those labels that produce larger adversarial loss values for unlabeled data. In other words, the sign of $\lambda$ indicates  {\it {optimism}} ($\lambda\leq0$), or {\it {pessimism}} ($\lambda>0$) during the label assignment. The choice between optimism vs. pessimism depends on the {\it {compatibility}} of the model set $\Theta$ with the true distribution $P_0$. In Section \ref{sec:proposed:general}, we show that enabling $\lambda$ to take values in $\mathbb{R}$ rather than $\mathbb{R}^{-}$ is crucial for establishing a generalization bound for \ref{eq:SSLmainMin}. In other words, for {\it {very bad}} hypothesis sets w.r.t. a particular input distribution, one must choose to be pessimistic to be able to generalize well. To see some situations where pessimism in Semi-Supervised Learning can help, reader may refer to \cite{loog2016contrastive}.

\subsection{Numerical Optimization}
\label{sec:proposed:paramoptim}

We propose a numerical optimization scheme for solving \eqref{eq:SSLmainMin} or equivalently \eqref{eq:main2}, which has a convergence guarantee. A hurdle in applying SGD to \eqref{eq:SSLmainMin} is the fact that the value of $\phi_{\gamma}\left(\boldsymbol{z};\theta\right)$ is itself the output of a maximization problem. Also, the loss function $\ell$ is not necessarily convex w.r.t. $\theta$, e.g. neural networks, and hence achieving the global minimum of \eqref{eq:SSLmainMin} is not feasible in general. The former problem has already been solved in supervised DRL, as long as we focus on a sufficiently small $\epsilon$ \cite{sinha2018certifying}.
\\[-3mm]
\begin{lemma}
Consider the setting described in Theorem \ref{thm:sslDual}. Assume $\ell\left(\boldsymbol{z};\theta\right)$ is differentiable w.r.t. $\boldsymbol{z}$, and $\nabla_{\boldsymbol{z}}\ell\left(\cdot;\theta\right)$ is $L_{zz}$-Lipschitz all over $\mathcal{Z}\times\Theta$, for some $L_{zz}\ge0$. Also, assume transportation cost $c$ is $1$-strongly convex in its first argument. Then, if $\gamma>L_{zz}$, the program
\begin{equation}
\sup_{\boldsymbol{z}'\in\mathcal{Z}}~\ell\left(\boldsymbol{z}';\theta\right)
-\gamma c\left(\boldsymbol{z}',\left(\boldsymbol{X},y\right)\right)
\label{eq:innerMaxEq}
\end{equation}
becomes $\left(\gamma-L_{zz}\right)$-strongly concave for all $\left(\boldsymbol{X},y\right)\in\mathcal{Z}$.
\label{lemma:innerMAxConcave}
\end{lemma}
The proof of Lemma \ref{lemma:innerMAxConcave} is based on Taylor's expansion series. By using a modified version of Danskin's theorem for minimax problems \cite{bonnans2013perturbation}, and followed by additional smoothness conditions on $\ell$, an efficient computation of the gradient of $\hat{R}_{\mathrm{SSAR}}$ in \eqref{eq:SSLmainMin} w.r.t. $\theta\in\Theta$ is as follows: 
\\[-3mm]
\begin{lemma}
Assume loss function $\ell:\mathcal{Z}\times\Theta\rightarrow\mathbb{R}$, $c:\mathcal{Z}\times\mathcal{Z}\rightarrow\mathbb{R}_{\ge0}$ and $\gamma\ge0$, such that conditions in Lemma \ref{lemma:innerMAxConcave} hold all over $\mathcal{Z}\times\Theta$. Assume $\ell$ is differentiable w.r.t. $\theta$, and let $\boldsymbol{g}_{\theta}\left(\boldsymbol{z}\right)\triangleq \nabla_{\theta}
\ell\left(\boldsymbol{z};\theta\right)$. For a fixed $\theta\in\Theta$ and $i\in\mathcal{I}_{\mathrm{l}}$, define $\boldsymbol{z}^{*}_i\left(\theta\right)$ as the maximizer of \eqref{eq:innerMaxEq} for $\left(\boldsymbol{X}_i,y_i\right)$.
Similarly, let $\boldsymbol{z}^{*}_i\left(y;\theta\right)$ to represent the maximizer of 
\begin{equation}
\label{eq:SSLDROinnermax}
J_i\left(y;\theta\right)
\triangleq
\sup_{\boldsymbol{z}'\in\mathcal{Z}}~\ell\left(\boldsymbol{z}';\theta\right)-\gamma c\left(\boldsymbol{z}',\left(\boldsymbol{X}_i,y\right)\right)
,~~
y\in\mathcal{Y},i\in\mathcal{I}_{\mathrm{ul}}.
\end{equation} 
Then, the gradient of \eqref{eq:SSLmainMin} w.r.t. $\theta\in\Theta$ can be attained as
\begin{align}
\label{eq:finalSSLDROderiv}
\nabla_{\theta}
\hat{R}_{\mathrm{SSAR}}\left(\theta;\boldsymbol{D}\right)  
=
\frac{1}{n}\sum_{i\in\mathcal{I}_{\mathrm{l}}}\boldsymbol{g}_{\theta}\left(\boldsymbol{z}^{*}_i\left(\theta\right)\right)
+\frac{1}{n}\sum_{i\in\mathcal{I}_{\mathrm{ul}}}\sum_{y\in\mathcal{Y}}q(y;\theta)
\boldsymbol{g}_{\theta}\left(\boldsymbol{z}^{*}_i\left(y;\theta\right)\right), 
\end{align}
where $q(y;\theta) \triangleq \exp(\lambda J_i\left(y;\theta\right)) / \left(\sum_{y'\in\mathcal{Y}}\exp(\lambda J_i\left(y';\theta\right)) \right)$.
\label{lemma:SSLDROderiv}
\end{lemma}
\begin{algorithm}[t]
\caption{Stochastic Gradient Descent for SSDRL}
\label{alg:SSLDROsgd}
\begin{algorithmic}[1]
\State Inputs: $\boldsymbol{D}, \gamma, \lambda, \left(k\leq n, \delta, \alpha, T\right)$
\State Initialize $\theta_0\in\Theta$, and set $t\leftarrow 0$.
	\vspace*{2mm}
\For {$t=0\rightarrow T-1$}
	\State Randomly select index set $\mathcal{I}\subseteq\left[n\right]$ with size $k$.
	\For {$i\in\mathcal{I}_{\mathrm{l}}\cap\mathcal{I}$}
		\State Compute a $\delta$-approx of $\boldsymbol{z}^{*}_i\left(\theta_t\right)$ from Lemma \ref{lemma:SSLDROderiv}.
	\EndFor
	\For {$\left(i,y\right)\in\left(\mathcal{I}_{\mathrm{ul}}\cap\mathcal{I}\right)\times\mathcal{Y}$}
		\State Compute a $\delta$-approx of $\boldsymbol{z}^{*}_i\left(y;\theta_t\right)$ from \eqref{eq:SSLDROinnermax}						
	\EndFor
	\State Compute the sub-gradient of $\hat{R}_{\mathrm{SSAR}}\left(\theta;\boldsymbol{D}\right)$ from \eqref{eq:finalSSLDROderiv} at point $\theta=\theta_t$
	\vspace{-4.5mm}
	\skipnumber[2]{\State (using only samples in $\mathcal{I}$), and denote it with $\partial_{\theta}\hat{R}_{\mathrm{SSAR}}\left(\theta_t;\boldsymbol{D}\right)$.}
	\State $\theta_{t+1} \leftarrow \mathrm{Proj}_{\Theta}\left(\theta_t - \alpha \partial_{\theta}\hat{R}_{\mathrm{SSAR}}\left(\theta_t;\boldsymbol{D}\right)\right)$
\EndFor
\State Output: $\theta^*\leftarrow\theta_T$
\end{algorithmic}
\end{algorithm} 
Proof of Lemma \ref{lemma:SSLDROderiv} is included in that of Theorem \ref{thm:sgdConv} and can be found in Appendix \ref{sec:appendix:thm}. Given the formulation in Lemma \ref{lemma:SSLDROderiv}, one can simply apply the mini-batch SGD to solve for \eqref{eq:main2} via Algorithm \ref{alg:SSLDROsgd} (the semi-supervised extension of \cite{sinha2018certifying}). The set of constants such as the maximum iteration number $T\in\mathbb{N}$, $\delta$, $\alpha$ and mini-batch size $k\leq n$ are all user-defined. Due to the strong concavity of \eqref{eq:SSLDROinnermax} under the conditions of Lemma \ref{lemma:innerMAxConcave}, $\delta$ can be chosen arbitrarily small. Other parameters such as $\gamma$ and $\lambda$ should be adjusted via cross-validation. The computational complexity of Algorithm \ref{alg:SSLDROsgd} is no more than $\eta+\left\vert\mathcal{Y}\right\vert\left(1-\eta\right)$ times of that of \cite{sinha2018certifying}, where the latter can only handle supervised data\footnote{In scenarios where $\left\vert\mathcal{Y}\right\vert$ is very large, one can employ heuristic methods to reduce the set of {\it {possible labels}} for an unlabeled data sample and gain more efficiency at the expense of degradation in performance}. Note that Algorithm \ref{alg:SSLDROsgd} reduces to \cite{sinha2018certifying} in fully-supervised scenarios, and coincides with Pseudo-Labeling and EM algorithm when $\left(\gamma=\infty,\lambda=-\infty\right)$ and $\left(\gamma=\infty,\lambda=-1\right)$, respectively. The following theorem guarantees the convergence of Algorithm \ref{alg:SSLDROsgd} to a local minimizer of \eqref{eq:SSLmainMin}.
\\[-3mm]
\begin{thm2}
\label{thm:sgdConv}
Assume the loss function $\ell$, transportation cost $c$, $\gamma\ge0$ and $\left\vert\lambda\right\vert<\infty$ to satisfy the conditions of Lemma \ref{lemma:SSLDROderiv}. Also, assume $\ell$ is differentiable w.r.t. both parameters $\boldsymbol{z}$ and $\theta$, with Lipschitz gradients. Also, assume $\left\Vert\nabla_{\theta}\ell\left(\boldsymbol{z};\theta\right)\right\Vert_2\leq\sigma$, for some $\sigma\ge0$ all over $\mathcal{Z}\times\Theta$. Let $\theta_0\in\Theta$ to be an initial hypothesis, and denote $\theta^*\in\Theta$ as a local minimizer of \eqref{eq:main2} or \eqref{eq:SSLmainMin}. Assume the partially-labeled dataset $\boldsymbol{D}$ to include $n$ i.i.d. samples drawn from $P_0\in M\left(\mathcal{X}\times\mathcal{Y}\right)$. Also, let
$\Delta\hat{R}
\triangleq
\hat{R}_{\mathrm{SSAR}}\left(\theta_0;\boldsymbol{D}\right)-
\hat{R}_{\mathrm{SSAR}}\left(\theta^*;\boldsymbol{D}\right)$. Then, for the fixed step size $\alpha^*$
\begin{equation}
\alpha^*
\triangleq
\frac{1}{\sigma^2}
\sqrt{\frac{\Delta\hat{R}}
{T\left(\frac{B}{\sigma^2}
+\left(1-{\eta}\right)\left\vert\lambda\right\vert\left\vert\mathcal{Y}\right\vert\right)}},
\end{equation}
the outputs of Algorithm \ref{alg:SSLDROsgd} with parameter set $k=1$, $\delta>0$, $\alpha=\alpha^*$ after $T$ iterations, say $\theta_1,\ldots,\theta_T$, satisfy the following inequality:
\begin{gather}
\label{eq:sgdBound}
\frac{1}{T}\sum_{t=1}^{T}\mathbb{E}\left\{
\big\Vert
\nabla_{\theta}
\hat{R}_{\mathrm{SSAR}}\left(\theta_t;\boldsymbol{D}\right)
\big\Vert^2_2
\right\}
\leq
4\sigma^2\sqrt{\frac{\Delta\hat{R}}{T}
\left(\frac{B}{\sigma^2}
+\left(1-{\eta}\right)\left\vert\lambda\right\vert\left\vert\mathcal{Y}\right\vert\right)}
+C\delta,
\end{gather}
where constants $B$ and $C$ only depend on $\gamma$ and Lipschitz constants of $\ell$. Also, ${\eta}\triangleq n_{\mathrm{l}}/n$, and the expectation in \eqref{eq:sgdBound} is w.r.t. dataset $\boldsymbol{D}$ and the randomness of Algorithm \ref{alg:SSLDROsgd}.
\end{thm2}
The proof of Theorem \ref{thm:sgdConv} with explicit formulations of constants $B$ and $C$ are given in Appendix \ref{sec:appendix:thm}. Theorem \ref{thm:sgdConv} guarantees a convergence rate of $O\left(T^{-1/2}\right)$ 
for Algorithm \ref{alg:SSLDROsgd}, if one neglects $\delta$. Note that the presence of $\delta$ is necessary since one cannot find the exact maximizer of \eqref{eq:SSLDROinnermax} in finite steps. However, due to Lemma \ref{lemma:innerMAxConcave}, $\delta$ can be chosen infinitesimally small. According to Theorem \ref{thm:sgdConv}, choosing a very large $\left\vert\lambda\right\vert$ reduces the convergence rate, since the derivative of $\softmin^{\left(\lambda\right)}$ starts to diverge. In fact, the limiting cases of $\lambda=\pm\infty$ represent two hard-label strategies with combinatoric structures. Convergence rates for such methods are not well-studied in the literature \cite{zhu2006semi}. However, convergence guarantees (even without establishing a convergence rate) can still be useful for these limiting cases. Theorem \ref{thm:hardConv} (Appendix) guarantees the convergence of Algorithm \ref{alg:SSLDROsgd} in hard-decision regimes, i.e. $\lambda=\pm\infty$.

Another interesting question is: can we guarantee the convexity of \eqref{eq:main2} or \eqref{eq:SSLmainMin}, given that loss function $\ell$ is twice differentiable and {\it {strictly convex}} w.r.t. $\theta$?  Although the convexity of $\ell$ does not hold in many cases of interest, e.g. neural networks, a careful convexity analysis of our method is still important. Theorem \ref{corl:Convexity} (Appendix) provides a sufficient condition for convexity of \eqref{eq:SSLmainMin}, when $\ell$ is strictly convex. The condition requires $\lambda\ge\lambda_{\min}$ with $\lambda_{\min}\left(\ell,\gamma,\left\vert\mathcal{Y}\right\vert\right)$ being a negative function. Note that when $\lambda=-\infty$, the r.h.s. of \eqref{eq:SSLmainMin} equals to the minimum of a finite number of convex functions ---which is not necessarily convex. On the other hand, a non-negative value of $\lambda$ is always {\textit{safe}} for this purpose, because the $\softmin$ operator in this case preserves convexity.

\subsection{Generalization Guarantees}
\label{sec:proposed:general}
This section addresses the statistical generalization of our method. More precisely, we intend to bound the {\it {true adversarial risk}}, i.e. 
$\sup_{P\in\mathcal{B}_{\epsilon}\left(P_0\right)}\mathbb{E}_{P}\left\{\ell\left(\boldsymbol{Z};{\theta^*}\right)\right\}$, where ${\theta^*}$ denotes the optimizer of the empirical risk in \eqref{eq:SSLmainMin}. To this aim, two major concerns need to be addressed: (i) we are training our model against an adversary, and (ii) our training dataset is only partially labeled. For (i), we introduce a novel complexity measure w.r.t. the hypothesis set $\mathcal{L}\triangleq\left\{\ell\left(\cdot;\theta\right)\vert~\theta\in\Theta\right\}$ and data distribution $P_0$, which extends the existing generalization analyses into an adversarial setting. For (ii), we establish a novel {\it {compatibility}} condition among $\mathcal{L}$, $P_0$ and $\eta$ that deals with the semi-supervised aspect of our work.

\subsubsection{Adversarial Complexity Measures}

Conventional Rademacher complexity, denoted by $\mathcal{R}_n\left(\mathcal{F}\right)$, is a tool to measure the richness of a function set $\mathcal{F}$ in classical learning theory \cite{mohri2012foundations}. In fact, this measure tells us about how much a function set is able to learn {\it {noise}}, and thus is exposed to overfitting on small datasets. We give a novel adversarial extension for Rademacher complexity which also appears in our generalization bound at the end of this section. Moreover, we show that our complexity measure converges to zero when $n\rightarrow\infty$, for all function sets with a finite VC-dimension, regardless of the strength of adversary. Before that, let us define the set of $\epsilon$-Monge maps $\mathcal{A}_{\epsilon}$ as the following function set:
\begin{equation}
\label{eq:MongeMapDef}
\mathcal{A}_{\epsilon}\triangleq
\left\{a:\mathcal{Z}\rightarrow\mathcal{Z}\vert~c\left(\boldsymbol{z},a\left(\boldsymbol{z}\right)\right)\leq\epsilon,~\forall\boldsymbol{z}\in\mathcal{Z}\right\}.
\end{equation}
Then, the Semi-Supervised Monge (SSM) Rademacher complexity can be defined as follows:
\\[-3mm]
\begin{definition}[SSM Rademacher Complexity]
\label{def:SSM_main}
For $\mathcal{Z}\triangleq\mathcal{X}\times\mathcal{Y}$,
assume a function set $\mathcal{F}\subseteq\mathbb{R}^{\mathcal{Z}}$ and a distribution $P_0\in M\left(\mathcal{Z}\right)$. Then, for $\epsilon\ge0$, a transportation cost $c$ and $n\in\mathbb{N}$, let us define
\begin{align*}
g_{\mathrm{l}}\left(n\right)
&\triangleq\mathbb{E}_{\boldsymbol{Z}_{1:n},\boldsymbol{\sigma}}\left\{
\sup_{f\in\mathcal{F}}~\frac{1}{n}\sum_{i=1}^{n}\sigma_i
\left[
\sup_{a\in\mathcal{A}_{\epsilon}}~
f\left(a\left(\boldsymbol{Z}_i\right)\right)
\right]
\right\}\quad\quad\mathrm{and}
\\
g_{\mathrm{ul}}\left(n\right)
&\triangleq
\sum_{y\in\mathcal{Y}}
\mathbb{E}_{\boldsymbol{X}_{1:n},\boldsymbol{\sigma}}\left\{
\sup_{f\in\mathcal{F}}~\frac{1}{n}\sum_{i=1}^{n}\sigma_i
\left[
\sup_{a\in\mathcal{A}_{\epsilon}}~
f\left(a\left(\boldsymbol{X}_i,y\right)\right)
\right]
\right\},
\end{align*}
where $\boldsymbol{Z}_{1:n}\widesim{i.i.d.}P_0$ and $\boldsymbol{X}_{1:n}\widesim{i.i.d.}P_{0_{\boldsymbol{X}}}$. $\mathcal{A}_{\epsilon}$ represents the set of $\epsilon$-Monge maps. Also, $\boldsymbol{\sigma}\in\left\{-1,+1\right\}^n$ indicates a vector of independent Rademacher random variables. Then, for a supervision ratio $\eta\in\left[0,1\right]$, the SSM Rademacher complexity of $\mathcal{F}$ is defined as
\begin{align*}
\mathcal{R}_{n,\left(\epsilon,\eta\right)}^{\left(\mathrm{SSM}\right)}\left(
\mathcal{F}
\right)\triangleq
\eta
g_{\mathrm{l}}\left(
\lceil n\eta\rceil
\right)+\left(1-\eta\right)
g_{\mathrm{ul}}\left(
\lceil n\left(1-\eta\right)\rceil
\right).
\end{align*}
\end{definition}
By setting $\epsilon=0$ and $\eta=1$, the above definition simply reduces to the classical Rademacher complexity $\mathcal{R}_n$. We define a function set to be {\it {learnable}}, if $\mathbb{R}_n$ decreases to zero as one increases $n$. Similarly, a function class $\mathcal{F}$ is said to be {\it {adversarially learnable}} w.r.t. parameters $\left(\epsilon,\eta\right)$, if
\begin{equation}
\lim_{n\rightarrow\infty}\mathcal{R}^{\left(\mathrm{SSM}\right)}_{n,\left(\epsilon,\eta\right)}\left(\mathcal{F}\right)=0.
\end{equation}
The above definition is necessary when $\epsilon>0$, since learnability of a function class w.r.t. some distribution $P_0$ does not necessarily guarantee its {\it {adversarial learnability}}. In fact, an adversary can shift the data points and forces the learner to experience regions in $\mathcal{X}\times\mathcal{Y}$ that cannot be accessed by $P_0$ alone. However, one may be concerned about how to numerically compute this measure in practice? The main difference between $\mathcal{R}_n$ and SSM Rademacher complexity is that the latter alters input samples (or distribution) by an adversary. Fortunately, several distribution-free bounds have been established on $\mathcal{R}_n$ so far \cite{mohri2012foundations}, which work for a variety of function classes of practical interest, e.g. classifiers with a bounded VC-dimension (including neural networks), polynomial regression tools with a bounded degree, and etc. 

We show that in case of having a distribution-free bound on the Rademacher complexity of $\mathcal{F}$, the SSM Rademacher complexity can be bounded as well. Mathematically speaking, assuming there exists an asymptotically decreasing upper-bound $\Delta\left(n\right)$ such that $\mathcal{R}_{n}\left(\mathcal{F}\right)\leq \Delta\left(n\right),~
\forall P_0\in M\left(\mathcal{X}\times\mathcal{Y}\right)$. Then for all $\eta\in\left[0,1\right]$ and $\epsilon\ge0$ the following holds (Lemma \ref{lemma:distFreeBounds}):
\begin{equation}
\mathcal{R}^{\left(\mathrm{SSM}\right)}_{n,\left(\epsilon,\eta\right)}\left(\mathcal{F}\right)
\leq
\eta \Delta\left(\lceil n\eta\rceil\right)
+ \left(1-\eta\right)\left\vert\mathcal{Y}\right\vert 
\Delta\left(\lceil n\left(1-\eta\right)\rceil\right),
\end{equation}
where the r.h.s. of the above equation always converges to zero as $n\rightarrow\infty$. This includes the vast majority of classifier families that are being used in real-world applications, e.g. neural networks, support vector machines, random forests and etc. Just as an example, consider the $0$-$1$ loss for a family of classifiers with a VC-dimension of $\mathrm{dim}\left(\Theta\right)$. Then, due to Dudley's entropy bound and Haussler's upper-bound \cite{mohri2012foundations}, there exists constant $C$ such that
\begin{equation}
\Delta\left(n\right)\leq
C\sqrt{\frac{\mathrm{dim}\left(\Theta\right)}{n}},~~\mathrm{and~so}\quad~~
\mathcal{R}^{\left(\mathrm{SSM}\right)}_{n,\left(\epsilon,\eta\right)}\left(\mathcal{F}\right)
\leq C\sqrt{\frac{\mathrm{dim}\left(\Theta\right)}{n}}
\left(\sqrt{\eta}+\sqrt{1-\eta}\left\vert\mathcal{Y}\right\vert\right),
\end{equation}
regardless of $\epsilon$ or the distribution $P_0$ (again, check Lemma \ref{lemma:distFreeBounds}). An interesting implication of this result is that by assuming a bounded VC-dimension for a function set, one can guarantee its learnability even in an adversarial setting where an adversary can apply arbitrarily powerful attacks. However, as long as one is interested in a function set whose classical complexity measures cannot be bounded regardless of the data distribution, not much can be said on adversarial learnability without directly computing $g_{\mathrm{l}}$ and $g_{\mathrm{ul}}$ in Definition \ref{def:SSM_main}. 

\subsubsection{Minimum Supervision Ratio}

As discussed earlier in Section \ref{sec:priorworks}, generalization guarantees for SSL frameworks generally require a {\it {compatibility}} assumption on the hypothesis set $\mathcal{F}$ and data distribution $P_0$. In Appendix \ref{sec:proposed:asymptotic} (and in particular, Definition \ref{def:compatibility}), a new compatibility function, denoted by Minimum Supervision Ratio (MSR), is introduced which has the following functional form:
\begin{equation*}
\mathrm{MSR}_{\left(\mathcal{F},P_0\right)}\left(\lambda,\mathrm{margin}\right):
\mathbb{R}\cup\left\{\pm\infty\right\}\times\mathbb{R}_{\ge0}\rightarrow\left[0,1\right].
\end{equation*}
Intuitively, $\mathrm{MSR}_{\left(\mathcal{F},P_0\right)}$ quantifies the strength of information theoretic relation between the marginal measure $P_{0_{\boldsymbol{X}}}$ and the conditional $P_{0_{\vert\boldsymbol{X}}}$. It also measures the suitability of function set $\mathcal{F}$ to learn such relations. As it will be shown in Theorem \ref{thm:generalBound1}, in order to bound the true risk when unlabeled data are involved, one needs $\eta\ge \mathrm{MSR}_{\left(\mathcal{F},P_0\right)}\left(\lambda,\mathrm{margin}\right)$, for some $\lambda$ and $\mathrm{margin}\ge0$. First argument, $\lambda$, denotes the pessimism of the learner and the second one, $\mathrm{margin}\ge0$, specifies a safety margin for small-size datasets. $\mathrm{MSR}$ is an increasing function w.r.t. $\mathrm{margin}$, while it decreases with $\lambda$. In particular,  $\mathrm{MSR}_{\left(\mathcal{F},P_0\right)}\left(+\infty,\mathrm{margin}\right)=0$, for all $\mathrm{margin}\ge0$.

For negative values of $\lambda$ (optimistic learning), MSR remains small as long as there exists a strong dependency between the distribution of feature vectors $P_{0_{\boldsymbol{X}}}$ and label conditionals $P_{0_{\vert\boldsymbol{X}}}$. This dependency can be obtained, for example, by the {\it {cluster assumption}}. However, MSR does not require such explicit assumptions and thus is able to impose a compatibility condition on the pair $\left(\mathcal{F},P_0\right)$ in a more fundamental way compared to existing works in SSL theory.
Additionally, some loss functions in $\mathcal{F}$ need to be capable of capturing such dependencies, e.g. at least one loss function in $\mathcal{F}$ should resemble the true negative log-likelihood $-\log P_0\left(\boldsymbol{X},y\right)$. Conversely, absence of any dependency between $P_{0_{\boldsymbol{X}}}$ and $P_{0_{\vert\boldsymbol{X}}}$, or the lack of sufficiently ``good" loss functions in $\mathcal{F}$ increases the MSR toward $1$, which forces the learner to choose a large $\lambda$ (in the extreme case $+\infty$) to be able to use the generalization bound of Theorem \ref{thm:generalBound1}. Not to mention that a large $\lambda$ increases the empirical loss which then loosens the bound. This fact, however, should not be surprising since improper usage of unlabeled data is known to be harmful to the generalization instead of improving it.
Based on previous discussions, Theorem \ref{thm:generalBound1} gives a generalization bound for our proposed framework in \eqref{eq:SSLmainMin}:
\\[-3mm]
\begin{thm2}[Generalization]
\label{thm:generalBound1}
For a feature-label space $\mathcal{Z}\triangleq\mathcal{X}\times\mathcal{Y}$ and a parameter space $\Theta$, assume the set of continuous functions $\mathcal{L}\triangleq\left\{\ell\left(\cdot;\theta\right)\vert~\theta\in\Theta\right\}$, with $\ell\left(\cdot;\theta\right):\mathcal{Z}\rightarrow\mathbb{R}$ and $\left\Vert \ell\right\Vert_{\infty}\leq B$ for some $B\ge0$. For $\gamma\ge0$, $\boldsymbol{z}\in\mathcal{Z}$ and $\theta\in\Theta$ let
\begin{equation}
\phi_{\gamma}\left(\boldsymbol{z};\theta\right)\triangleq
\sup_{\boldsymbol{z}'\in\mathcal{Z}}~\ell\left(\boldsymbol{z}';\theta\right)-\gamma c\left(\boldsymbol{z}',\boldsymbol{z}\right),
\end{equation}
and $\Phi\triangleq 
\left\{
\phi_{\gamma}\left(\cdot;\theta\right)~\big\vert~\theta\in\Theta
\right\}$,
where $c:\mathcal{Z}\times\mathcal{Z}\rightarrow\mathbb{R}_{\ge0}$ is a transportation cost. For a supervision ratio $\eta\in\left[0,1\right]$, 
assume a partially labeled dataset $\boldsymbol{D}=\left\{\left(\boldsymbol{X}_i,y_i\right)\right\}^n_{i=1}$ including $n$ i.i.d. samples drawn from $P_0\in M\left(\mathcal{Z}\right)$, where labels can be observed with probability of $\eta$, independently. For $0<\delta\leq1$ and $\lambda\in\mathbb{R}\cup\left\{\pm\infty\right\}$, assume $\eta$ satisfies the following condition:
\begin{equation}
\hspace{-0mm}
\eta\ge
\mathrm{MSR}_{\left(\Phi,P_0\right)}\left(\lambda,
4B\sqrt{\frac{\log
\left(
{1}/{\delta}
\right)
}{2n}}+
4\mathcal{R}^{\left(\mathrm{SSM}\right)}_{n,\left(\epsilon,\eta\right)}\left(\mathcal{L}\right)
\right).
\label{eq:generalThmCondition}
\end{equation}
Then, with probability at least $1-\delta$, the following bound holds for all $\epsilon\ge0$:
\begin{align}
\label{eq:finalGeneralBound}
\sup_{P\in\mathcal{B}_{\epsilon}\left(P_0\right)}\mathbb{E}_P\left\{\ell\left(\boldsymbol{Z};\theta^*\right)\right\}
&\leq
\hat{R}_{\mathrm{SSAR}}\left(\theta^*;\boldsymbol{D}\right)+
2B\sqrt{\frac{\log\left({1}/{\delta}\right)}{2n}}+
2\mathcal{R}^{\left(\mathrm{SSM}\right)}_{n,\left(\epsilon,\eta\right)}\left(\mathcal{L}\right),
\end{align}
where $\theta^*$ is the minimizer of $\hat{R}_{\mathrm{SSAR}}\left(\theta;\boldsymbol{D}\right)$.
\end{thm2}
Proof of Theorem \ref{thm:generalBound1} is given in Appendix \ref{sec:appendix:thm}. Condition in \eqref{eq:generalThmCondition} can always be satisfied based on Lemma \ref{lemma:compExistLemma}, as long as $\lambda$ and $n$ are sufficiently large and $\mathcal{L}$ is adversarially learnable. 
A strongly-compatible pair of hypothesis set $\Phi$ and data distribution $P_0$ should encourage optimism, where learner can choose small (generally negative) $\lambda$ values.
However, in some situations increasing $\lambda$ might be necessary for \eqref{eq:generalThmCondition} to hold; In fact, for a weakly-compatible $\left(\Phi,P_0\right)$, $\lambda$ must be positive or even $+\infty$ (the latter always satisfies \eqref{eq:generalThmCondition} regardless of $n$ or $\eta$). Note that choosing a  larger $\lambda$ increases the empirical risk $\hat{R}_{\mathrm{SSAR}}\left(\theta^*;\boldsymbol{D}\right)$, which then increases our bound in \eqref{eq:finalGeneralBound}. Interestingly, $\lambda=+\infty$ coincides with the setting of \cite{loog2016contrastive}, which makes it as a special case of our analysis. 

For a fixed $\epsilon$, $\gamma$ should be tuned to minimize the upper-bound for a better generalization. On the other hand, for every $\gamma$ there exists $\epsilon\ge0$ where \eqref{eq:finalGeneralBound} becomes asymptotically tight. More importantly, the limiting cases of Theorem \ref{thm:generalBound1}, i.e. $\epsilon=0$ and $\eta=1$, provide us with a new  generalization bound for non-robust SSL, and an already-established bound for supervised DRL in \cite{sinha2018certifying}, respectively.

\section{Experimental Results}
\label{sec:exp}

In this section, we demonstrate our experimental results on a number of real-world datasets and also compare our method with some state-of-the-art rival methodologies. We have chosen our loss function set 
$\left\{\ell\left(\cdot;\theta\right)\vert\theta\in\Theta\right\}$ as a particular family of Deep Neural Networks (DNN). Architecture and other specifications about our DNNs are explained in details in Appendix \ref{sec:appendix:exp}. Throughout this section, our method is denoted by SSDRL and the rival frameworks are Virtual Adversarial Training (VAT) \cite{miyato2018virtual}, Pseudo-Labeling (PL) \cite{lee2013pseudo}, and the fully-supervised DRL of \cite{sinha2018certifying}, simply denoted as DRL. We have also implemented a fast version of SSDRL, called F-SSDRL, where for each unlabeled training sample only a limited number of {\it {more favorable}} labels are considered for Algorithm \ref{alg:SSLDROsgd}. Here by {\it {more favorable}} labels, we refer to those labels that correspond to smaller non-robust loss values of $\ell\left(\cdot;\theta\right)$. As a result, F-SSDRL runs much faster than SSDRL without much degradation in performance. Surprisingly, we found out that F-SSDRL often yields better performances in practice compared to SSDRL (also see Appendix \ref{sec:appendix:exp} for more details).

\begin{figure*}[t]
	\centering
    \begin{subfigure}{0.32\textwidth}
		\includegraphics[width=1.0\textwidth]{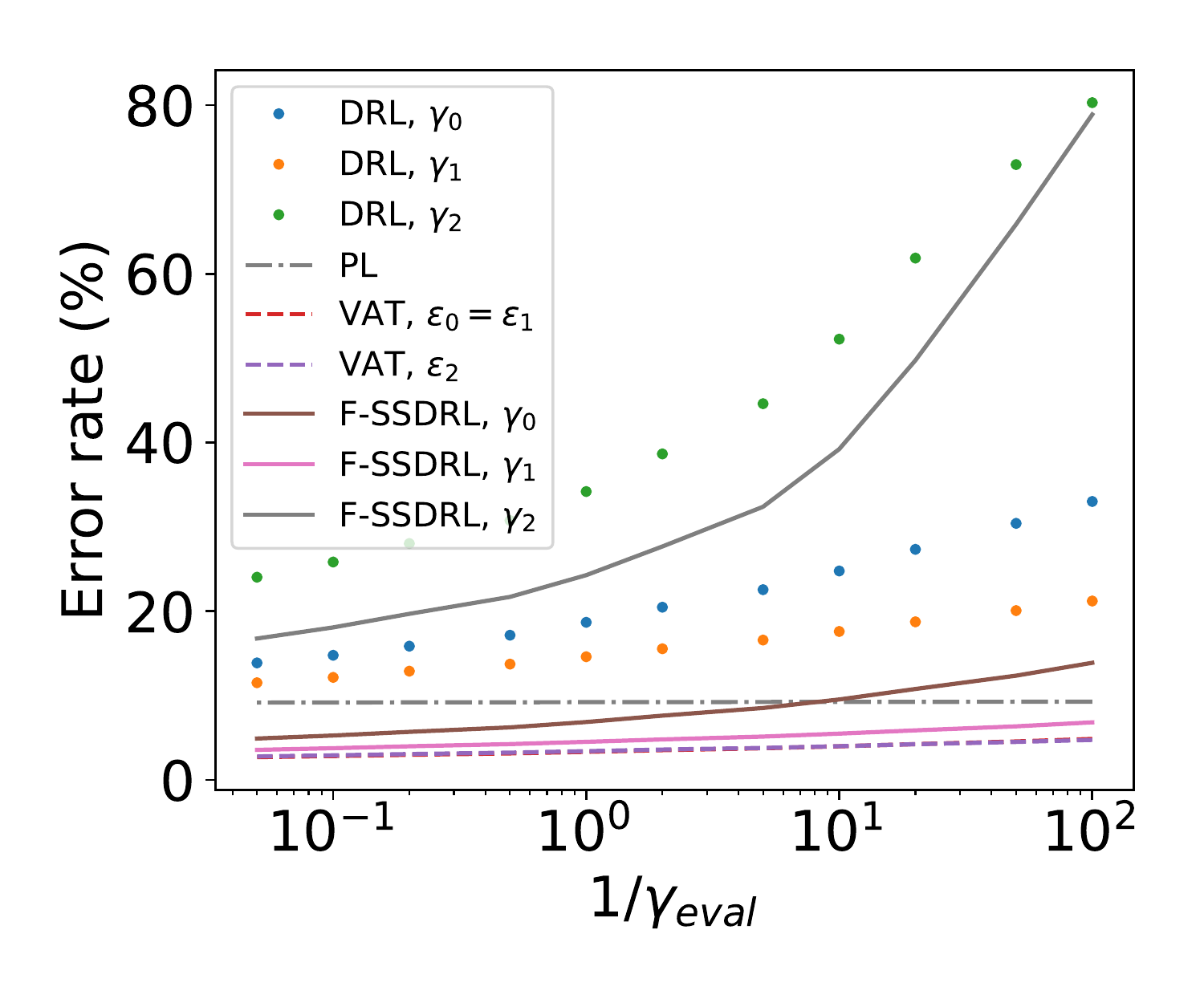}
		\vspace*{-7mm}
        \caption{\label{fig:1:compare_methods_mnist} MNIST}
    \end{subfigure}
    \begin{subfigure}{0.32\textwidth}
		\includegraphics[width=1.0\textwidth]{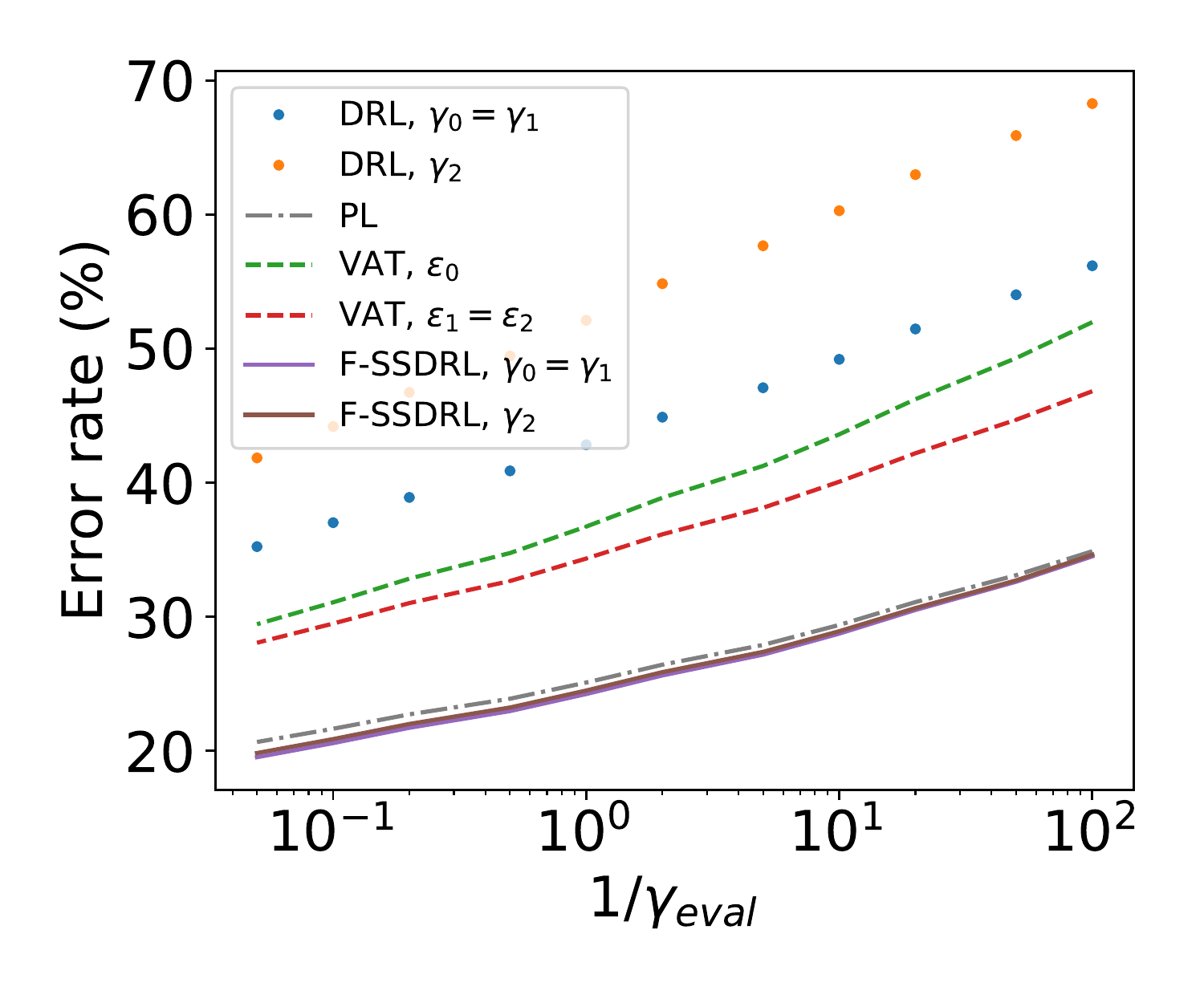}
		\vspace*{-7mm}
        \caption{\label{fig:1:compare_methods_svhn} SVHN}
    \end{subfigure}
    \begin{subfigure}{0.32\textwidth}
		\includegraphics[width=1.0\textwidth]{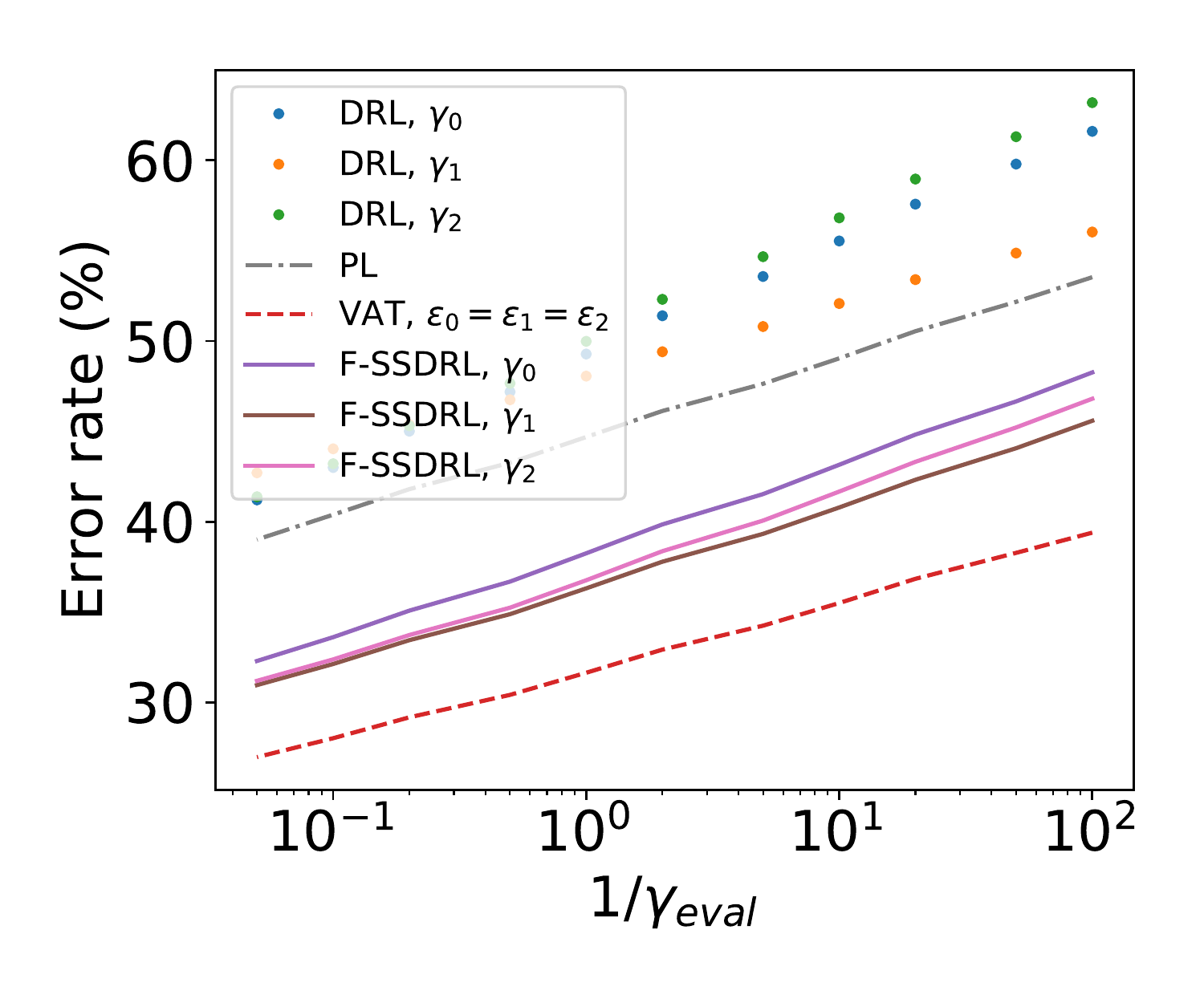}
		\vspace*{-7mm}
        \caption{\label{fig:1:compare_methods_cifar10} CIFAR-10}
    \end{subfigure}
    \caption{Comparison of the test error-rates on adversarial examples attained via \cite{sinha2018certifying} among different methods.}
    \label{fig:mainAcc1}
\end{figure*}

\begin{figure*}[t]
	\centering
    \begin{subfigure}{0.32\textwidth}
		\includegraphics[width=1.0\textwidth]{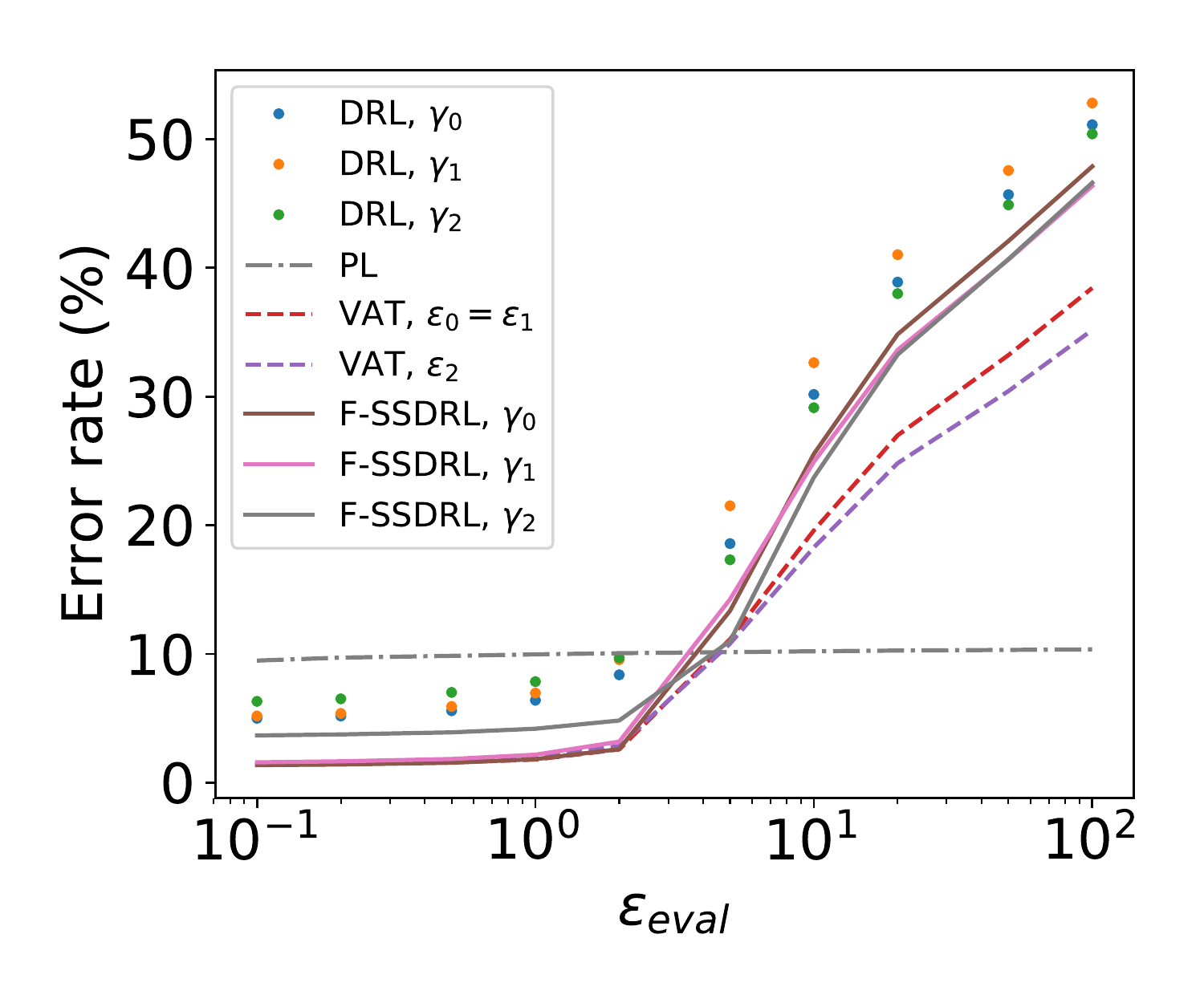}
		\vspace*{-7mm}
        \caption{\label{fig:2:compare_methods_mnist} MNIST}
    \end{subfigure}
    \begin{subfigure}{0.32\textwidth}
		\includegraphics[width=1.0\textwidth]{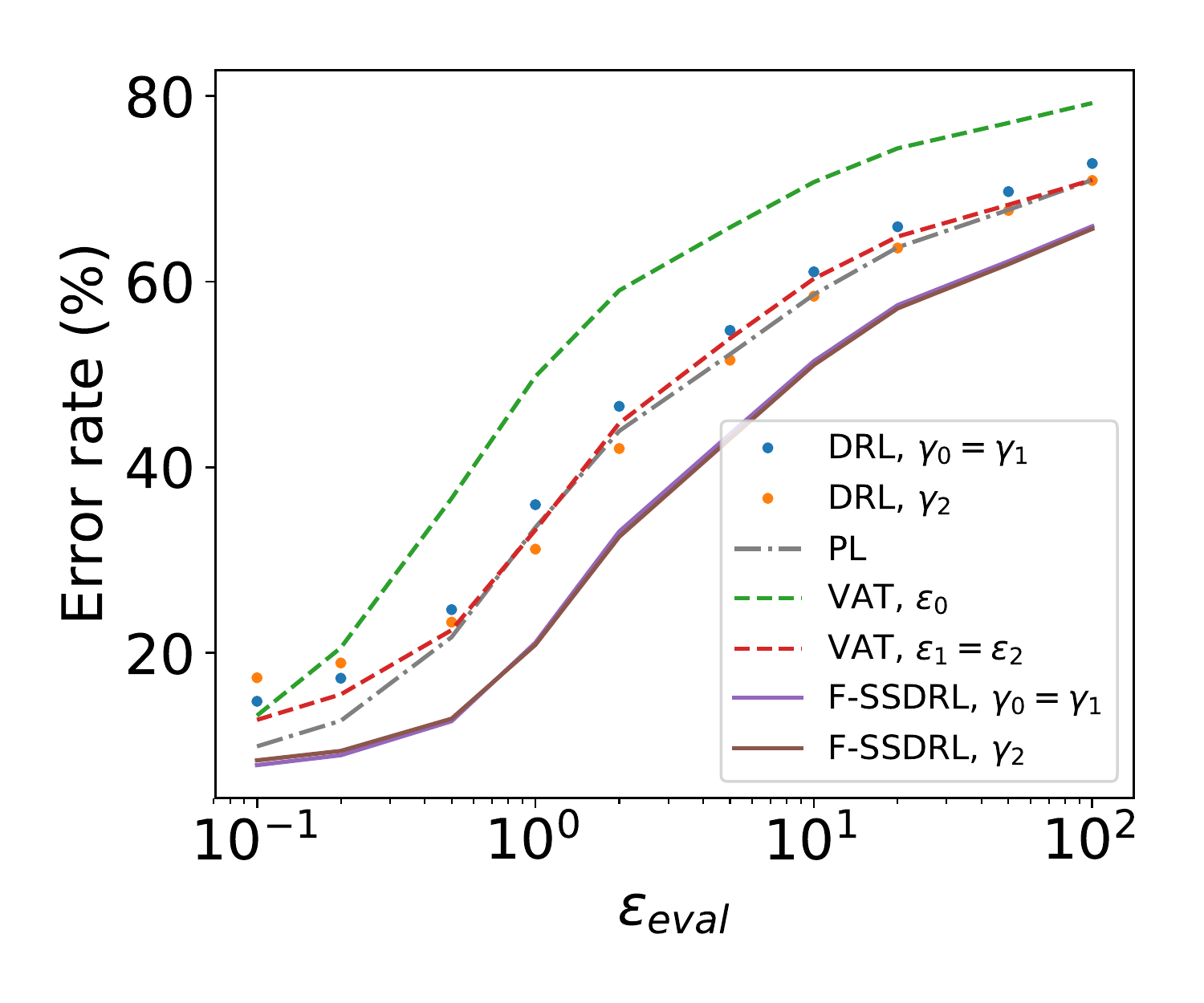}
		\vspace*{-7mm}
        \caption{\label{fig:2:compare_methods_svhn} SVHN}
    \end{subfigure}
    \begin{subfigure}{0.32\textwidth}
		\includegraphics[width=1.0\textwidth]{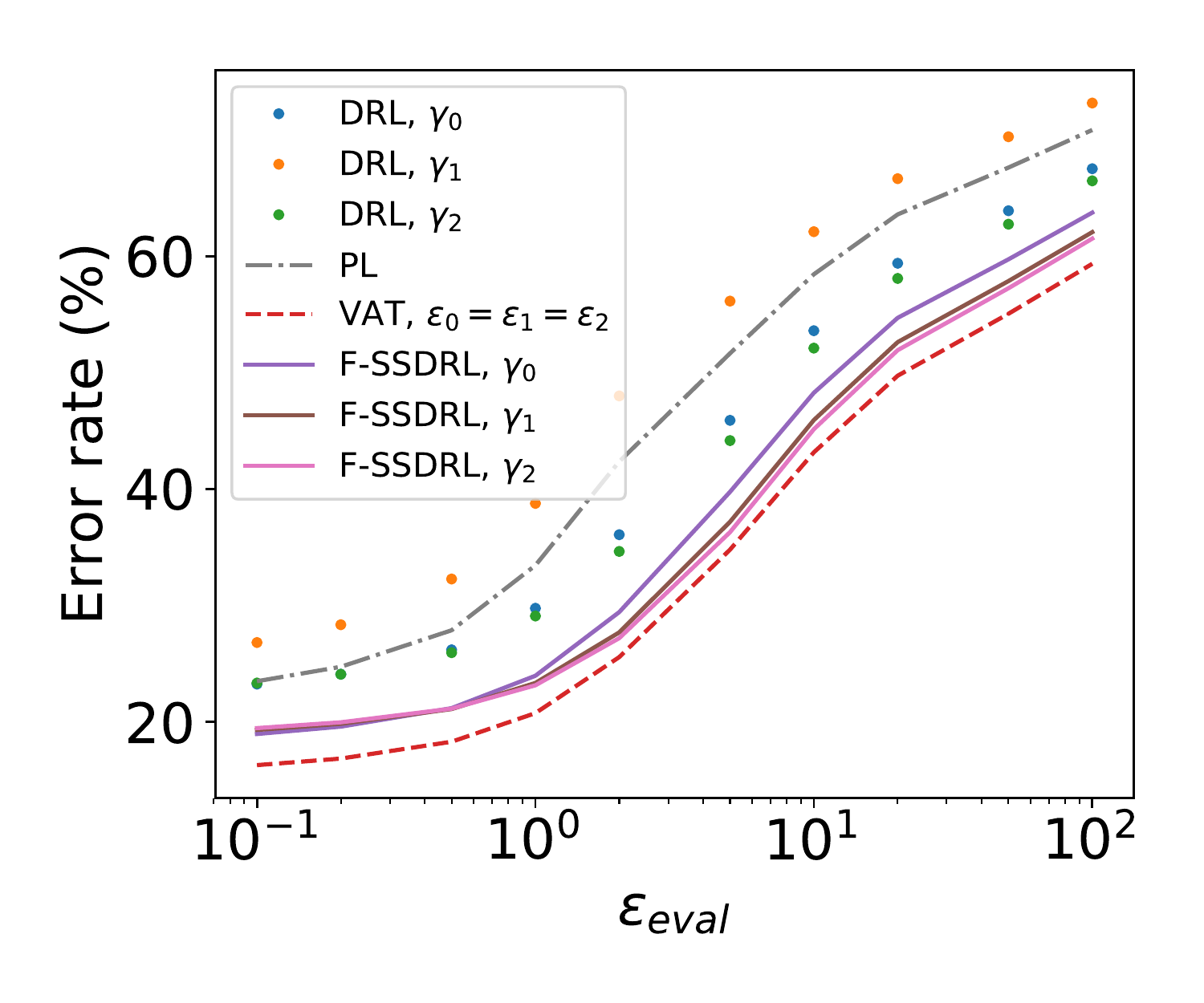}
		\vspace*{-7mm}
        \caption{\label{fig:2:compare_methods_cifar10} CIFAR-10}
    \end{subfigure}
    \caption{Comparison of the test error-rates on adversarial examples calculated by PGM \cite{madry2017towards}, under $\ell_2$-norm constraint.}
    \label{fig:mainAcc2}
\end{figure*}

Figure \ref{fig:mainAcc1} shows the misclassification rate vs. $\gamma^{-1}$ on adversarial test examples attained by \eqref{eq:innerMaxEq} (same attack strategy as \cite{sinha2018certifying}). Recall $\gamma$ as the dual-counterpart of the Wasserstein radius $\epsilon$ in \eqref{eq:main2}. Thus, $\gamma^{-1}$ somehow quantifies the strength of adversarial attacks, as suggested by \cite{sinha2018certifying}. Results have been depicted for MNIST, SVHN and CIFAR-10 datasets. Figure \ref{fig:mainAcc2} demonstrates the same procedure for adversarial examples generated by Projected-Gradient Method (PGM) \cite{madry2017towards}; In this case, the error-rate is depicted vs. PGM's {\it {strength of attack}}, i.e. $\varepsilon$. For VAT and SSDRL, curves have been shown for different choices of hyper-parameters, i.e., $\left(\gamma_i~\mathrm{or}~\varepsilon_i\right),i=1,2,3$, which correspond to  the lowest error rates on: ($i=1$) clean examples, ($i=2$) adversarial examples by \cite{sinha2018certifying}, and ($i=3$) adversarial examples by PGM, respectively. Values of $\left(\gamma_i,\varepsilon_i\right)$, and the choices of $\lambda$, transportation cost $c$, and the supervision ratio $\eta$ with more details on the experiments can be found in Appendix \ref{sec:appendix:exp}.

According to Figures \ref{fig:mainAcc1} and \ref{fig:mainAcc2}, the proposed method is always superior to DRL and PL. Also, SSDRL outperforms VAT on SVHN dataset regardless of the attack type, while it has a comparable error-rate on MNIST and CIFAR-10 based on Figures \ref{fig:1:compare_methods_mnist} and \ref{fig:2:compare_methods_cifar10}, respectively. The superiority over DRL highlights the fact that exploitation of unlabeled data has improved the performance. However, SSDRL under-performs VAT on MNIST and  CIFAR-10 datasets if the order of attacks are reversed. According to Figure \ref{fig:2:compare_methods_mnist}, accuracy of PL degrades quite slowly as PGM's $\varepsilon$ increases, although the loss values increase in Figure \ref{fig:mnist_pgd-l2-adv_val_loss}. This phenomenon is due to the fact that the adversarial directions for increasing the {\it {loss}} and {\it {error-rate}} are not correlated in this particular case.

\begin{SCtable}[0.62][b]
  \centering
		\caption{\label{tab:semisup_wodataaug} Test error-rates on clean examples. For DRL, VAT and F-SSDRL, rows 1 to 3 correspond to the parameter ($\gamma_i$ for DRL and  F-SSDRL, and $\varepsilon_i$ for VAT) that yields the lowest error rates on: ($i=1$) clean examples, ($i=2$)  adversarial examples by \cite{sinha2018certifying}, and ($i=3$) adversarial examples by PGM, respectively. }
		\begin{tabular}{lrrr}
			\toprule
			\textbf{Method} & \multicolumn{3}{c}{\textbf{Test Error-Rate}(\%) }  \\[1mm]
		     & MNIST & SVHN & CIFAR-10 \\
            \midrule
            DRL \\
             \;$\gamma_1$ & 4.67$\pm$0.38 & 10.89$\pm$0.53 & 21.62$\pm$0.40 \\
             \;$\gamma_2$ & 4.77$\pm$0.15 & 10.89$\pm$0.53 & 23.77$\pm$0.65 \\
             \;$\gamma_3$ & 5.95$\pm$0.13 & 14.45$\pm$0.93 & 21.97$\pm$0.35 \\
             \midrule
            PL         &8.70$\pm$0.47&6.39$\pm$0.46&21.19$\pm$0.25\\
             \midrule
            VAT \\
            \;$\varepsilon_1$ & 1.30$\pm$0.04 & 5.47$\pm$0.33 & 15.19$\pm$0.55 \\
            \;$\varepsilon_2$ & 1.30$\pm$0.04 & 7.01$\pm$0.24 & 15.19$\pm$0.55\\
            \;$\varepsilon_3$ & 1.32$\pm$0.10 & 7.01$\pm$0.24 & 15.19$\pm$0.55 \\
            \midrule
            F-SSDRL & & & \\
            \;$\gamma_1$ & 1.29$\pm$0.09 & 6.19$\pm$0.22 & 17.94$\pm$0.20 \\
            \;$\gamma_2$ & 1.51$\pm$0.03 & 6.19$\pm$0.22 & 18.35$\pm$0.24 \\
            \;$\gamma_3$ & 3.58$\pm$1.13 & 6.74$\pm$0.22 & 18.64$\pm$0.16 \\
			\bottomrule
		\end{tabular}
\end{SCtable}

\begin{figure*}[t]
	\centering
    \begin{subfigure}[b]{0.32\textwidth}
		\includegraphics[width=1.0\textwidth]{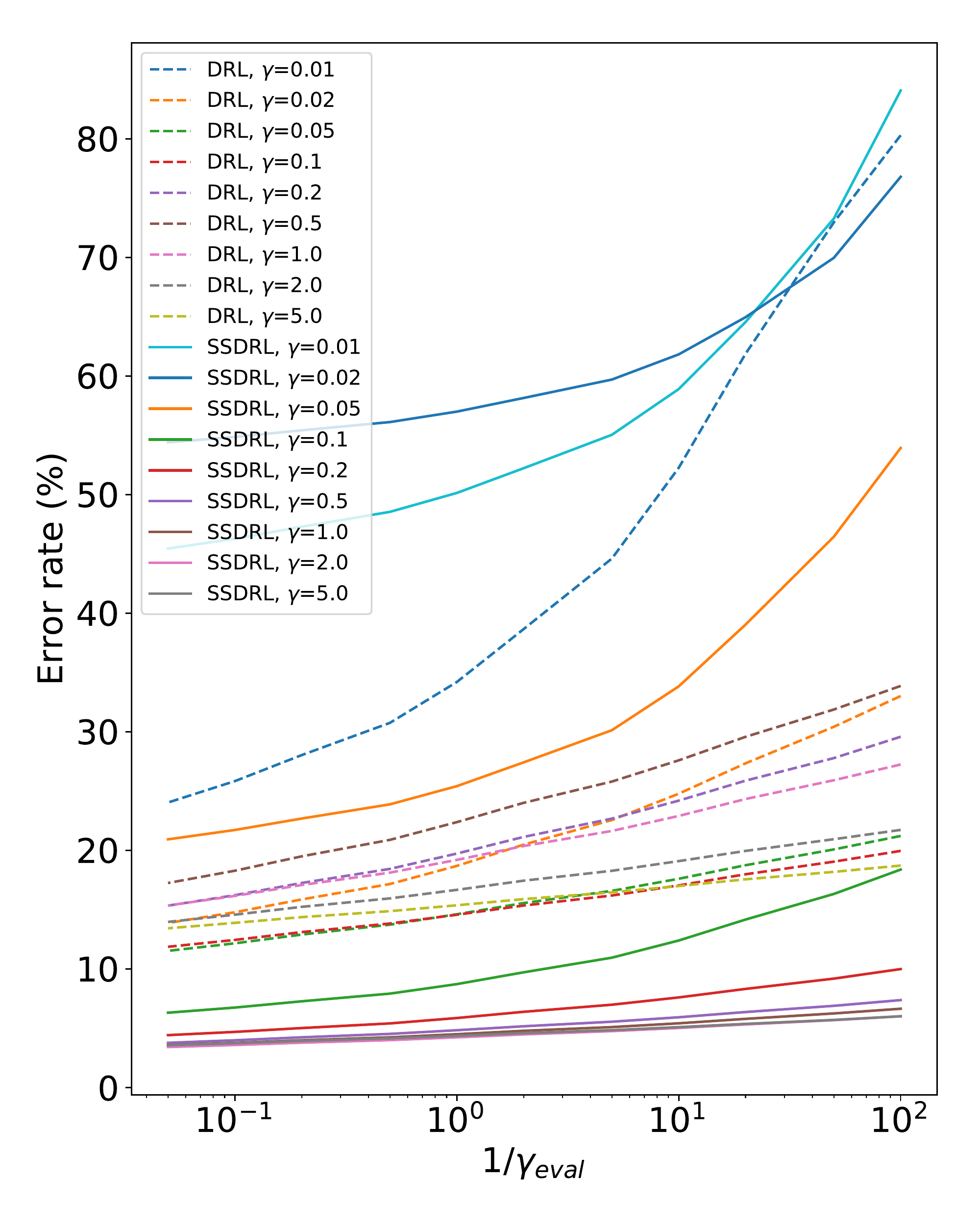}
        \caption{\label{fig:compare_drl_and_ssdrl}}
    \end{subfigure}    
    \begin{subfigure}[b]{0.32\textwidth}
		\includegraphics[width=1.0\textwidth]{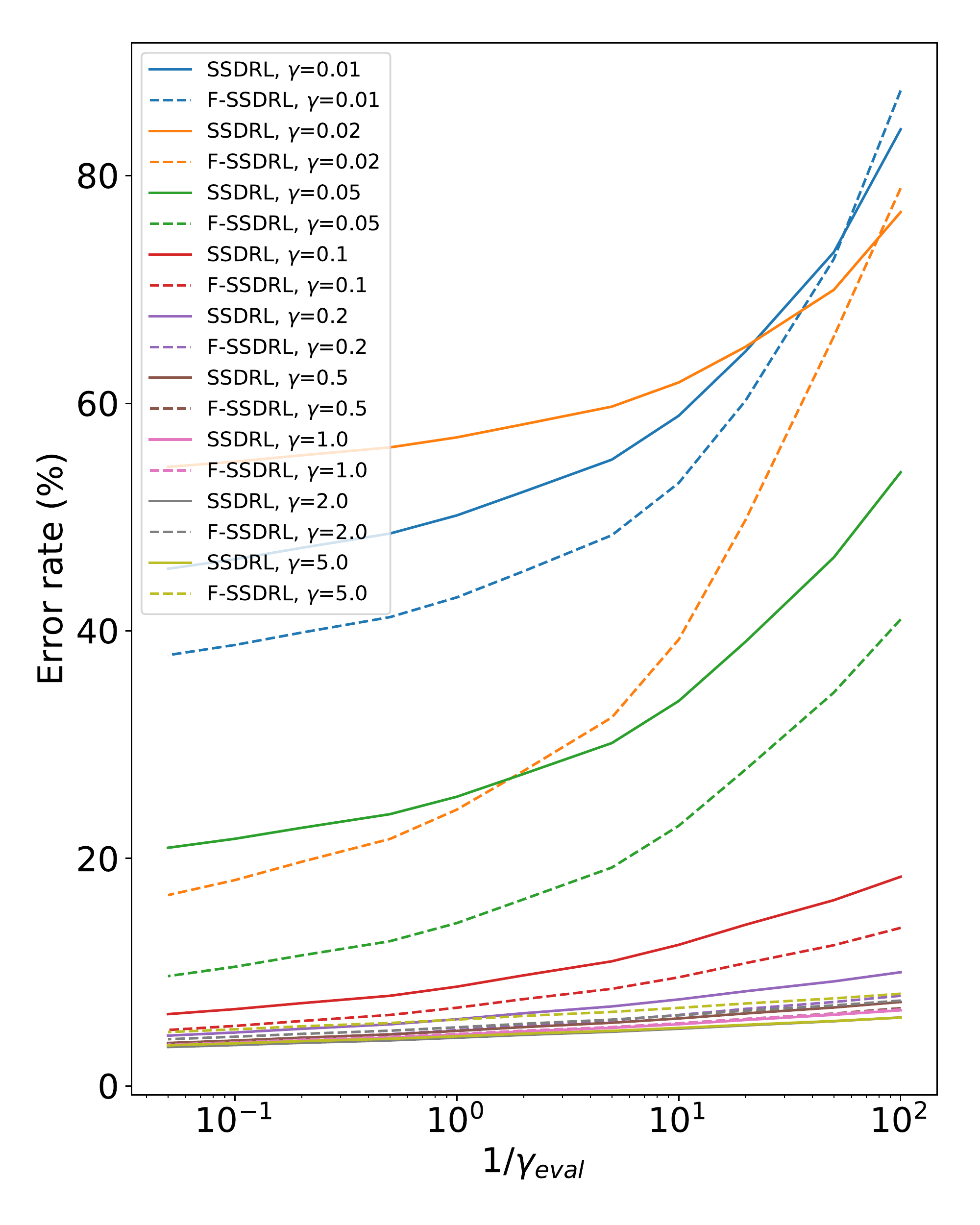}
        \caption{\label{fig:compare_with_fast}}
    \end{subfigure}
    \begin{subfigure}[b]{0.32\textwidth}
		\includegraphics[width=1.0\textwidth]{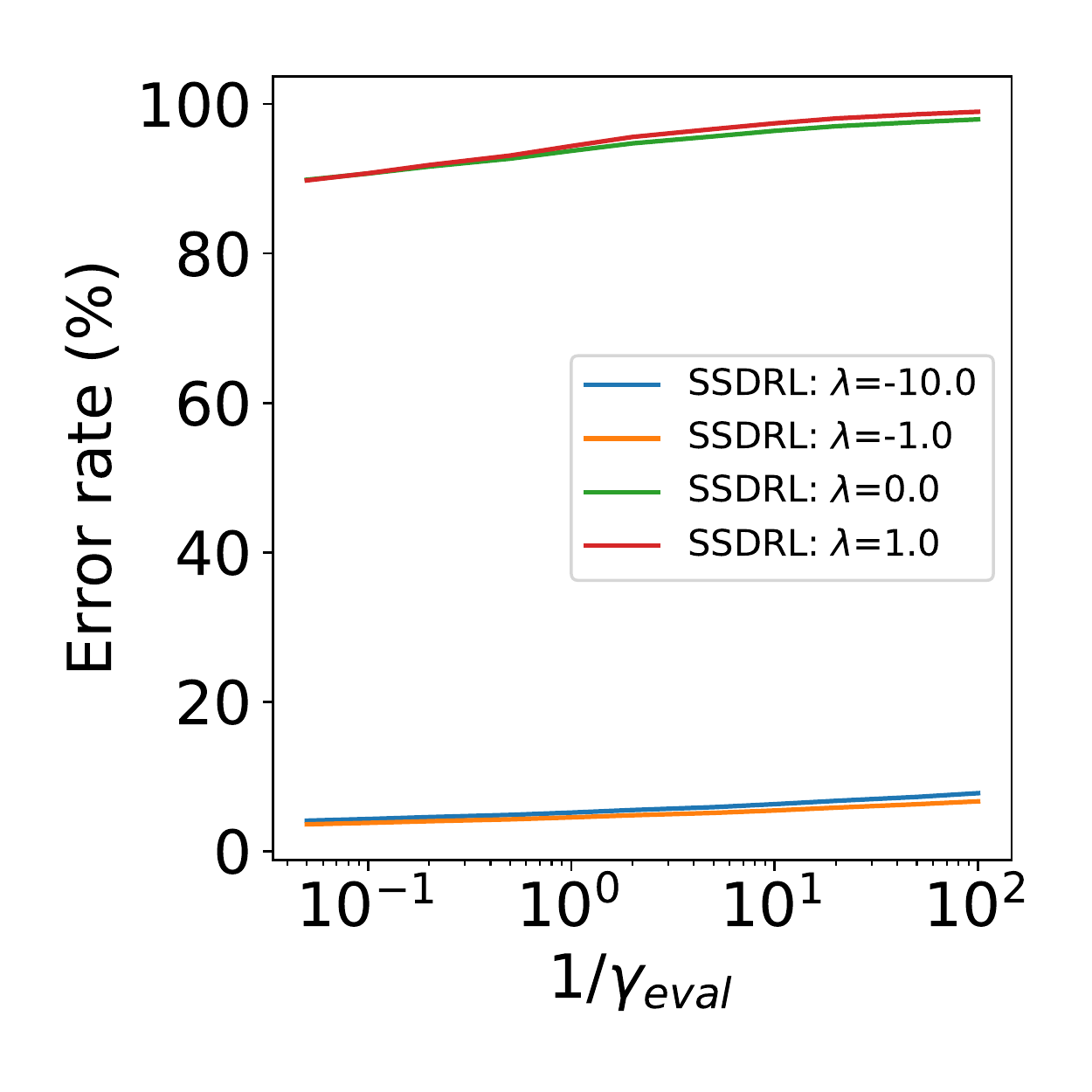}
        \caption{\label{fig:lambda}}
    \end{subfigure}
    \caption{\label{fig:mnist} Error rates on adversarial examples generated via the algorithm in \cite{sinha2018certifying} vs. $\gamma^{-1}_{\mathrm{eval}}$ on the MNIST dataset.}
\end{figure*}

Figure \ref{fig:mnist} depicts the error-rate corresponding to DRL, SSDRL and F-SSDRL as a function of $\gamma^{-1}$, on adversarial examples in the MNIST dataset which are generated via the maximization problem $\argmax_{\boldsymbol{z}'}~\ell\left(\boldsymbol{z}';\theta\right)-\gamma c\left(\boldsymbol{z}';\cdot\right)$ (as described in \cite{sinha2018certifying}). Unlike Figures \ref{fig:mainAcc1} and \ref{fig:mainAcc2}, we have shown the results for a range of values of $\gamma$ and $\lambda$, in order to experimentally measure the sensitivity of our method to these hyper-parameters. Also, we have performed the same procedure for DRL for the sake of comparison. In particular, Figure \ref{fig:compare_drl_and_ssdrl} shows the comparison between DRL and SSDRL (with $\lambda$ set to $-1$ for SSDRL) and different values of $\gamma$. As it is evident for the majority of cases ($\gamma \geq 0.05$), SSDRL performs much better than DRL. This result indicates that employing the unlabeled data samples improves the generalization, which is highly favorable. Figure \ref{fig:compare_with_fast} depicts the comparison between F-SSDRL and the original SSDRL (again $\lambda$ is set to -1 for SSDRL). Figure \ref{fig:lambda} shows the effect of varying $\lambda$ (with $\gamma$ fixed to $1$). Surprisingly, the error-rate experiences a drastic jump when one changes the sign of $\lambda$, which indicates a trade-off between {\it {optimism}} and {\it {pessimism}}. This result might be related to the fact that for the case of MNIST dataset, learned neural networks on the labeled part of the dataset are sufficiently reliable, and thus encourage the user to employ an optimistic approach (i.e., setting a negative $\lambda$) in order to improve the performance. However, while the sign of $\lambda$ is fixed, error-rate does not show that much sensitivity to the magnitude of $\lambda$, which can be noted as a point of strength for SSDRL.

Table \ref{tab:semisup_wodataaug} shows the test error-rates on clean examples for F-SSDRL, VAT, PL and DRL on MNIST, SVHN and CIFAR-10 datasets. In fact, Table \ref{tab:semisup_wodataaug} characterizes the non-adversarial generalization that can be attained via distributional robustness. Again, F-SSDRL outperforms both PL and DRL in almost all experimental settings. It also surpasses VAT on SVHN dataset. F-SSDRL under-performs VAT on MNIST and CIFAR-10, however, the difference in error-rates remains small and the two methods have close performances.

\begin{figure*}[t]
	\centering
    \begin{subfigure}{0.32\textwidth}
		\includegraphics[width=1.0\textwidth]{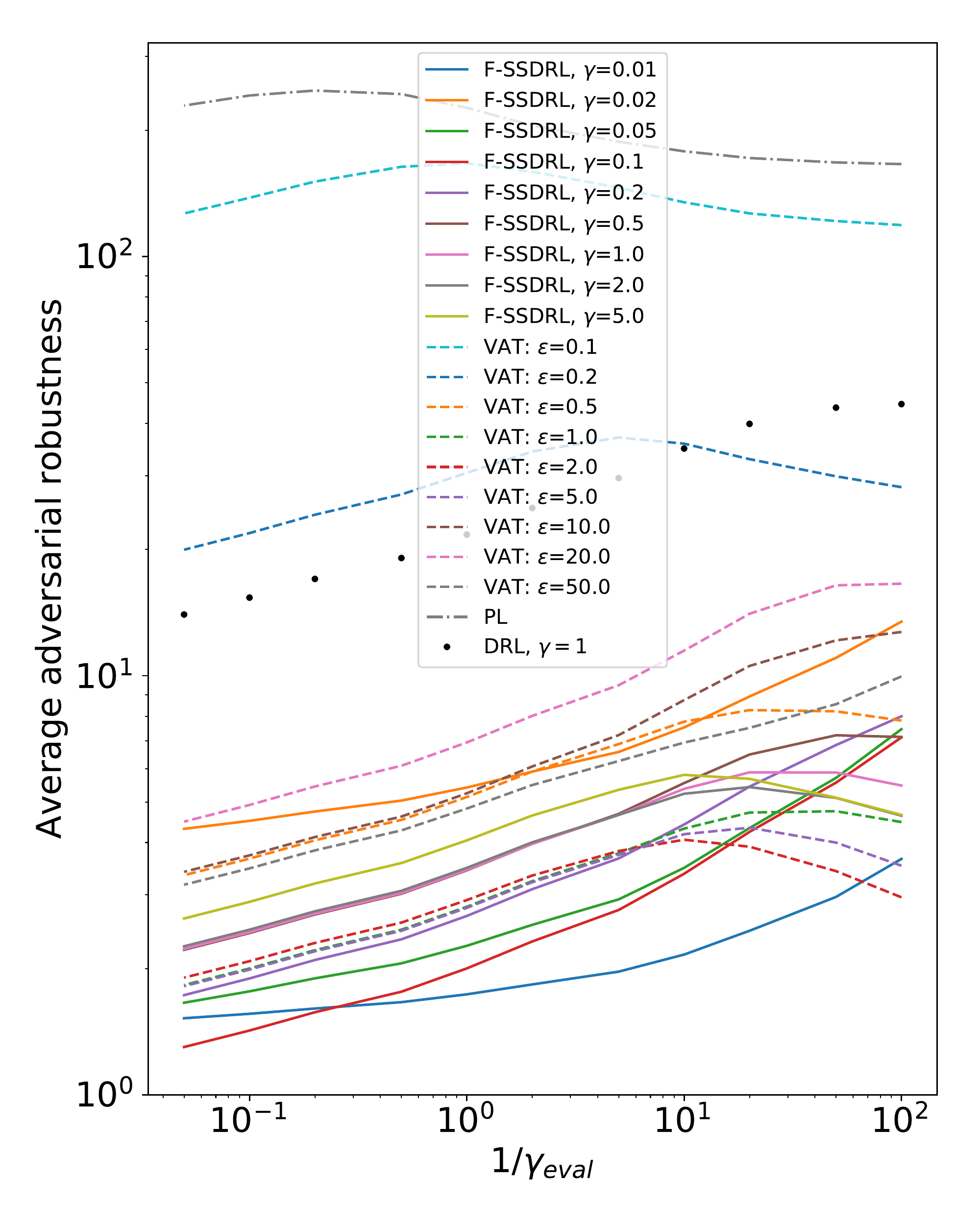}
        \caption{\label{fig:mnist_drl-phi} MNIST}
    \end{subfigure}
    \begin{subfigure}{0.32\textwidth}
		\includegraphics[width=1.0\textwidth]{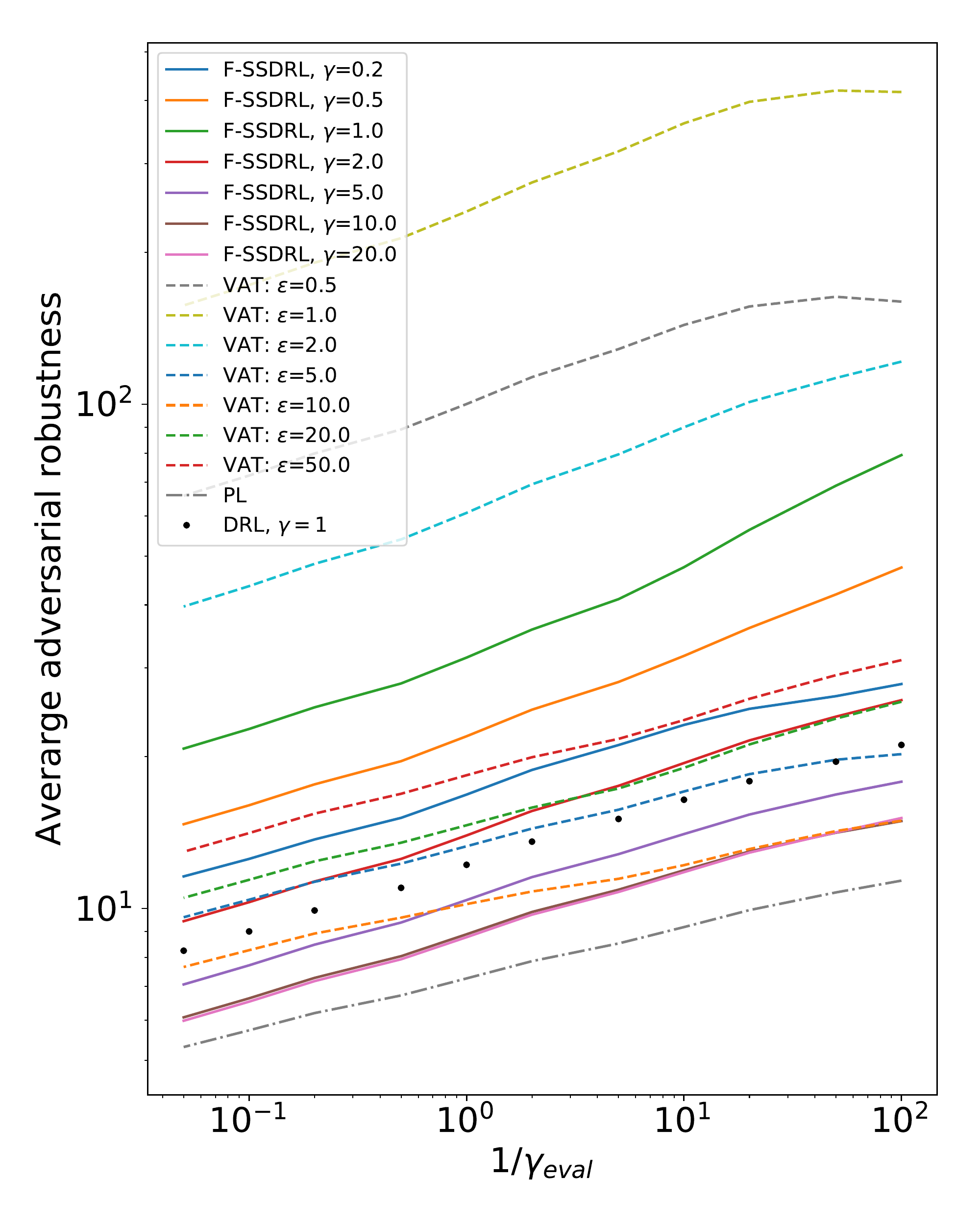}
        \caption{\label{fig:svhn_drl-phi} SVHN}
    \end{subfigure}
    \begin{subfigure}{0.32\textwidth}
		\includegraphics[width=1.0\textwidth]{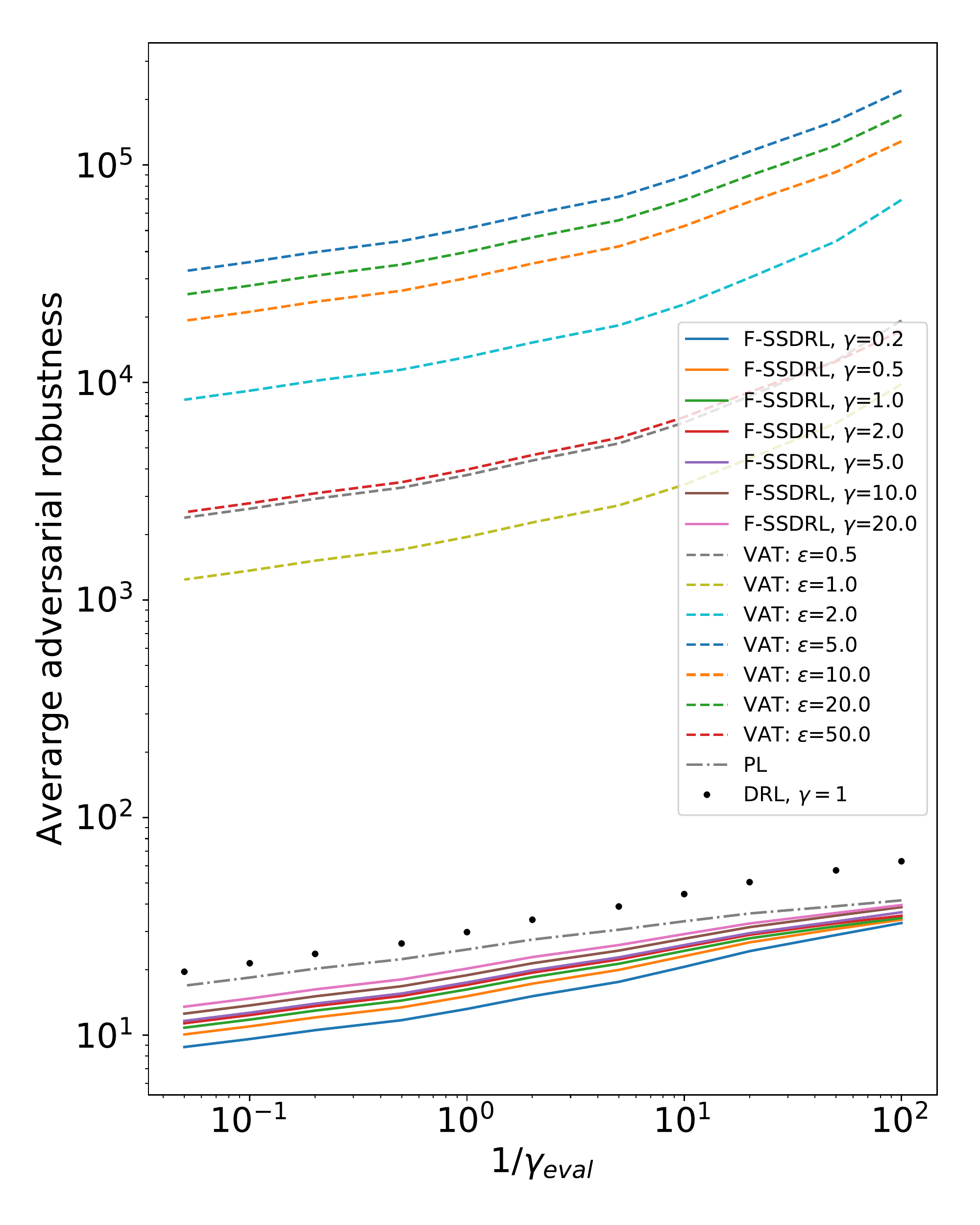}
        \caption{\label{fig:cifar10_drl-phi} CIFAR-10}
    \end{subfigure}
    \caption{\label{fig:drl-phi} Comparison of the average adversarial loss among different methods.}
\end{figure*}

\begin{figure*}[t]
	\centering
    \begin{subfigure}{0.32\textwidth}
		\includegraphics[width=1.0\textwidth]{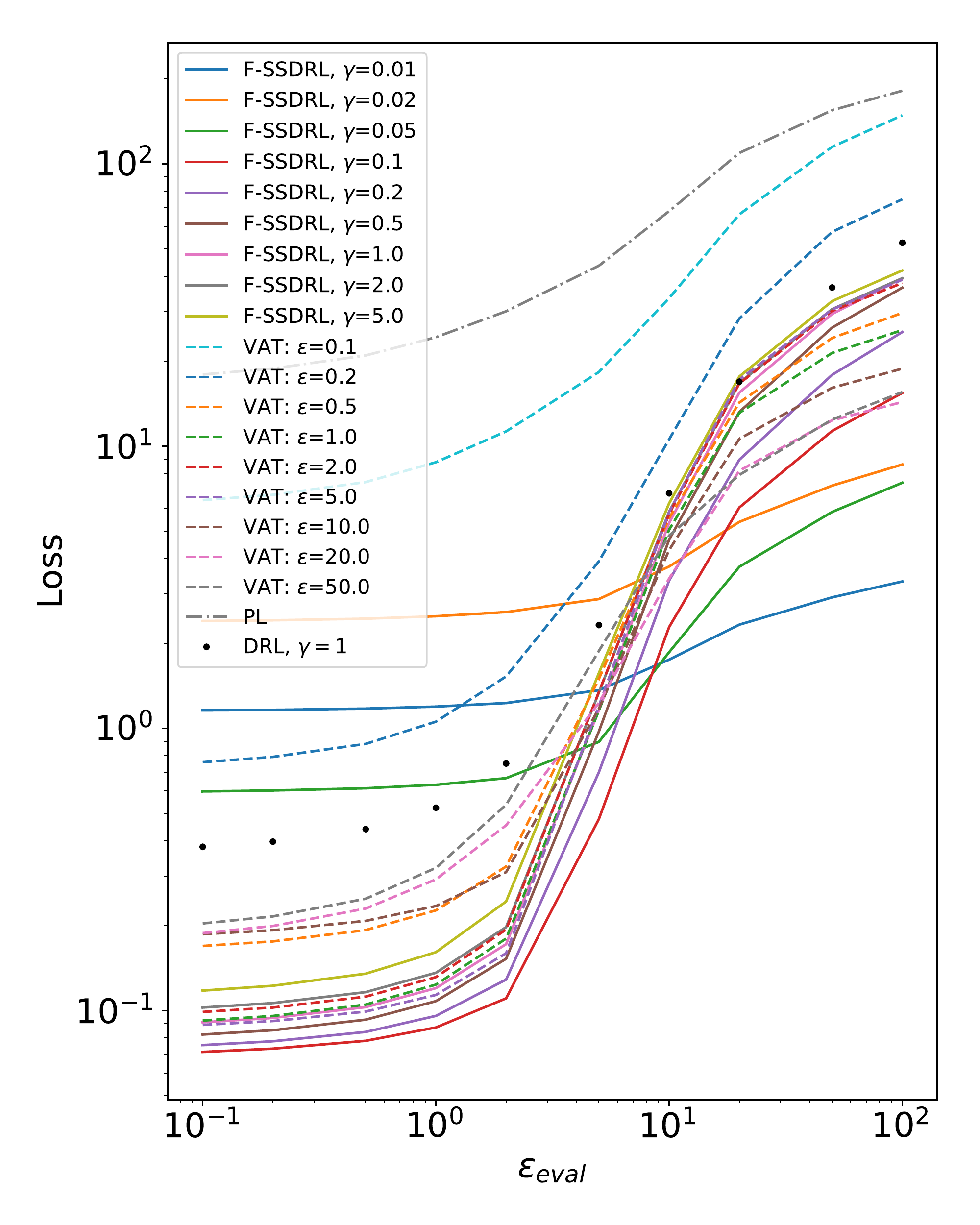}
        \caption{\label{fig:mnist_pgd-l2-adv_val_loss} MNIST}
    \end{subfigure}
    \begin{subfigure}{0.32\textwidth}
		\includegraphics[width=1.0\textwidth]{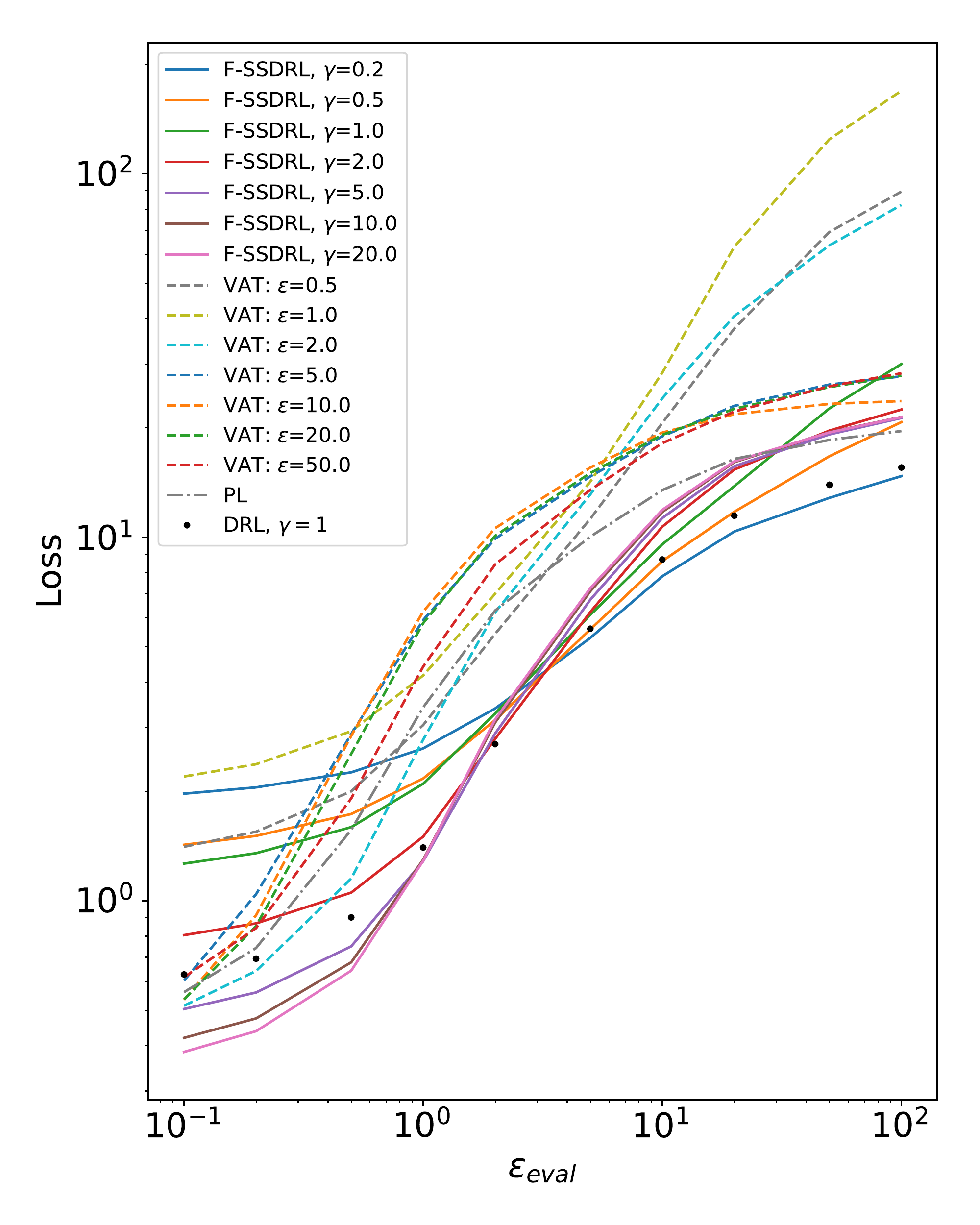}
        \caption{\label{fig:svhn_pgd-l2-adv_val_loss} SVHN}
    \end{subfigure}
    \begin{subfigure}{0.32\textwidth}
		\includegraphics[width=1.0\textwidth]{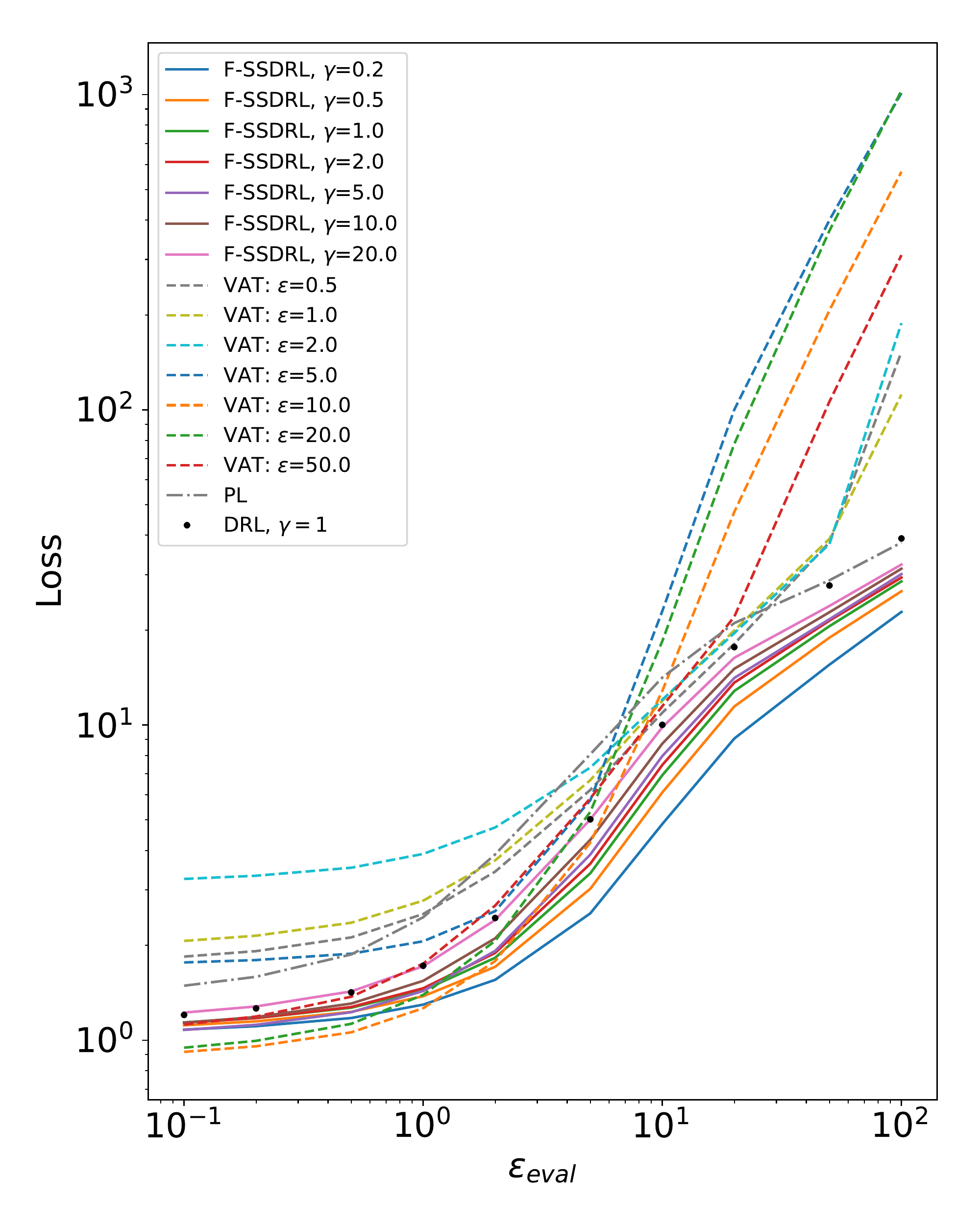}
        \caption{\label{fig:cifar10_pgd-l2-adv_val_loss} CIFAR-10}
    \end{subfigure}
    \caption{\label{fig:pgd-l2-adv_val_loss} Comparison of the loss on adversarial examples calculated by projected-gradient method (PGM)~\cite{madry2017towards} under $\ell_2$ norm constraint.}
\end{figure*}

So far, the performance of SSDRL has been demonstrated w.r.t. its misclassification rate. We have also provided extensive experimental results on the value of adversarial loss $\phi_{\gamma}$, which are crucial for the computation of our generalization bound in Section \ref{sec:proposed:general}. Figure \ref{fig:drl-phi} shows the average adversarial loss, i.e. $\frac{1}{n_{\mathrm{test}}}\sum_{i\in\mathrm{test}}\phi_{\gamma}\left(\boldsymbol{z}_i\right)$, for different methods and on different datasets. $\lambda$ is set to $-1$ for SSDRL. Again, it should be noted that the adversarial examples used in Figures \ref{fig:drl-adv_val_acc} and \ref{fig:drl-phi} are generated via the procedure described in \cite{sinha2018certifying}. Figure \ref{fig:pgd-l2-adv_val_loss} is the counterpart of Figure \ref{fig:drl-phi}, where the attack strategy is replaced with Projected-Gradient Method (PGM). As a result, adversarial loss values have been depicted as a function of PGM's strength of attack, i.e. $\varepsilon$. As can be seen, SSDRL (or its fast version F-SSDRL) are always among the few methods that generate the smallest adversarial loss values, regardless of the strength of attacks. This means that the proposed method can establish a reliable certificate of robustness for test samples via Theorem \ref{thm:generalBound1}. Note that VAT, another method that performs well in practice in terms of error-rate, does not have any theoretical guarantees.

\section{Conclusions}
\label{sec:future}
This paper aims to investigate the applications of distributionally robust learning in partially labeled datasets. The core idea is to focus on a well-known semi-supervised technique, known as self-learning, and make it robust to adversarial attacks. A novel framework, called SSDRL, has been proposed which builds upon an existing general scheme in supervised DRL. SSDRL encompasses many existing methods such as Pseud-Labeling (PL) and EM algorithm as its special cases. Computational complexity of our method is shown to be only slightly higher than those of its supervised counterparts. We have also derived convergence and generalization guarantees for SSDRL, where for the latter, a number of novel complexity measures have been introduced. We have proposed an adversarial extension of the Rademacher complexity in classical learning theory, and showed that it can be bounded for a broad range of learning frameworks, including neural networks, that have a finite VC-dimension. Moreover, our theoretical analysis reveals a more fundamental way to quantify the role unlabeled data in the generalization through a new complexity measure called Minimum Supervision Ratio (MSR). This is in contrast to many existing works that need more restrictive conditions such as {\it {cluster~assumption}} to be applicable. Extensive computer simulation on real-world benchmark datasets demonstrate a comparable-to-superior performance for our method compared with those of the state-of-the-art. In future, one may attempt to improve the generalization bounds, for example, by finding empirical estimations for MSR function. Fitting a broader range of SSL methods into the core idea of Section \ref{sec:proposed:main} could be another good research direction.

\bibliographystyle{IEEEtran}
\bibliography{IEEEabrv,ref}


\appendix
\numberwithin{equation}{section}
\numberwithin{thm}{section}
\numberwithin{thm2}{section}
\numberwithin{lemma}{section}
\numberwithin{definition}{section}
\numberwithin{corl}{section}
\numberwithin{figure}{section}

\section{Additional Simulations and Experimental Settings}
\label{sec:appendix:exp}
This section presents a number of additional experiments w.r.t. the proposed method and shows more comparison with rival methodologies. We also give an extensive description of the experimental setting that we have used for our computer simulations.

\subsection{Additional Simulations}

Figure \ref{fig:drl-adv_val_acc} is a complete version of Figure \ref{fig:mainAcc1} from Section \ref{sec:exp}, where the performances of SSDRL, fully-supervised DRL, PL and VAT are extensively investigated on three benchmark datasets, i.e. MNIST, SVHN and CIFAR-10. SSDRL and VAT have been tested with a variety of their corresponding hyper-parameters $\gamma$ and $\epsilon$. Figure \ref{fig:pgd-l2-adv_val_acc} is the counterpart of Figure \ref{fig:drl-adv_val_acc}, where the {\it {attack}} strategy is replaced with Projected-Gradient Method (PGM). Again, error-rates have been depicted as a function of PGM's {\it {attack~strength}}, i.e. $\varepsilon$. Even though more variation in hyper-parameters has been considered, we have not observed any significant sensitivity that is caused by a slight change of parameter values. As a result, one can say that DRL, SSDRL and VAT are all stable algorithms w.r.t. to their parameter values, at least up to some certain levels. 

\begin{figure*}[t]
	\centering
    \begin{subfigure}{0.32\textwidth}
		\includegraphics[width=1.0\textwidth]{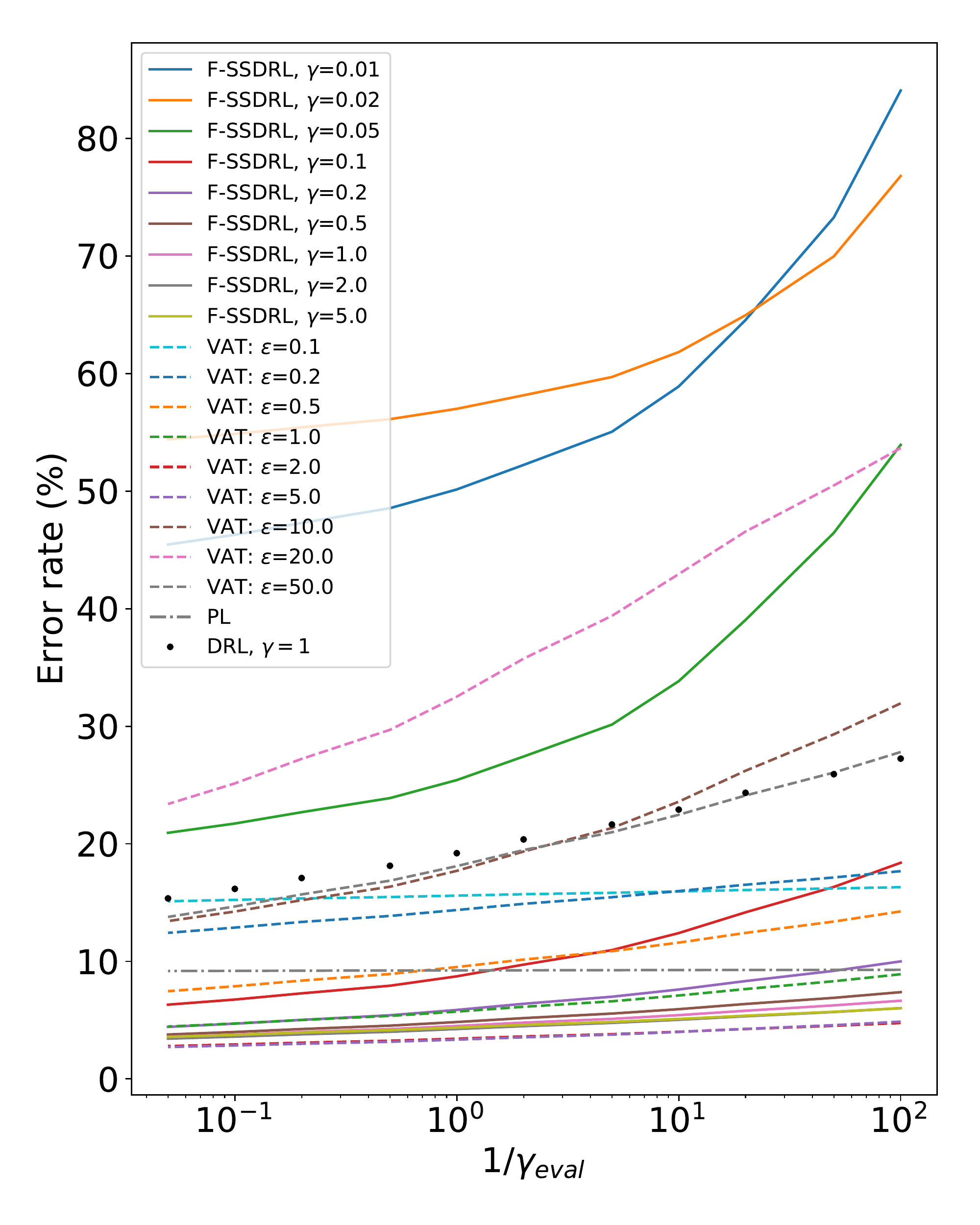}
        \caption{\label{fig:mnist_drl-adv_val_acc} MNIST}
    \end{subfigure}
    \begin{subfigure}{0.32\textwidth}
		\includegraphics[width=1.0\textwidth]{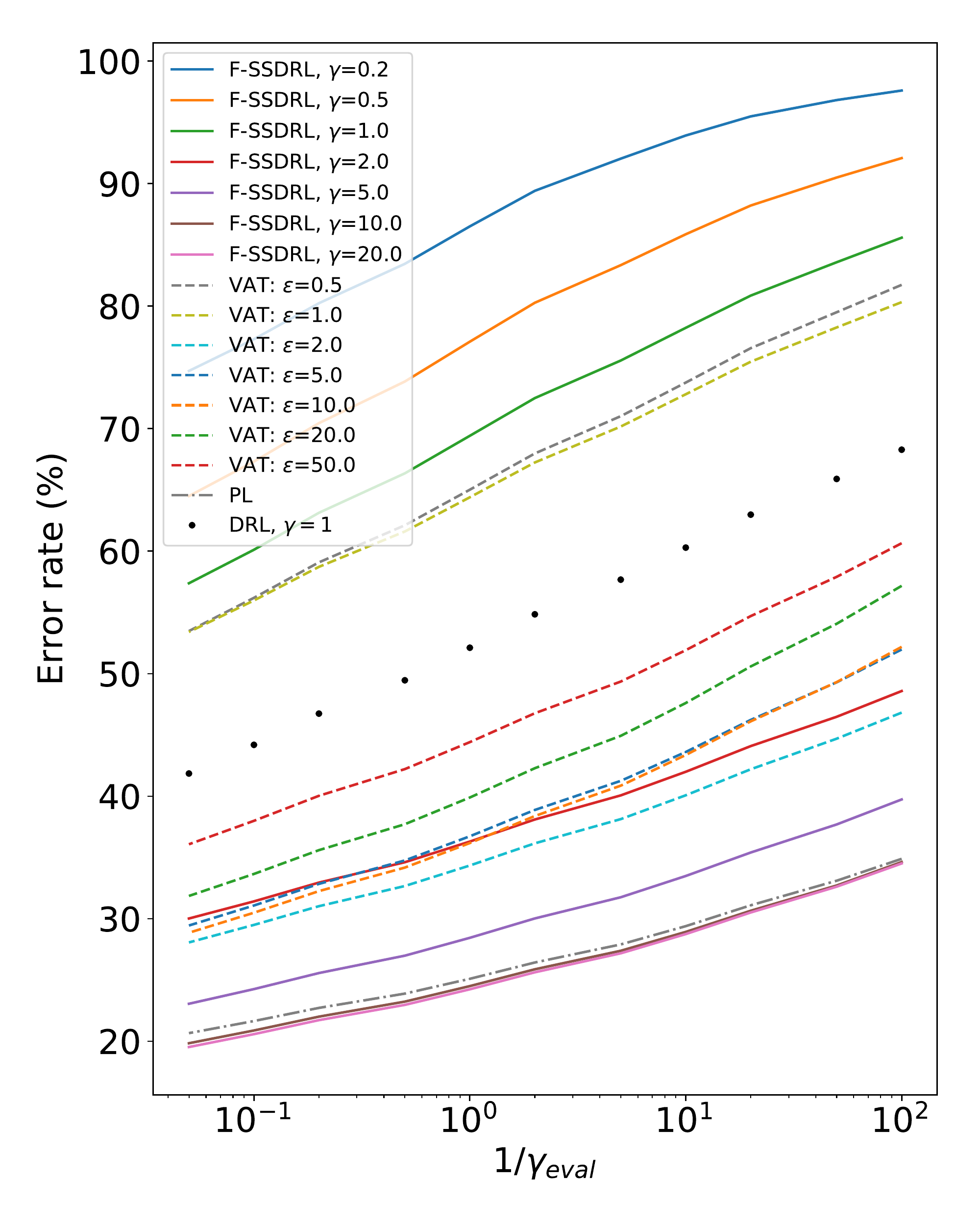}
        \caption{\label{fig:svhn_drl-adv_val_acc} SVHN}
    \end{subfigure}
    \begin{subfigure}{0.32\textwidth}
		\includegraphics[width=1.0\textwidth]{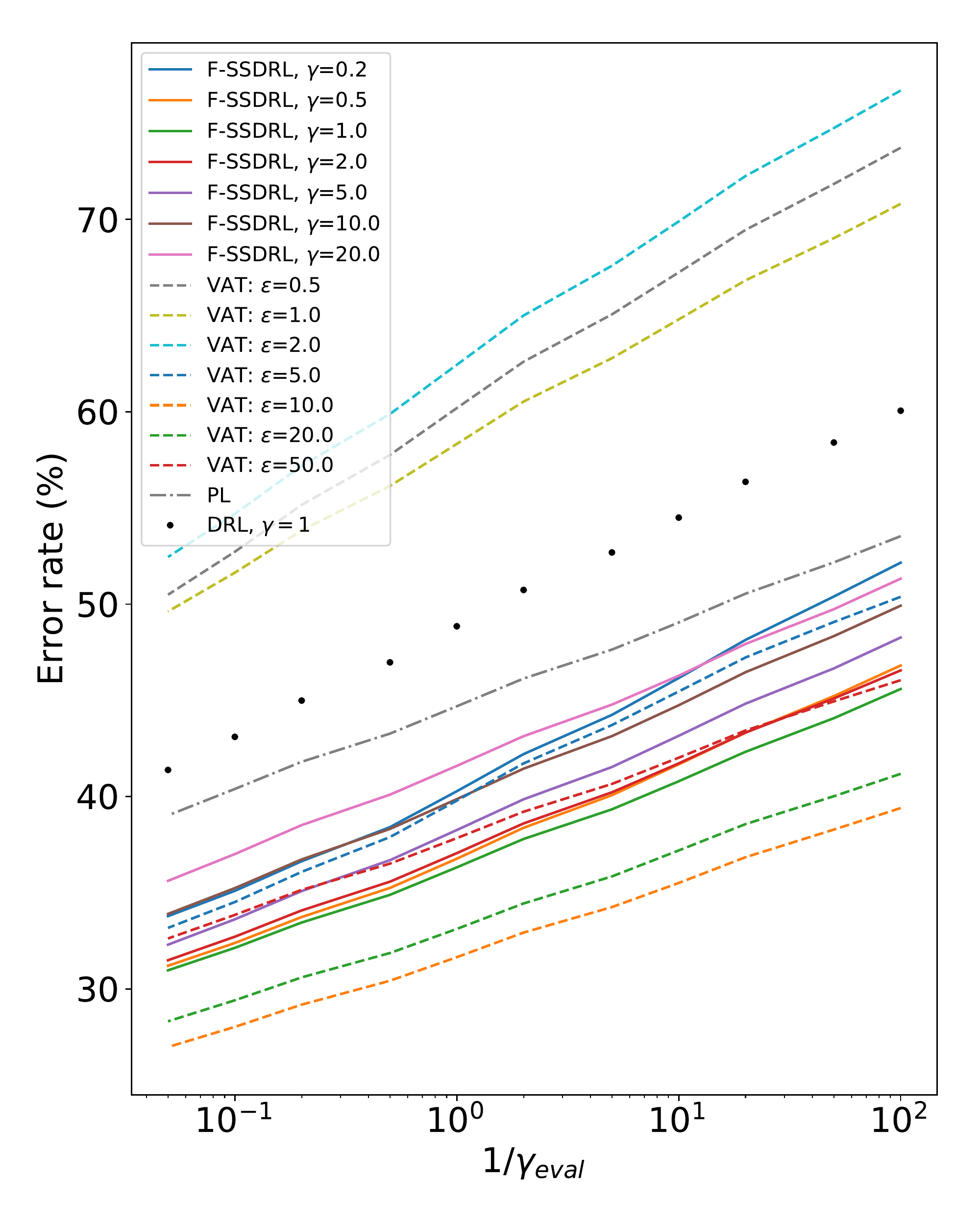}
        \caption{\label{fig:cifar10_drl-adv_val_acc} CIFAR10}
    \end{subfigure}
    \caption{\label{fig:drl-adv_val_acc} Comparison of test error rates of SSDRL, DRL, PL and VAT on the adversarial examples generated via \cite{sinha2018certifying} on different datasets. $\lambda$ is set to $-1$.}
\end{figure*}
\begin{figure*}[t]
	\centering
    \begin{subfigure}{0.32\textwidth}
		\includegraphics[width=1.0\textwidth]{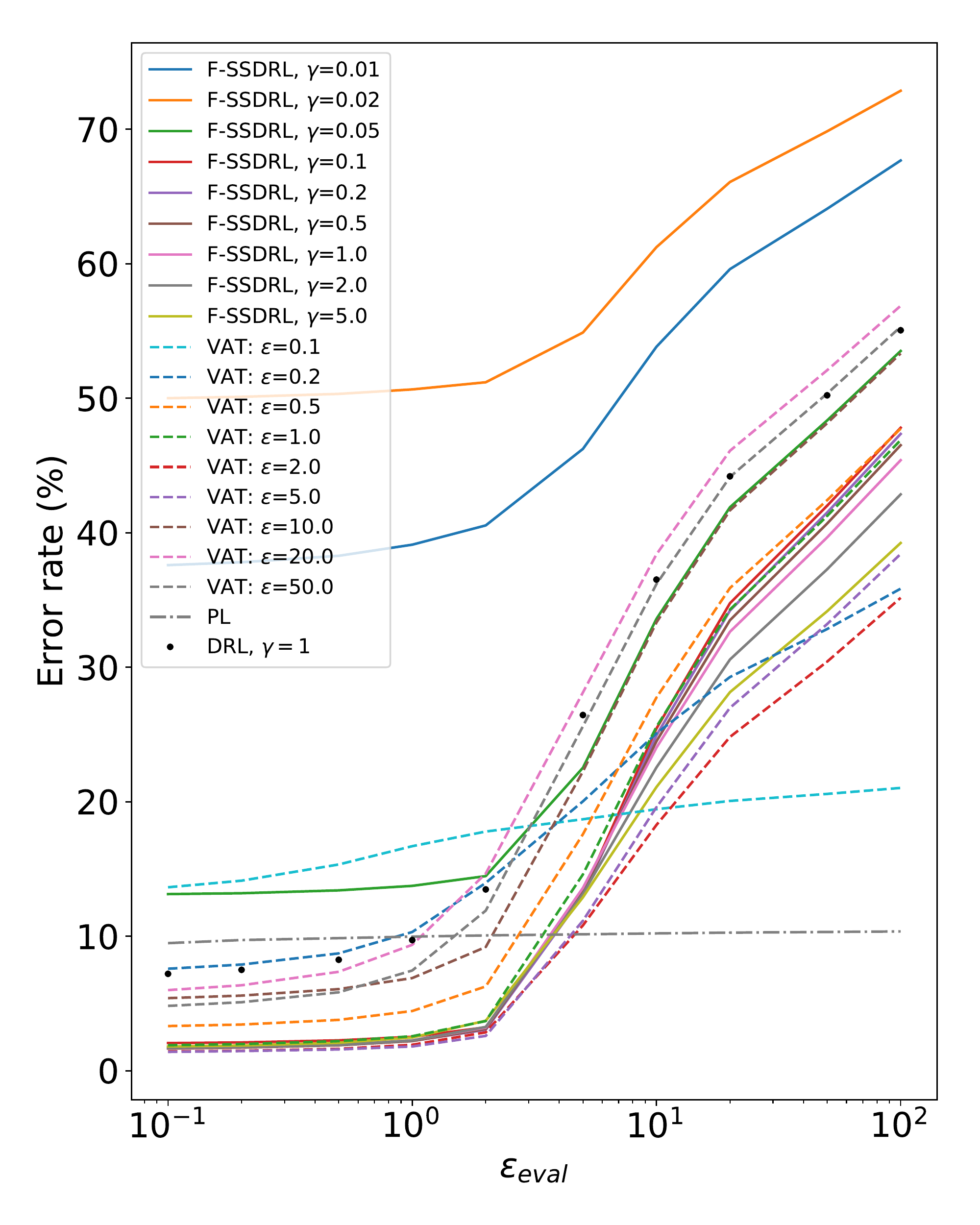}
        \caption{\label{fig:mnist_pgd-l2-adv_val_acc} MNIST}
    \end{subfigure}
    \begin{subfigure}{0.32\textwidth}
		\includegraphics[width=1.0\textwidth]{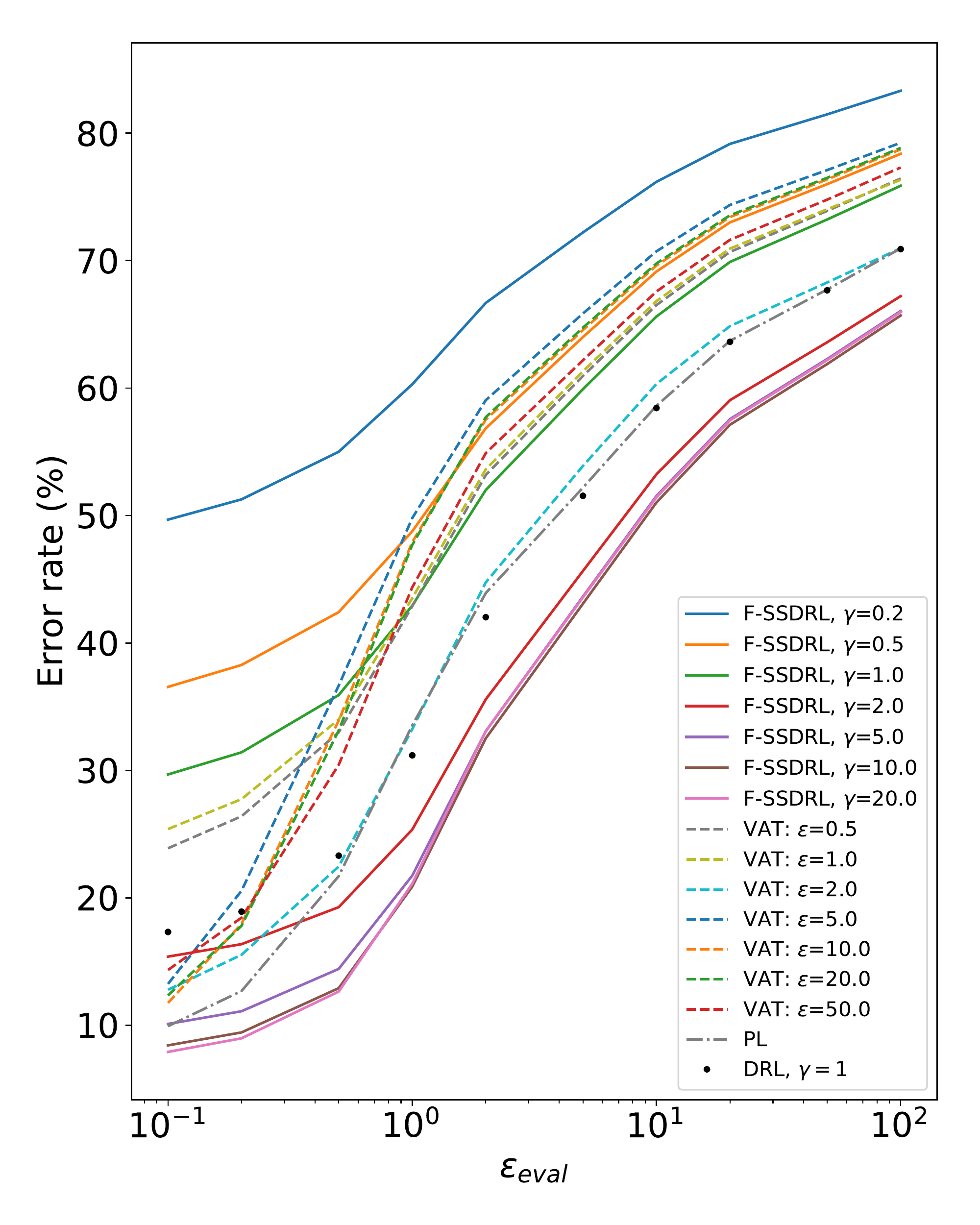}
        \caption{\label{fig:svhn_pgd-l2-adv_val_acc} SVHN}
    \end{subfigure}
    \begin{subfigure}{0.32\textwidth}
		\includegraphics[width=1.0\textwidth]{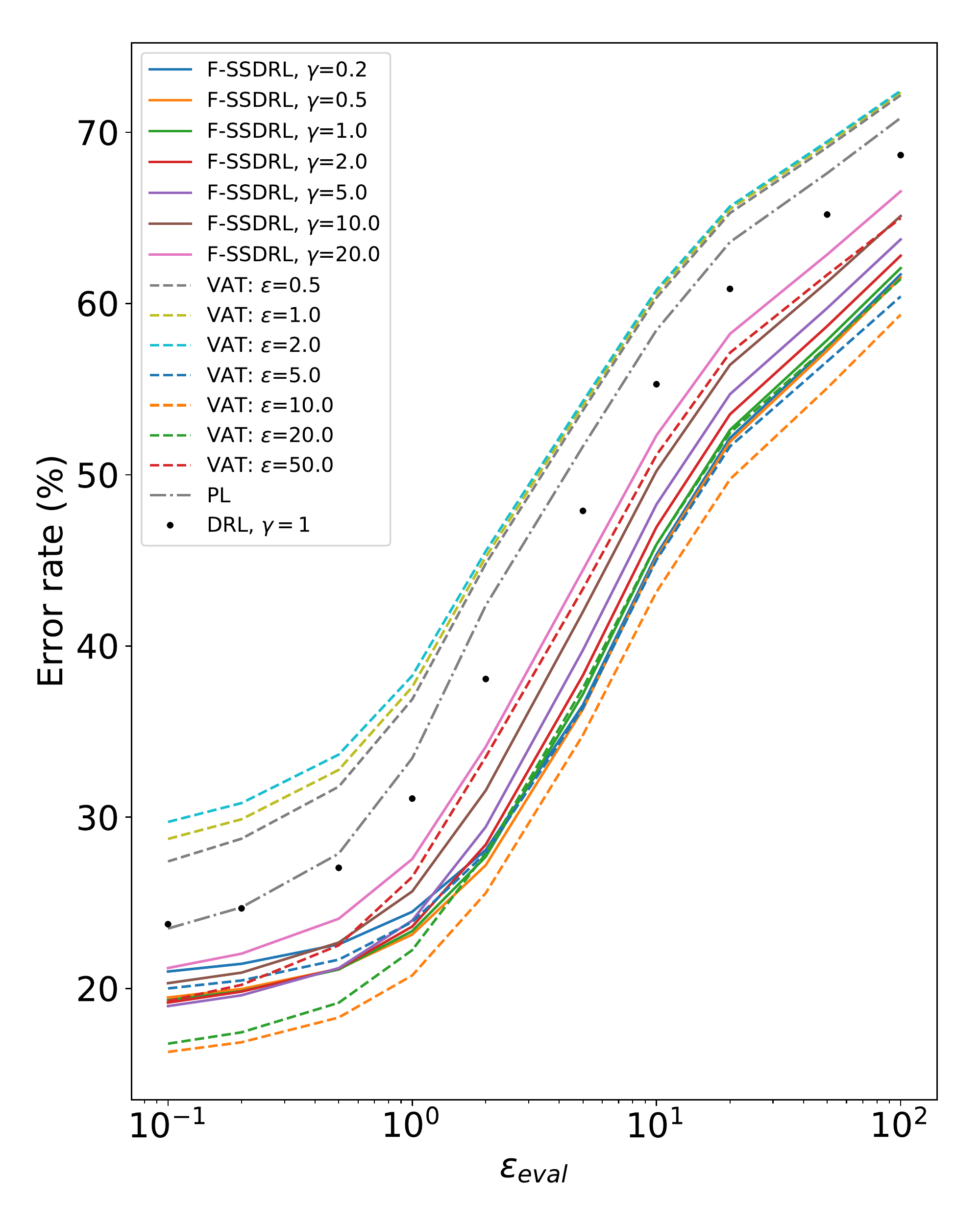}
        \caption{\label{fig:cifar10_pgd-l2-adv_val_acc} CIFAR10}
    \end{subfigure}
    \caption{\label{fig:pgd-l2-adv_val_acc} Comparison of the test error rates on adversarial examples computed by Projected-Gradient Method (PGM)~\cite{madry2017towards} under $\ell_2$ norm constraint.}
\end{figure*}

Figures \ref{fig:clean_drls} and \ref{fig:clean_vat} represent the performance (again in terms of error-rate) over clean examples from different datasets, and for SSDRL and VAT, respectively. In Figure \ref{fig:clean_drls}, different values of $\gamma$ have been used for training and the test error-rate is depicted as a function of $\gamma^{-1}$. Also, $\lambda$ is set to $-1$ for SSDRL. Apparently, SSDRL (or F-SSDRL), for a particular range of parameters, overfits during the training stage on MNIST and as a result its performance is degraded when compared to that of DRL. However, SSDRL outperforms DRL (its fully-supervised counterpart) on SVHN and CIFAR-10 datasets. Also, SSDRL and VAT have comparable performances on clean examples, specifically on SVHN and CIFAR-10 datasets. This observation is in agreement with Table \ref{tab:semisup_wodataaug}.

\begin{figure*}[t]
	\centering
    \begin{subfigure}{0.32\textwidth}
		\includegraphics[width=1.0\textwidth]{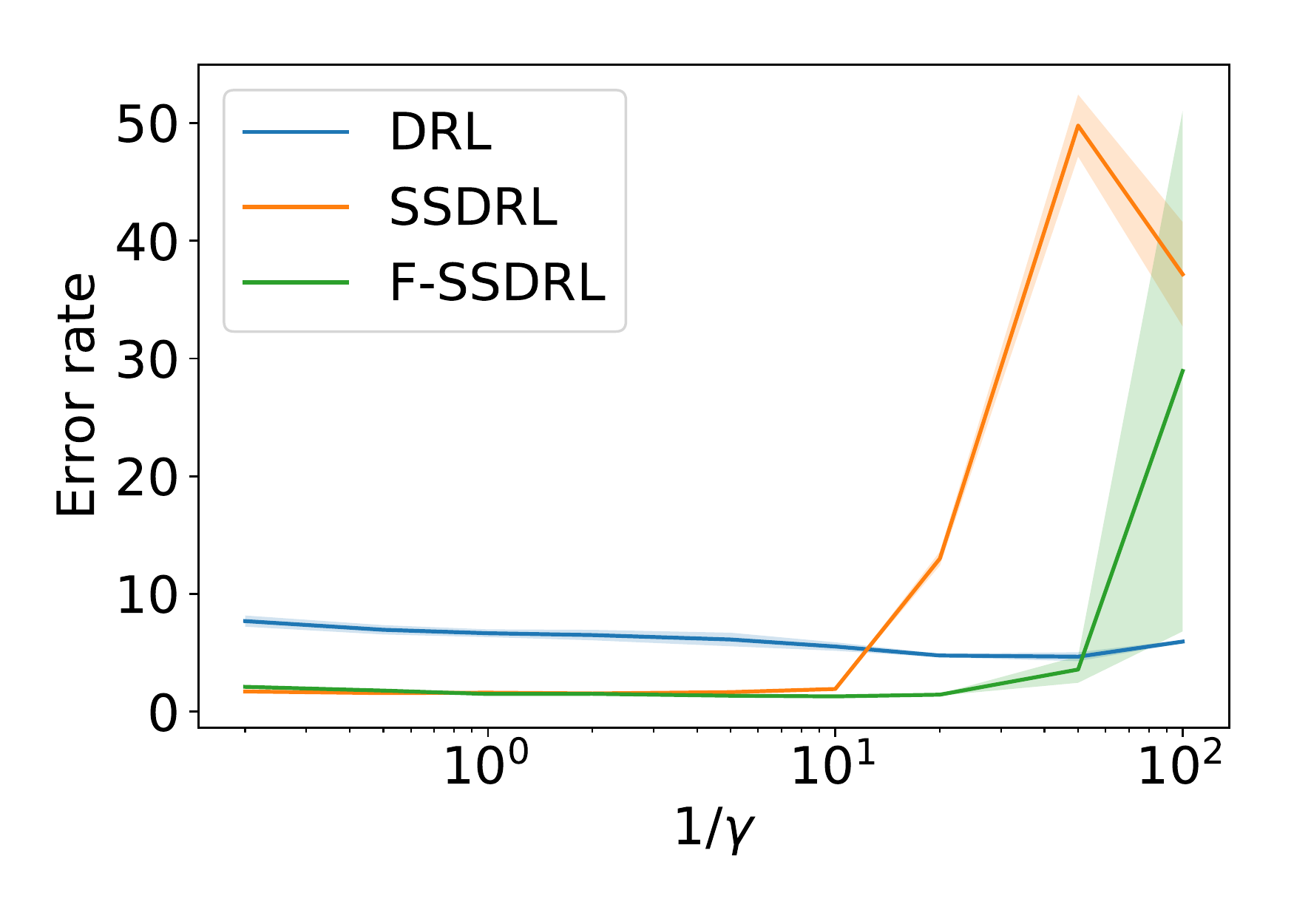}
        \caption{\label{fig:clean_drls_mnist} MNIST}
    \end{subfigure}
    \begin{subfigure}{0.32\textwidth}
		\includegraphics[width=1.0\textwidth]{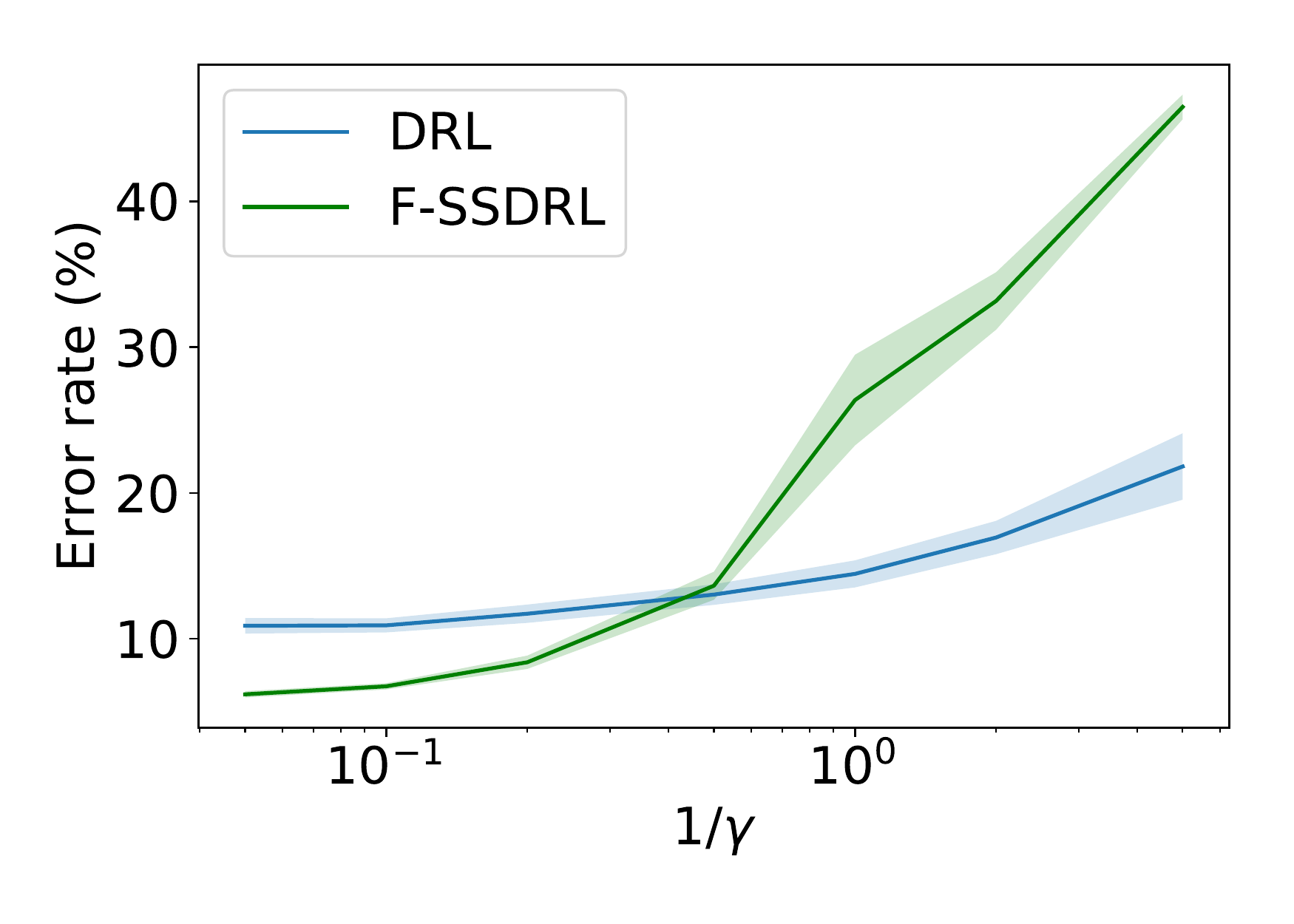}
        \caption{\label{fig:clean_drls_svhn} SVHN}
    \end{subfigure}
    \begin{subfigure}{0.32\textwidth}
		\includegraphics[width=1.0\textwidth]{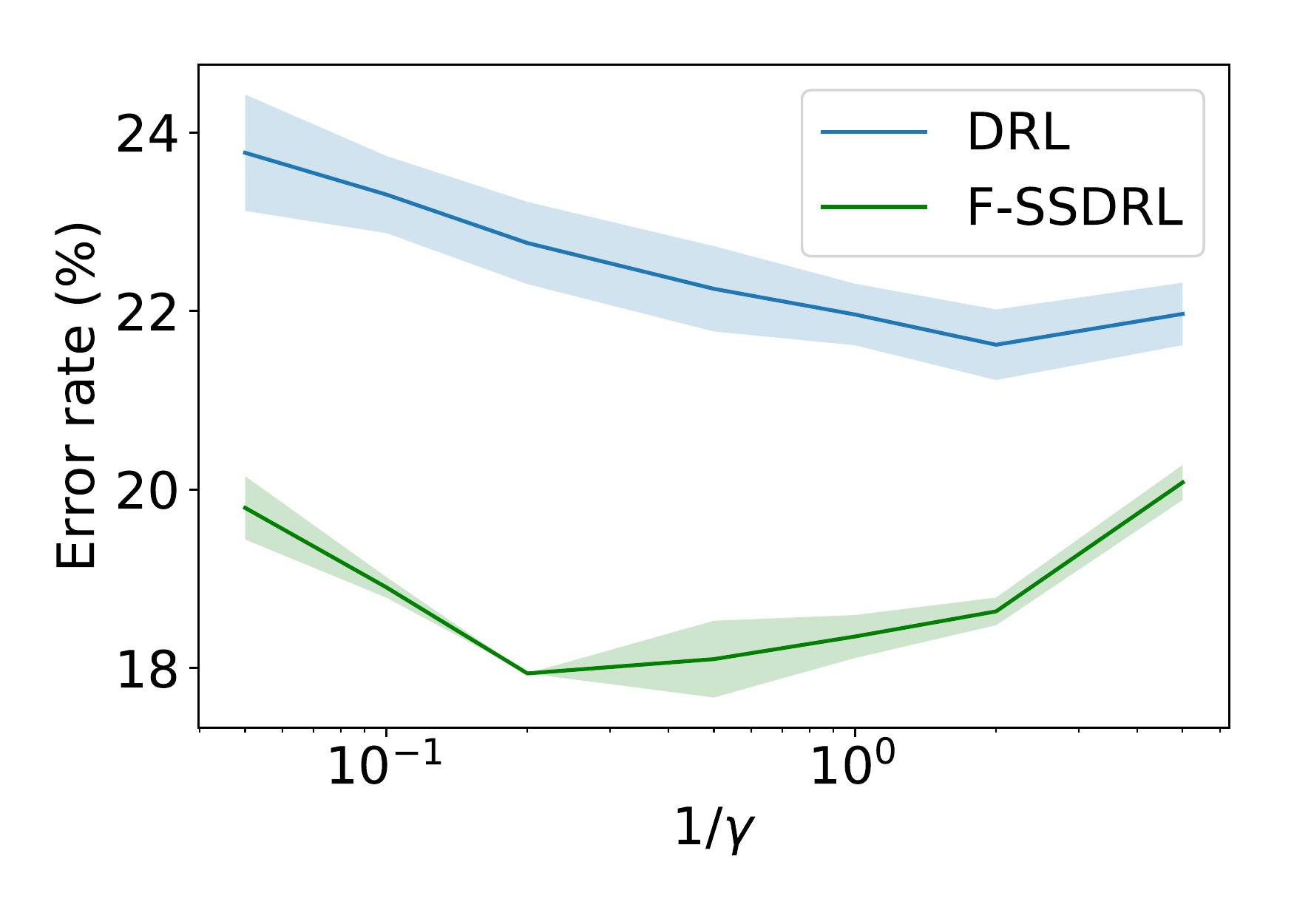}
        \caption{\label{fig:clean_drls_cifar10} CIFAR10}
    \end{subfigure}
    \caption{\label{fig:clean_drls} Test error rates of distributionally robust learning methods on clean examples. The solid lines and shaded regions around them represent the mean and standard deviation of results over multiple random seeds, respectively.}
\end{figure*}

\begin{figure*}[t]
	\centering
    \begin{subfigure}{0.32\textwidth}
		\includegraphics[width=1.0\textwidth]{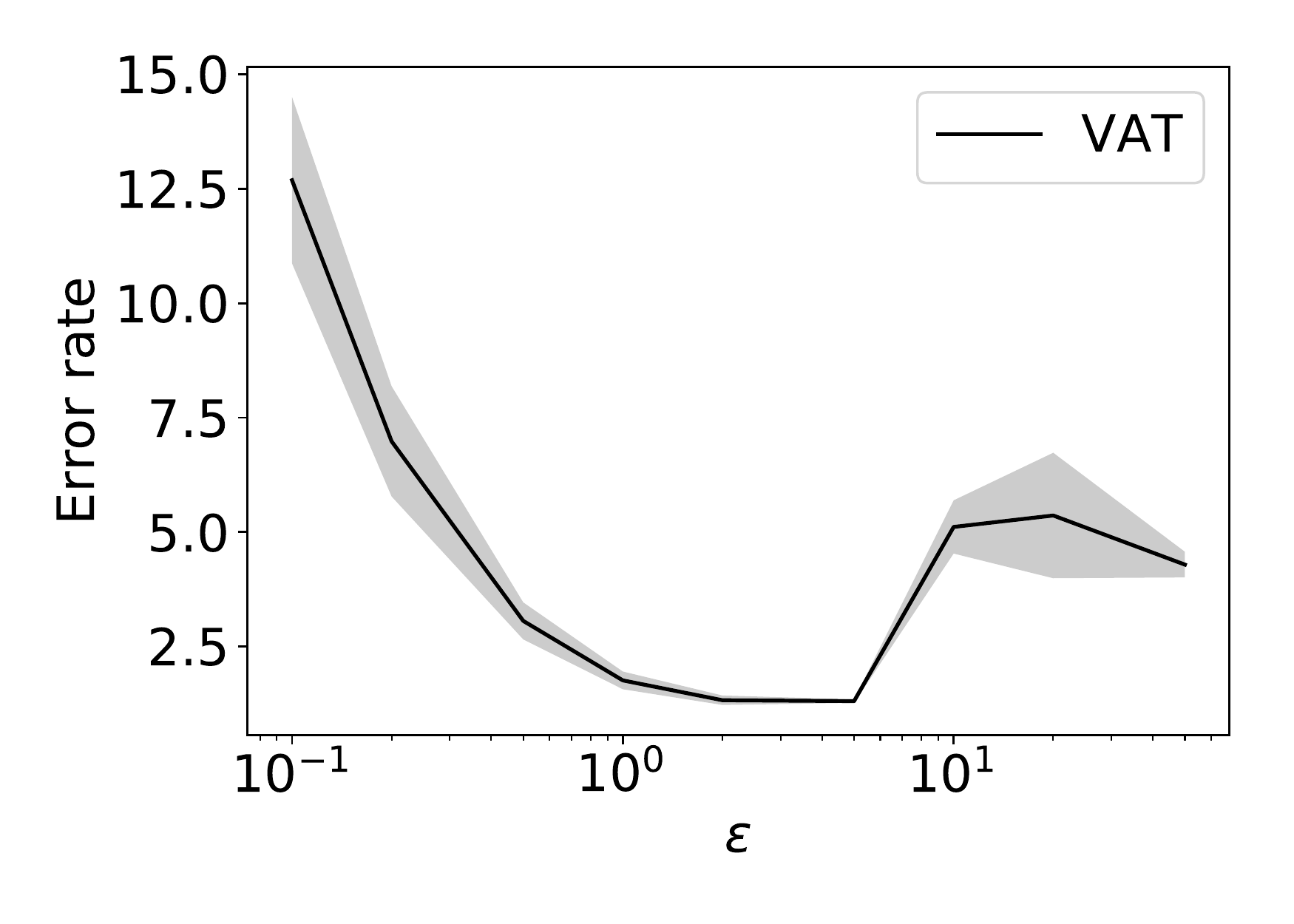}
        \caption{\label{fig:clean_vat_mnist} MNIST}
    \end{subfigure}
    \begin{subfigure}{0.32\textwidth}
		\includegraphics[width=1.0\textwidth]{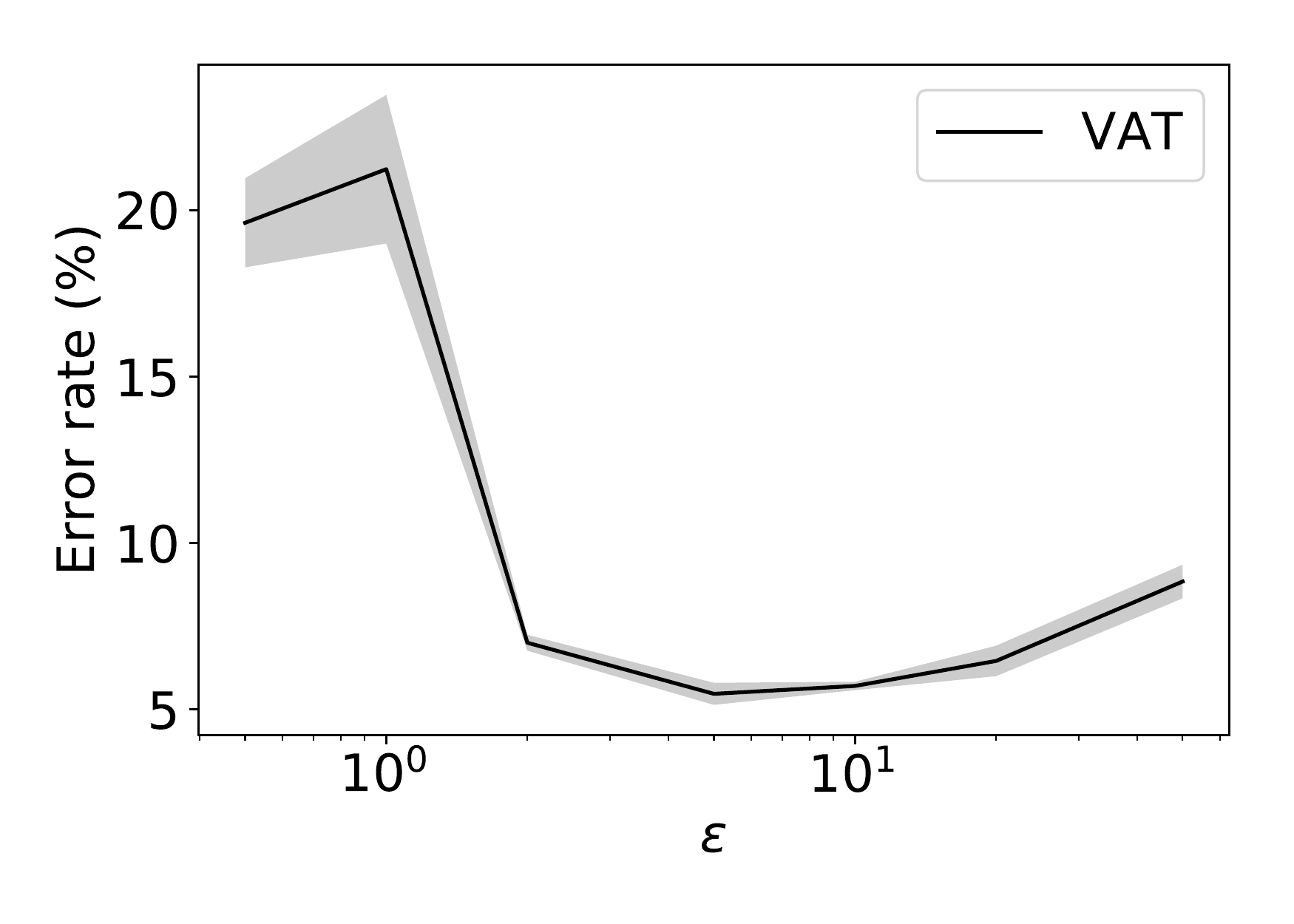}
        \caption{\label{fig:clean_vat_svhn} SVHN}
    \end{subfigure}
    \begin{subfigure}{0.32\textwidth}
		\includegraphics[width=1.0\textwidth]{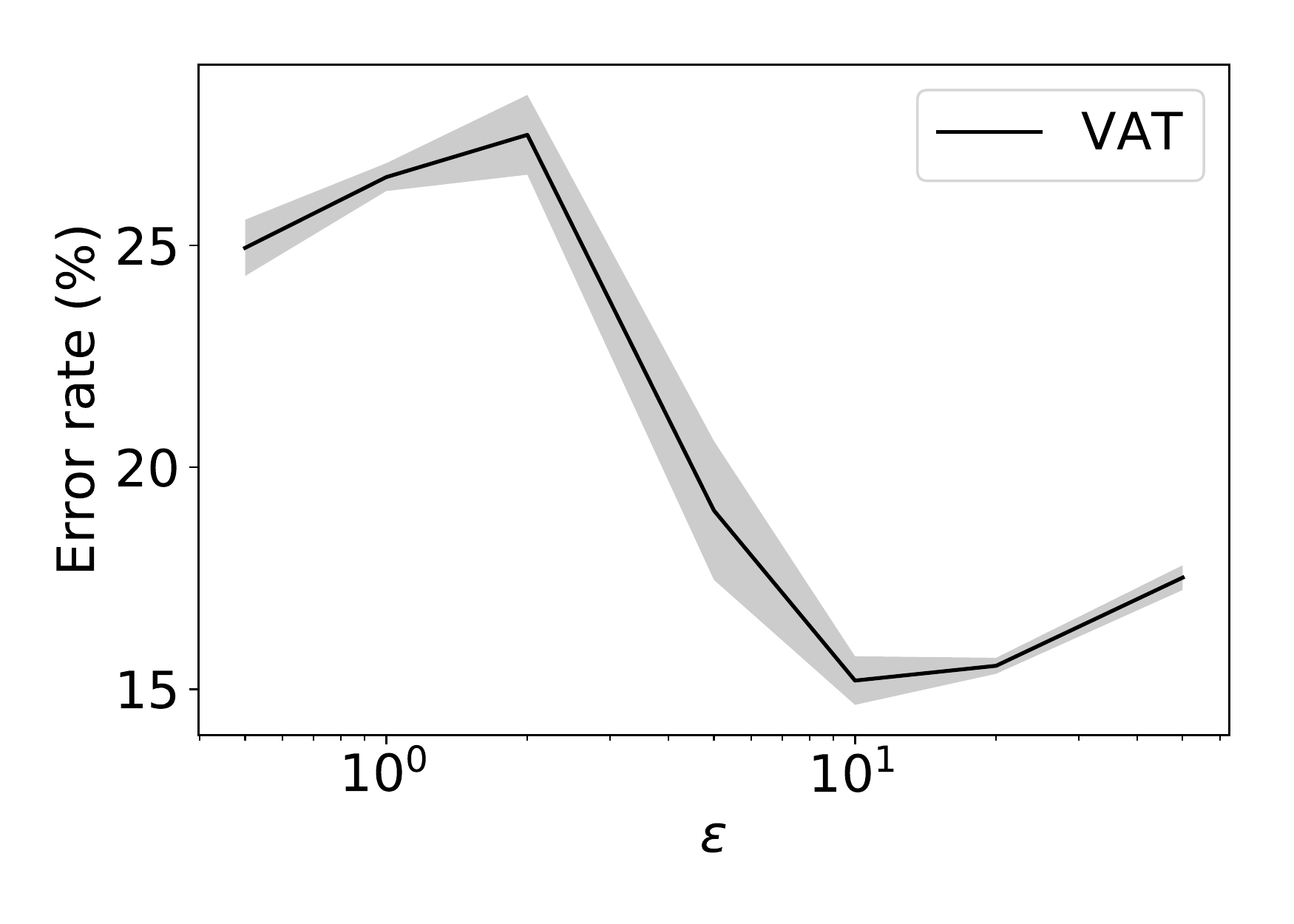}
        \caption{\label{fig:clean_vat_cifar10} CIFAR10}
    \end{subfigure}
    \caption{\label{fig:clean_vat} Test error rates of VAT on clean examples with different $\epsilon$. The solid lines and shaded regions around them represent the mean and standard deviation of results over multiple random seeds, respectively.}
\end{figure*}


\subsection{Experimental Settings}

In this part, we present a detailed description of the experimental settings which have been used for Section \ref{sec:exp}. It should be noted that the majority of the settings used for SVHN and CIFAR-10 datasets follow the same procedure as described in \cite{miyato2018virtual}.

\subsubsection{Real-world Datasets}

Three main datasets have been used during the experiments: MNIST, SVHN  and CIFAR-10. 
\begin{itemize}
\item
The MNIST dataset consists of $28 \times 28$ pixel, gray-scale images of handwritten digits together with their corresponding labels. Each label is a natural number from $0$ to $9$.
The number of training examples and test examples in the dataset are $60,000$ and $10,000$, respectively.
\item
The SVHN dataset consists of $32\times 32 \times 3$ pixel RGB images of street view house numbers with their corresponding labels. Again, labels are natural numbers ranging from $0$ to $9$. The number of training and test samples in the dataset are $73,257$ and $26,032$, respectively. 
\item
CIFAR-10 dataset consists of $32\times 32 \times 3$ pixel RGB images of categorized objects, i.e., cars, trucks, planes, animals, and humans. The number of training examples and test examples in the dataset are $50,000$ and $10,000$, respectively. For CIFAR-10 dataset, we conducted Zero-phase Component Analysis (ZCA) as a pre-processing stage prior to the experiments.
\end{itemize}

\subsubsection{Supervision Ratio and Training Data-points}

In order to create a dataset (training+testing) for the semi-supervised learning task in the paper, we selected a subset of size $1,000$ as the labeled dataset from MNIST and SVHN, while the size goes up to $4,000$ for CIFAR-10. The rest of the samples in the training partition are treated as unlabeled data. We repeated the experiment three times with different choices of labeled and unlabeled data-points on all of the three datasets. For MNIST, a mini-batch of size $64$ is used for both the labeled and unlabeled term, and for SVHN and CIFAR-10, a mini-batch of size $32$ is used for the calculation of the labeled term, while a mini-batch of size $128$ is employed for the unlabeled term during the implementation of each method. We trained each model with $50,000$ updates for MNIST and $48,000$ updates for SVHN and CIFAR10. We have used ADAM optimizer in the training stage. In this regard, the initial learning rate of ADAM is set to $0.001$ and then linearly decayed over the last $10,000$ updates for MNIST, and the last $16,000$ updates for SVHN and CIFAR-10.

As for the transportation cost function $c$, we follow the work presented in \cite{sinha2018certifying} and thus employed the following cost function throughout all our experiments: 
\begin{equation}
\label{eq:actual_cost}
c(\boldsymbol{z}, \boldsymbol{z}') = \Vert\boldsymbol{z}-\boldsymbol{z}'\Vert_2^2 + \infty\cdot \boldsymbol{1}\{y\neq y'\},
\end{equation}
where $\boldsymbol{1}\left(\cdot\right)$ is an indicator function which returns $1$ if its input condition holds and zero, otherwise. It should be noted that this choice is solely for the sake of simplicity, and as described before, every valid lower semi-continuous function is a legitimate choice for $c$.

Also, the {\it {pessimism/optimism~trade-off}} parameter $\lambda$ is always set to $-1$, except when stated otherwise. This option yields certain degrees of optimism during the learning stage, which is motivated by the fact that Deep Neural Networks (DNN) have already proven to work well on all the above-mentioned three datasets. Thus, trusting the learner to assign soft pseudo-labels to the unlabeled data is somehow encouraged which in turn indicates a negative value for $\lambda$.

\subsubsection{Creating Adversarial Examples}

To solve the inner maximization problem in \eqref{eq:phiGammaDef} and \eqref{eq:SSLDROinnermax} for each pair of $\left(\boldsymbol{X},y\right)\in\mathcal{X}\times\mathcal{Y}$, we simply apply {\it {Gradient Ascent}} with the following update rule:
\begin{equation}
\label{eq:update_rule_for_z}
\boldsymbol{X}_{t+1} =  
\boldsymbol{X}_{t} + 
r_t 
\nabla_{\boldsymbol{X}_{t}} \left[\ell((\boldsymbol{X}_{t}, y); \theta) - \gamma c\left(\left(\boldsymbol{X}_{t}, y\right), \left(\boldsymbol{X}, y\right)\right)\right],
\end{equation}
where the initial value $\boldsymbol{X}_0$ is set to $\boldsymbol{X}$, and the ascent rate is defined as $r_t\triangleq\frac{\kappa/\gamma}{(t+1)}$, where $\kappa$ is a hyper-parameter. 
We set $\kappa$ to 1.0 for MNIST and CIFAR-10, and $0.5$ for SVHN. 
During the training, we repeat the update in \eqref{eq:update_rule_for_z} $5$ times for both the DRL and SSDRL method. However, we repeat it $15$ times during the evaluation.

While generating the adversarial examples via the Projected-Gradient Method (PGM), we applied the following update rule which is also used in some previous works in this area \cite{sinha2018certifying,madry2017towards}:
\begin{equation}
\label{eq:pgm}
\boldsymbol{X}_{t+1} =  
\mathrm{Proj}_{\boldsymbol{X},\epsilon}\left(\boldsymbol{X}_{t} + \xi \overline{\nabla_{\boldsymbol{X}_{t}}\ell((\boldsymbol{X}_{t}, y); \theta)}\right),
\end{equation}
where $\mathrm{Proj}_{\boldsymbol{X},\epsilon}$ represents the projection operator to an $\epsilon$-ball (w.r.t. $\ell_2$ norm) centered on $\boldsymbol{X}$. Also, $\bar{\boldsymbol{v}}$ for an arbitrary vector $\boldsymbol{v}$ denotes its normalized version, which is mathematically defined as $\boldsymbol{v}/\Vert\boldsymbol{v}\Vert_2$ under the $\ell_2$-norm constraint.
We have defined the length parameter $\xi$ as $\epsilon / \log (T)$, where $T$ denotes the number of iterations of the update~\eqref{eq:pgm}. Accordingly, we set $T=15$. 

\subsubsection{Architecture of Deep Neural Networks}

A class of Convolutional Neural Networks (CNN) has been used for the loss function set $\mathcal{L}=\left\{\ell\left(\cdot;\theta\right)~,\theta\in\Theta\right\}$. Table~\ref{tab:cnn_models} shows the CNN models used in our experiments. We use ELU~\cite{clevert2015fast} for the activation function in MNIST, and leakyReLU (lReLU)~\cite{maas2013rectifier} for SVHN and CIFAR-10. In the CNNs used for SVHN and CIFAR-10, all the convolutional layers as well as the fully connected (or equivalently dense) layers are followed by batch normalization~\cite{ioffe2015batch}, except for the fully connected layer on CIFAR-10. The slopes of all lReLU in the network are set to $0.1$.

\begin{SCtable}[1][t]
	\centering
	\caption{\label{tab:cnn_models} CNN models used in our experiments. The deep structures that have been used for SVHN and CIFAR-10 datasets are different from the one used for MNIST. The specifications that correspond to each structure are inspired from \cite{miyato2018virtual}.}
    \begin{subtable}{.33\textwidth}
        \centering
        \caption{For MNIST}
		\begin{tabular}{c}
		\toprule
		    28$\times$28 gray-scale image  \\
		\midrule	
            4$\times$4 conv. stride 2, 64 ELU 	\\
            4$\times$4 conv. stride 2, 64 ELU 	\\
            4$\times$4 conv. stride 2, 64 ELU	\\
        \midrule
            global average pool \\
        \midrule   
            dense 64 $\rightarrow$ 10 \\
        \midrule
		    10-way softmax  \\
		\bottomrule
    	\end{tabular}
    \end{subtable}
    \begin{subtable}{.33\textwidth}
        \centering
        \caption{For SVHN and CIFAR-10}
		\begin{tabular}{c}
		\toprule
		    32$\times$32 RGB image  \\
		\midrule	
            3$\times$3 conv. 128 lReLU 	\\
            3$\times$3 conv. 128 lReLU 	\\
            3$\times$3 conv. 128 lReLU	\\
		\midrule
		    2$\times2$ max-pool, stride 2  \\
		    dropout, $p=0.5$  \\
        \midrule   
            3$\times$3 conv. 256 lReLU 	\\
            3$\times$3 conv. 256 lReLU 	\\
            3$\times$3 conv. 256 lReLU 	\\
        \midrule  
		    2$\times2$ max-pool, stride 2 \\
		    dropout, $p=0.5$  \\
        \midrule 
            3$\times$3 conv. 512 lReLU 	\\
            1$\times$1 conv. 256 lReLU 	\\
            1$\times$1 conv. 128 lReLU 	\\
        \midrule
           global average pool \\
        \midrule   
            dense 128$\rightarrow$ 10 \\
        \midrule
		    10-way softmax  \\
		\bottomrule
    	\end{tabular}
    \end{subtable}
\end{SCtable}

\section{Minimum Supervision Ratio: Definition and Implications}
\label{sec:proposed:asymptotic}

In this section, we present some complementary discussions with respect to our generalization bound in Section \ref{sec:proposed:general}. In particular, the mathematical definition and intuitive implications behind one of our proposed complexity measures, i.e. the Minimum Supervision Ratio, are explained in details.

In order to better understand the intuition behind the proposed optimization programs in \eqref{eq:main2} or \eqref{eq:SSLmainMin}, it is necessary to investigate them under the asymptotic regime of $n\rightarrow\infty$. In this regard, this section provides a rigorous mathematical framework to study the semi-supervised learning in general (and its distributionally robust extension in particular), under the specific problem setting of this paper. We then provide conditions on the hypothesis set and data-generating distribution, under which unlabeled data can help the overall learning procedure. Final bounds on the performance improvement through incorporation of unlabeled samples (which is mostly from the generalization aspect), are given with mathematical details in Theorem \ref{thm:generalBound1} and its proof. In order to achieve the above-mentioned goal, first let us make the following definition:
\\[-3mm]
\begin{definition}
For a feature space $\mathcal{X}$ and a finite label set $\mathcal{Y}$,
the conditional composition of a distribution $P\in M\left(\mathcal{X}\times\mathcal{Y}\right)$ with a conditional distribution $\Omega\in M^{\mathcal{X}}\left(\mathcal{Y}\right)$ through a supervision ratio of $0\leq\eta\leq1$, denoted by $\mathrm{comp}\left(P,\Omega,\eta\right)\in M\left(\mathcal{X}\times\mathcal{Y}\right)$, is defined as
\begin{equation}
\mathrm{comp}\left(P,\Omega,\eta\right)\left(\boldsymbol{X},y\right)
\triangleq
\eta P\left(\boldsymbol{X},y\right)+
\left(1-\eta\right)\Omega\left(y\vert\boldsymbol{X}\right)\left(\sum_{y'\in\mathcal{Y}}P\left(\boldsymbol{X},y'\right)\right).
\end{equation}
\end{definition}
It can be easily verified that the following properties hold for the conditional composition distribution of any two corresponding distributions:
\begin{equation}
\mathrm{comp}\left(P,\Omega,\eta\right)_{\boldsymbol{X}}=P_{\boldsymbol{X}}
\quad,\quad
\mathrm{comp}\left(P,\Omega,\eta\right)_{\vert\boldsymbol{X}}=
\eta P_{\vert\boldsymbol{X}}+
\left(1-\eta\right) \Omega_{\vert\boldsymbol{X}},
\end{equation}
where the first relation means: the marginal of the composition distribution w.r.t. $\boldsymbol{X}$ (which is a measure supported on $\mathcal{X}$) is the same as that of $P$, while the second property states that:  conditional distribution over $\mathcal{Y}$ (given $\boldsymbol{X}\in\mathcal{X}$) is a weighted mixture of conditional distributions $P_{\vert\boldsymbol{X}}$ and $\Omega_{\vert\boldsymbol{X}}$. 

An interesting asymptotic property of a consistent distribution set (see Definition \ref{def:consistent}) is that, given both fully and partially-observed samples in $\boldsymbol{D}$ are i.i.d. samples generated from a single arbitrary distribution $P_0\in M\left(\mathcal{X}\times\mathcal{Y}\right)$, the following relation holds almost surely w.r.t. $P_0$:
\begin{equation}
\lim_{n\rightarrow\infty}\hat{\mathcal{P}}\left(\boldsymbol{D}\right)
\stackrel{a.s.}{=}
\left\{
\mathrm{comp}\left(P_0,\Omega,\eta=\lim_{n\rightarrow\infty}\frac{n_{\mathrm{l}}}{n}\right)
\bigg\vert
\Omega\in M^{\mathcal{X}}\left(\mathcal{Y}\right)
\right\},
\end{equation}
where the asymptotic equality in the above relation corresponds to a member-wise convergence between the two sets. Consequently, rewriting \eqref{eq:SSLmainMin} in the asymptotic regime of $n\rightarrow\infty$ would give us the following equalities:
\begin{align}
\label{eq:BIASdef}
\lim_{n\rightarrow\infty}\hat{R}_{\mathrm{SSAR}}\left(\theta;\boldsymbol{D}\right)
&\stackrel{a.s.}{=}
\mathbb{E}_{P_0}\left\{\hat{R}_{\mathrm{SSAR}}\left(\theta;\boldsymbol{D}\right)\right\}
\\
&=
\eta \mathbb{E}_{\left(\boldsymbol{X},y\right)\sim P_0}\left\{\phi_{\gamma}\left(\boldsymbol{X},y;\theta\right)\right\}+
\left(1-\eta\right)
\mathbb{E}_{\boldsymbol{X}\sim P_{0_{\boldsymbol{X}}}}\left\{
\softmin_{y\in\mathcal{Y}}^{\left(\lambda\right)}\left\{
\phi_{\gamma}\left(\boldsymbol{X},y;\theta\right)
\right\}
\right\}.
\nonumber
\end{align}
The first term in the r.h.s. of \eqref{eq:BIASdef} is proportional to the true risk which we intend to bound. However, the second term models the asymptotic effect of unlabeled data for a fixed supervision ratio $\eta$. The main question that we try to answer in this section can be intuitively stated as: under what conditions, the second term becomes {\it {approximately proportional}} to the true risk as well?

Before investigating the above question in more theoretical details, a closer look at the semi-supervised adversarial risk $\hat{R}_{\mathrm{SSAR}}$ reveals that
\begin{equation}
\frac{\partial}{\partial\lambda}\hat{R}_{\mathrm{SSAR}}\left(\theta;\boldsymbol{D}\right)\ge0.
\end{equation}
This fact implies that by decreasing $\lambda$, one can also decrease $\hat{R}_{\mathrm{SSAR}}$ (at least in the majority of non-trivial scenarios). This issue has been previously mentioned in Section \ref{sec:proposed}, which indicates that {\it {optimism}} always results in lower empirical risks. But how does this strategy affect the true expected loss, i.e. 
$\mathbb{E}_{P_0}\left\{\phi_{\gamma}\left(\boldsymbol{Z};\theta\right)\right\}$? On the other hand, moving $\lambda$ toward $+\infty$ guarantees that the learner is minimizing a legitimate upper-bound of the true risk, i.e. extreme pessimism, however, this also increases the empirical risk. Again, one could ask is it really necessary to be so pessimistic? 

In order to answer the above questions, we introduce a new compatibility measure function for a function set ${\Phi}\subset\mathbb{R}^{\mathcal{X}\times\mathcal{Y}}$ and distribution $P_0$, denoted by {\it {minimal supervision ratio}} or $\mathrm{MSR}_{\left(\Phi,P_0\right)}:\mathbb{R}\times\mathbb{R}_{\ge0}\rightarrow\left[0,1\right]$. We then show that as long as a particular inequality holds among parameters such as $n$, $\lambda$ and $\eta$ according to $\mathrm{MSR}_{\left(\Phi,P_0\right)}$, one can guarantee minimizing a valid upper-bound for the true risk, while avoiding the extreme pessimism of \cite{loog2016contrastive} (less harm to the empirical risk minimization). In order to do so, first let us introduce a number of useful additional tools:
\\[-3mm]
\begin{definition}
\label{def:properness}
Assume function class $\Phi\subseteq \mathbb{R}^{\left(\mathcal{X}\times\mathcal{Y}\right)}$ and  distribution $P_0\in M\left(\mathcal{X}\times\mathcal{Y}\right)$ for a finite label-set $\mathcal{Y}$. For the ease of notation, let
$\phi_{\boldsymbol{X}}\triangleq\phi\left(\boldsymbol{X},\cdot\right)\in\mathbb{R}^{\mathcal{Y}}$ for $\forall\boldsymbol{X}\in\mathcal{X}$. Then, $\rho_{\lambda}\left(\phi\right)$ for $\phi\in\Phi$ and $\lambda\in\mathbb{R}\cup\left\{\pm\infty\right\}$ is defined as
\begin{equation}
\rho_{\lambda}\left(\phi\right)
\triangleq
\mathbb{E}_{P_{0_{\boldsymbol{X}}}}\left\{\softmin_{y\in\mathcal{Y}}^{\left(\lambda\right)}
\left\{
\phi_{\boldsymbol{X}}
\right\}
\right\}-
\mathbb{E}_{P_0}\left\{\phi\right\}.
\label{eq:properness}
\end{equation}
\end{definition}
As it becomes evident in the proceeding arguments of this section, the introduced functional in Definition \ref{def:properness}, i.e. $\rho_{\lambda}$, plays an important role in determining the relation of expected (or asymptotic) semi-supervised risk with the true (supervised) one. Mathematically speaking, enforcing $\rho_{\lambda}\left(\phi\right)$ for $\phi=\phi_{\gamma}\left(\cdot;\theta\right)$ to remain non-negative guarantees that $\mathbb{E}_{\boldsymbol{D}\sim P_0}\left\{\hat{R}_{\mathrm{SSAR}}\left(\theta;\boldsymbol{D}\right)\right\}\ge \mathbb{E}_{P_0}\left\{\phi_{\gamma}\left(\cdot;\theta\right)\right\}$ for any $\theta\in\Theta$. This allows us to upper-bound the true risk with the value of $\hat{R}_{\mathrm{SSAR}}$ computed for that particular $\theta$. Surprisingly, this condition can always be satisfied by choosing $\lambda = +\infty$ (extreme pessimism). This configuration, in the special non-robust case, coincides with the framework presented in \cite{loog2016contrastive}.
\\[-3mm]
\begin{lemma}
\label{lemma:propernessLB}
For any function set $\Phi\subseteq \mathbb{R}^{\left(\mathcal{X}\times\mathcal{Y}\right)}$ and distribution $P_0\in M\left(\mathcal{X}\times\mathcal{Y}\right)$, we have $\rho_{\infty}\left(\phi\right)\ge0$ for all $\phi\in\Phi$.
\end{lemma}
\begin{proof}
$P_{0_{\vert\boldsymbol{X}}}$ is a distribution over $\mathcal{Y}$, thus can be considered as a vector in a simplex, i.e. all components are non-negative and sum up to one. Then, the lemma's argument can be justified by the fact that
\begin{equation}
\left\langle \phi_{\boldsymbol{X}}
\big\vert
P_{0_{\vert\boldsymbol{X}}}\right\rangle
\leq
\max_{y\in\mathcal{Y}}~\phi_{\boldsymbol{X}},
\quad\mathrm{while}\quad
\softmin^{\left(\infty\right)}_{y\in\mathcal{Y}}\left\{
\phi_{\boldsymbol{X}}
\right\}=
\max_{y\in\mathcal{Y}}~\phi_{\boldsymbol{X}},
\end{equation}
where $\langle\cdot\vert\cdot\rangle$ denotes the inner product. More precisely, one can write:
\begin{align*}
\rho_{\infty}\left(\phi\right)&=
\mathbb{E}_{P_{0_{\boldsymbol{X}}}}\left\{
\softmin^{\left(\infty\right)}_{y\in\mathcal{Y}}\left\{
\phi_{\boldsymbol{X}}
\right\}
\right\}
-\mathbb{E}_{P_0}\left\{
\phi
\right\}
\\
&=\mathbb{E}_{P_{0_{\boldsymbol{X}}}}\left\{
\softmin^{\left(\infty\right)}_{y\in\mathcal{Y}}\left\{
\phi_{\boldsymbol{X}}
\right\}-
\left\langle \phi_{\boldsymbol{X}}
\big\vert
P_{0_{\vert\boldsymbol{X}}}\right\rangle
\right\}\ge0.
\end{align*}
The last inequality is a direct result of the fact that inside of the expectation operator is non-negative. This completes the proof. 
\end{proof}
However, we are more interested in those cases where $\lambda$ can be bounded, or even negative, while $\rho_{\lambda}$ is still non-negative in {\it {some regions}} of $\Phi$. The main problem is that the minimizer of \eqref{eq:SSLmainMin} (semi-supervised empirical risk) must fall in {\it {those regions}}, as well. Otherwise one cannot upper-bound the true risk by minimizing \eqref{eq:SSLmainMin}. Mathematically speaking, assume $\Phi\triangleq\left\{\phi_{\gamma}\left(\cdot;\theta\right):\mathcal{Z}\rightarrow\mathbb{R}\vert~\theta\in\Theta\right\}$ as described in \eqref{eq:SSLmainMin}. Then, we are interested to see if there exists a non-empty subset of $\Phi$, say $\psi$, such that:
\begin{equation}
\exists\psi\subseteq\Phi\bigg\vert~
\argmin_{\phi\in\Phi}~\hat{R}_{\mathrm{SSAR}}\left(\phi;\boldsymbol{D}\right)\in\psi
\quad\mathrm{and}\quad
\rho_{\lambda}\left(\phi\right)\ge0,~\forall\phi\in\psi.
\end{equation}
We give a theoretical solution for the non-trivial case of the above-mentioned problem ($\lambda<+\infty$). This way, one can still choose small (or generally negative) values of $\lambda$, which substantially lower the empirical loss and improve the generalization bound. The following definitions provide us with more generalized means to achieve this goal.
\\[-3mm]
\begin{definition}
\label{def:GAP_PI}
Assume the function set $\Phi\subseteq{\mathbb{R}}^{\left(\mathcal{X}\times\mathcal{Y}\right)}$, probability distribution $P_0\in M\left(\mathcal{X}\times\mathcal{Y}\right)$, and let us define $\phi^*=\argmin_{\phi\in\Phi}~\mathbb{E}_{P_0}\left\{\phi\left(\boldsymbol{X},y\right)\right\}$. Let $\psi\subseteq\Phi$ to denote a subset of functions in $\Phi$. Then, the loss gap functional ${\mathrm{GAP}}\left(\psi\right)$, and $\Gamma\left(\psi;\lambda\right)$ for $\lambda\in\mathbb{R}\cup\left\{\pm\infty\right\}$ w.r.t. $P_0$ and $\Phi$ are defined as
\begin{align}
{\mathrm{GAP}}\left(\psi\right)\triangleq
\inf_{\phi\in\Phi-\psi}\mathbb{E}_{P_0}\left\{\phi-\phi^*\right\}
\ge0
\quad,\quad
\Gamma\left(\psi;\lambda\right)
\triangleq
\inf_{\phi\in\Phi-\psi}
~\rho_{\lambda}\left(\phi\right)
-\rho_{\lambda}\left(\phi^*\right).
\end{align}
For the special case of $\psi=\Phi$, we define $\mathrm{GAP}\left(\Phi\right)=\infty$ and $\Gamma\left(\Phi;\lambda\right)=0$, respectively. Also, let us define $\Lambda:2^\Phi\rightarrow\mathbb{R}\cup\left\{\pm\infty\right\}$ as
\begin{align}
\Lambda\left(\psi\right)\triangleq&~\inf_{\lambda\in\mathbb{R}\cup\left\{\pm\infty\right\}}~\lambda
\nonumber\\
&\mathrm{subject~to}\quad
\rho_{\lambda}\left(\phi\right)\ge0,\quad
\forall\phi\in\psi.
\end{align} 
\end{definition}
All the functionals $\mathrm{GAP}$, $\Gamma$ and $\Lambda$ are defined to enable us to capture the properties of a hypothesis set $\Phi$ and a corresponding data distribution $P_0$, inside arbitrary subsets of $\Phi$. Another interesting attribute is that $\mathrm{GAP}$ and $\Lambda$ are not functions of $\lambda$, and correspond to the fundamental features of the pair $\left(\Phi,P_0\right)$ in a fully-supervised sense. Note that due to Lemma \ref{lemma:propernessLB}, $\Lambda\left(\psi\right)$ is always well-defined, since its corresponding feasible set cannot be empty. This way, we can present the most important definition in this section, which is the key to provide the generalization bounds derived in Theorem \ref{thm:generalBound1} for general semi-supervised learning via self-learning.
\\[-3mm]
\begin{definition}[Minimum Supervision Ratio]
\label{def:compatibility}
Assume function set $\Phi\subseteq\mathbb{R}^{\left(\mathcal{X}\times\mathcal{Y}\right)}$ and distribution $P_0\in M\left(\mathcal{X}\times\mathcal{Y}\right)$ for a feature space $\mathcal{X}$ and finite label set $\mathcal{Y}$. Then, the minimum supervision ratio function, $\mathrm{MSR}_{\left(\Phi,P_0\right)}:\mathbb{R}\cup\left\{\pm\infty\right\}\times\mathbb{R}_{\ge0}\rightarrow\left[0,1\right]$, is defined as
\begin{equation}
\mathrm{MSR}_{\left(\Phi,P_0\right)}\left(\lambda,\zeta\right)\triangleq
\inf_{\quad~\psi\subseteq\Phi\vert\Lambda\left(\psi\right)\leq\lambda}~h\left(
1-\frac{\mathrm{GAP}\left(\psi\right)-\zeta}
{u\left(-\Gamma\left(\psi;\lambda\right)\right)}
\right),
\end{equation}
for $\lambda\in\mathbb{R}\cup\left\{\pm\infty\right\}$ and $\zeta\ge0$, where $u:\mathbb{R}\rightarrow\mathbb{R}$ denotes the ramp function, i.e. $u\left(x\right)=x,~x\ge0$ and $0$ otherwise, and $h\left(\cdot\right)\triangleq\min\left\{1,u\left(\cdot\right)\right\}$. Also, let $\mathrm{MSR}_{\left(\Phi,P_0\right)}\left(\lambda,\zeta\right)=1$, in case the feasible set $\Lambda\left(\psi\right)\leq\lambda$ is empty for an input $\lambda$.
\end{definition}
$\mathrm{MSR}_{\left(\Phi,P_0\right)}$ is a learning-theoretic attribute of the pair $\left(\Phi,P_0\right)$, and also a central ingredient of Theorem \ref{thm:generalBound1}. It has the following properties: First,
$\mathrm{MSR}_{\left(\Phi,P_0\right)}\left(\lambda,\zeta\right)$ is an increasing function w.r.t. $\zeta$, and decreasing w.r.t. $\lambda$, for all $\Phi$ and $P_0$.
Second,
for all $\Phi$ and $P_0$, there exist $\lambda\in\mathbb{R}\cup\left\{\pm\infty\right\}$ and $\zeta\ge0$ such that $\mathrm{MSR}_{\left(\Phi,P_0\right)}\left(\lambda,\zeta\right)=0$ (see Lemma \ref{lemma:compExistLemma} below).
\\[-3mm]
\begin{lemma}[Compatibility Guarantee]
\label{lemma:compExistLemma}
For any function set $\Phi$ and a corresponding probability distribution $P_0$, there exist $\lambda\in\mathbb{R}\cup\left\{\pm\infty\right\}$ and $\zeta\ge0$ such that $\mathrm{MSR}_{\left(\Phi,P_0\right)}\left(\lambda,\zeta\right)=0$.
\end{lemma}
\begin{proof}
By simple mathematical manipulations, it can be easily verified that 
\begin{align}
&\mathrm{MSR}_{\left(\Phi,P_0\right)}\left(\lambda,\zeta\right)=
\inf_{\psi\subseteq\Phi\vert~\Lambda\left(\psi\right)\leq\lambda}~h\left(
1-\frac{\mathrm{GAP}\left(\psi\right)-\zeta}
{u\left(-\Gamma\left(\psi;\lambda\right)\right)}
\right)=0,
\nonumber\\
\Rightarrow\quad&\exists\lambda\in\mathbb{R}\cup\left\{\pm\infty\right\}\bigg\vert~
\sup_{\psi\subseteq\Phi\vert~\Lambda\left(\psi\right)\leq\lambda}\mathrm{GAP}\left(\psi\right)+\Gamma\left(\psi;\lambda\right)\ge\zeta.
\end{align}
In this regard, in order to prove the lemma one can alternatively try to show that there exists $\zeta\ge0$, such that
\begin{equation}
\sup_{\lambda\in\mathbb{R}\cup\left\{\pm\infty\right\}}~
\sup_{\psi\subset\Phi\vert~\Lambda\left(\psi\right)\leq\lambda}~
{\mathrm{GAP}\left(\psi\right)}+
\Gamma\left(\psi;\lambda\right)\ge\zeta.
\label{eq:totalProperness}
\end{equation}
Note that $\mathrm{GAP}\left(\psi\right)\ge0$ based on the definition, and for all $\psi\subseteq\Phi$. Moreover, according to assumption there exist $\psi^*\subset\Phi$, such that $\mathrm{GAP}\left(\psi^*\right)>0$. Let us define $\Gamma^*$ as
\begin{equation}
\Gamma^*\triangleq
\sup_{\lambda\in\mathbb{R}\cup\left\{\pm\infty\right\}}~
\sup_{\psi\subseteq\Phi\vert~\Lambda\left(\psi\right)\leq\lambda}
\Gamma\left(\psi;\lambda\right).
\end{equation}
It is easy to see that $\Gamma^*\ge0$, since $\psi=\Phi-\phi^*$ and  $\lambda\ge\Lambda\left(\Phi-\phi^*\right)$ lead to $\Gamma\left(\psi,\lambda\right)=0$. The rest of the proof can be divided into two separate parts, based on the assumptions on the value of $\Gamma^*$ w.r.t. function set $\Phi$, and probability distribution $P_0$. First, assume $\Gamma^*>0$. Then, it can be easily checked that there exists $\zeta>0$, $\lambda\in\mathbb{R}$ and $\psi\subset\Phi$ such that for any $\eta\in\left[0,1\right]$:
\begin{equation}
\mathrm{GAP}\left(\psi\right)+\left(1-\eta\right)\Gamma\left(\psi;\lambda\right)-\zeta\ge0.
\end{equation}
In the second regime, we assume $\Gamma^*=0$. This very special case indicates a highly incompatible pair of hypothesis set $\Phi$ and distribution $P_0$. In simple words, it means there are functions such as $\phi_{\mathrm{inc}}\in\Phi$, so $\phi_{\mathrm{inc}}$ is highly correlated with label-conditional distribution $P_{0_{\vert\boldsymbol{X}}}$, in an expected sense. Therefore, it produces large expected loss values, while it can easily fool the learner during the pseudo-labeling procedure (for example, by assigning very small loss values for some irrelevant labels). In this case, $\lambda=+\infty$ (which means $\psi=\Phi$) gives us the desired result and completes the proof.
\end{proof}

Based on Definition \ref{def:compatibility} and previous discussions, the following theorem bounds the true expected adversarial risk, i.e.
$\mathbb{E}_{P_0}\left\{
\phi_{\gamma}\left(\boldsymbol{Z};\theta\right)
\right\}$ based on the expected value of the proposed risk $\mathbb{E}_{P_0}\left\{\hat{R}_{\mathrm{SSAR}}\left(\theta;\boldsymbol{D}\right)\right\}$, for all $\theta$ that happen to be in a neighborhood of its minimizer.
\\[-3mm]
\begin{thm2}[Statistical Consistency]
\label{thm:statGeneral}
Assume the function set $\Phi\triangleq\left\{\phi_{\gamma}\left(\cdot;\theta\right)\vert~\theta\in\Theta\right\}$ of adversarial loss functions $\phi_{\gamma}:\mathcal{Z}\times\Theta\rightarrow\mathbb{R}$ defined in \eqref{eq:phiGammaDef}, for a feature-label space $\mathcal{Z}\triangleq\mathcal{X}\times\mathcal{Y}$, a parameter space $\Theta$ and dual parameter $\gamma\ge0$. Let $P_0\in M\left(\mathcal{X}\times\mathcal{Y}\right)$ to be any distribution. Also, assume $\theta^*$ to be the minimizer of the actual adversarial loss, i.e.
$\theta^*=\argmin_{\theta\in\Theta}~\mathbb{E}_{P_0}\left\{
\phi_{\gamma}\left(\boldsymbol{Z};\theta\right)
\right\}$. Let $\eta\in\left[0,1\right]$ to denote a supervision ratio, and assume $\zeta\ge0$ and $\lambda\in\mathbb{R}\cup\left\{\pm\infty\right\}$ such that the following condition holds:
\begin{equation}
\eta~\ge~
\mathrm{MSR}_{\left(\Phi,P_0\right)}\left(\lambda,\zeta\right).
\end{equation}
Consider a partially labeled dataset $\boldsymbol{D}\triangleq\left\{\left(\boldsymbol{X}_i,y_i\right)\right\}_{i=1}^{n}$ consisting of $n$ i.i.d. samples drawn from $P_0$, where labels can be observed with probability of $\eta$, independently from each other. Then, there exists a neighborhood $\Theta_{\mathrm{local}}$, such that $\theta^*\in\Theta_{\mathrm{local}}\subseteq\Theta$ and all the following relations hold:
\begin{align}
&\argmin_{\theta\in\Theta}~\mathbb{E}_{P_{0}}\left\{
\hat{R}_{\mathrm{SSAR}}\left(\theta;\boldsymbol{D}\right)
\right\}\in\Theta_{\mathrm{local}},
\nonumber\\
&
\mathbb{E}_{P_{0}}\left\{\hat{R}_{\mathrm{SSAR}}\left(\theta;\boldsymbol{D}\right)-
\hat{R}_{\mathrm{SSAR}}\left(\theta^*;\boldsymbol{D}\right)\right\}
\ge\zeta,\quad&\forall\theta\notin\Theta_{\mathrm{local}},
\nonumber\\
\mathrm{{\it {and}}}\quad~
&\mathbb{E}_{P_0}\left\{\phi_{\gamma}\left(\boldsymbol{Z};\theta\right)\right\}
+\gamma\epsilon
\leq
\mathbb{E}_{P_{0}}\left\{
\hat{R}_{\mathrm{SSAR}}\left(\theta;\boldsymbol{D}\right)
\right\},
&\forall\theta\in\Theta_{\mathrm{local}},
\end{align}
where the term $\gamma\epsilon$ appears due to the definition of $\hat{R}_{\mathrm{SSAR}}$ in Theorem \ref{thm:sslDual}.
\end{thm2}
\begin{proof}
Based on the proof of Lemma \ref{lemma:compExistLemma} and definition of $\mathrm{MSR}_{\left(\Phi,P_0\right)}$, it can be easily checked that the condition $\eta\ge\mathrm{MSR}_{\left(\Phi,P_0\right)}\left(\lambda,\zeta\right)$ implies that:
\begin{align}
\exists\psi\subseteq\Phi~\bigg\vert\quad
\frac{\mathrm{GAP}\left(\psi\right)-\zeta}{1-\eta}+
\Gamma\left(\psi;\lambda\right)\ge0
\quad\mathrm{and}\quad
\rho_{\lambda}\left(\phi\right)\ge0,~\forall\phi\in\psi.
\end{align}
Let $\phi^*\triangleq\phi_{\gamma}\left(\cdot;\theta^*\right)$. According to the definition of $\mathrm{GAP}$ and $\Gamma$ in Definition \ref{def:GAP_PI}, the first condition in the above results in the following chain of relations:
\begin{align}
\zeta&\leq
\min_{\phi\in\Phi-\psi}\mathbb{E}_{P_0}\left\{\phi-\phi^*\right\}+\left(1-\eta\right)\min_{\phi\in\Phi-\psi}
\mathbb{E}_{P_{0_{\boldsymbol{X}}}}\left\{
\rho_{\lambda}\left(\phi\right)-\rho_{\lambda}\left(\phi^*\right)
\right\}
\nonumber\\
&\leq
\min_{\phi\in\Phi-\psi}\left\{
\mathbb{E}_{P_0}\left\{\phi\right\}+\left(1-\eta\right)
\rho_{\lambda}\left(\phi\right)
\right\} - 
\left\{\mathbb{E}_{P_0}\left\{\phi^*\right\}-
\left(1-\eta\right)\rho_{\lambda}\left(\phi^*\right)\right\}
\nonumber \\
&=
\min_{\phi\in\Phi-\psi}
\mathbb{E}_{P_0}\left\{
\eta\phi+\left(1-\eta\right)
\softmin^{\left(\lambda\right)}_{y\in\mathcal{Y}}\left\{\phi_{\boldsymbol{X}}\right\}
\right\} - 
\mathbb{E}_{P_0}\left\{
\eta\phi^*
+\left(1-\eta\right)
\softmin^{\left(\lambda\right)}_{y\in\mathcal{Y}}\left\{\phi^*_{\boldsymbol{X}}\right\}
\right\}
\nonumber\\
&=
\min_{\theta\in\Theta-\Theta_{\mathrm{local}}}
\mathbb{E}_{P_0}\left\{
\hat{R}_{\mathrm{SSAR}}\left(\theta;\boldsymbol{D}\right)
\right\}-
\mathbb{E}_{P_0}\left\{
\hat{R}_{\mathrm{SSAR}}\left(\theta^*;\boldsymbol{D}\right)
\right\},
\end{align}
where $\Theta_{\mathrm{local}}$ denotes the subset of parameter space $\Theta$ that corresponds to function subset $\psi$. This proves the first two arguments of the Theorem. Note that the first argument can be directly deduced from the second one, and we have only written it separately for the sake of emphasis and clarity. The third argument can also be directly deduced from the fact that $\Lambda\left(\psi\right)\leq\lambda$. Note that based on Definition \ref{def:GAP_PI} and for all $\phi\in\psi$ (or equivalently $\theta\in\Theta_{\mathrm{local}}$), we have $\rho_{\Lambda\left(\psi\right)}\left(\phi\right)\ge0$. Therefore:
\begin{align}
\left(1-\eta\right)
\rho_{\Lambda\left(\psi\right)}\left(\phi\right)
&=
\left(1-\eta\right)
\mathbb{E}_{P_{0_{\boldsymbol{X}}}}\left\{
\softmin^{\left(\Lambda\left(\psi\right)\right)}_{y\in\mathcal{Y}}\left\{\phi_{\boldsymbol{X}}
\right\}
-\mathbb{E}_{P_{0_{\vert\boldsymbol{X}}}}\left\{
\phi_{\boldsymbol{X}}
\right\}
\right\}
\nonumber\\
&=
\left(1-\eta\right)
\mathbb{E}_{P_{0_{\boldsymbol{X}}}}\left\{
\softmin^{\left(\Lambda\left(\psi\right)\right)}_{y\in\mathcal{Y}}\left\{\phi_{\boldsymbol{X}}
\right\}\right\}
-
\left(1-\eta\right)
\mathbb{E}_{P_0}\left\{\phi\right\}
\nonumber\\
&=
\eta\mathbb{E}_{P_0}\left\{\phi\right\}+
\left(1-\eta\right)
\mathbb{E}_{P_{0_{\boldsymbol{X}}}}\left\{
\softmin^{\left(\Lambda\left(\psi\right)\right)}_{y\in\mathcal{Y}}\left\{\phi_{\boldsymbol{X}}
\right\}\right\}
-
\mathbb{E}_{P_0}\left\{\phi\right\}
\nonumber\\
&=
\mathbb{E}_{P_0}\left\{\hat{R}_{\mathrm{SSAR}}\left(\theta;\boldsymbol{D}\right)\right\}
-\mathbb{E}_{P_0}\left\{\phi_{\gamma}\left(\boldsymbol{Z};\theta\right)\right\}
-\gamma\epsilon
\ge0.
\end{align}
Taking into account the fact that $\softmin^{\left(\lambda\right)}_{y\in\mathcal{Y}}\left(\cdot\right)$ is an increasing function w.r.t. $\lambda$ leads to the third argument, and thus completes the proof.
\end{proof}

Theorem \ref{thm:statGeneral} provides a mathematical foundation for establishing a general learning-theoretic bound on the generalization aspect of self-learning paradigm, that can be applied to our distributionally robust setting as well. Intuitively, it states that for good choices of the pair $\left(\eta,\lambda\right)$, one can guarantee the following two outcomes:

First, the minimizer of the expected proposed loss happens to be in a neighborhood of the true minimizer, i.e. $\argmin_{\theta\in\Theta}\sup_{P\in\mathcal{B}_{\epsilon}\left(P_0\right)}\mathbb{E}_{P}\left\{\ell\left(\boldsymbol{Z};\theta\right)\right\}$. Also, a positive margin $\zeta>0$ can be considered which puts a gap between the minimum value of the proposed expected loss and those that fall outside of this neighborhood. This margin will be extremely helpful when we are dealing with empirical risks instead of the statistical ones (see the proof of Theorem \ref{thm:generalBound1}).

Second, all over the above-mentioned neighborhood, $R_{\mathrm{SSAR}}$ provides an upper-bound on the true expected loss. In this regard, as long as a minimum level of pessimism is considered with respect to the compatibility of the hypothesis set and distribution duo, i.e. $\lambda\ge\Lambda\left(\psi\right)$, it can be guaranteed that the self-learning module does not overfit and assigns meaningful labels to the unlabeled data.

\section{Auxiliary Theorems and Proofs}
\label{sec:appendix:thm}
\begin{proof}[Proof of Theorem \ref{thm:sslDual}]
The proof proceeds by the substitution of original proposed semi-supervised problem in \eqref{eq:main2} by its dual form. This way, we can take advantage of the good mathematical properties that this dual form can provide, specially w.r.t. maximization over $P\in\mathcal{B}_{\epsilon}\left(S\right)$. The following lemma (see Theorem $1$ and Remark $1$ of \cite{blanchet2019quantifying}), formulates the dual form:%
\begin{lemma}[Lagrangian Relaxation and Duality]
Assume $\mathcal{Z}$ to be a sample space and let $\Theta$ to denote the space of parameters. Let loss function $\ell:\mathcal{Z}\times\Theta\rightarrow\mathbb{R}_{\ge0}$ and  function $c:\mathcal{Z}\times\mathcal{Z}\rightarrow\mathbb{R}_{\ge0}$ to be continuous, and further assume $c$ is lower semi-continuous and $c\left(\boldsymbol{z},\boldsymbol{z}\right)=0,~\forall\boldsymbol{z}\in\mathcal{Z}$. Then, for any $\epsilon\ge0$ and any distribution $Q\in M\left(\mathcal{Z}\right)$, the following equality holds for all $\theta\in\Theta$:
\begin{equation}
\sup_{P\in\mathcal{B}_{\epsilon}\left(Q\right)}\mathbb{E}_P\left\{
\ell\left(\boldsymbol{Z};\theta\right)\right\}
=
\inf_{\gamma\ge0}\left\{
\gamma\epsilon+
\mathbb{E}_Q\left\{
\sup_{\boldsymbol{z}'\in\mathcal{Z}}
\ell\left(\boldsymbol{z}';\theta\right) - \gamma c\left(\boldsymbol{z}',\boldsymbol{Z}\right)
\right\}
\right\}.
\label{eq:dual}
\end{equation}
\label{thm:dual}
\end{lemma}

Proof is explained in details in the original reference. Based on the duality equation in Lemma \ref{thm:dual}, the following chain of relations hold:
\begin{align}
\label{eq:thmDROSSLproof}
&\inf_{S\in\hat{\mathcal{P}}\left(\boldsymbol{D}\right)}
\left(
\sup_{P\in\mathcal{B}_{\epsilon}\left(S\right)}
\mathbb{E}_P\left\{
\ell\left(\boldsymbol{X},y;\theta\right)
\right\}
+
\frac{1}{\lambda}
\left(\frac{n_{\mathrm{ul}}}{n}\right)
\hat{\mathbb{E}}_{\boldsymbol{D}_{\mathrm{ul}}}\left\{
\mathbb{H}\left(S_{\vert\boldsymbol{X}}\right)
\right\}
\right)
\\
=~&
\inf_{S\in\hat{\mathcal{P}}\left(\boldsymbol{D}\right)}
\left[
\inf_{\gamma\ge0}\left(
\gamma\epsilon +
\mathbb{E}_S\left\{
\sup_{\boldsymbol{z}'\in\mathcal{Z}}
\ell\left(\boldsymbol{z}';\theta\right)
-\gamma
c\left(\boldsymbol{z}',\boldsymbol{Z}\right)\right\}
\right)
+
\frac{1}{\lambda}
\left(\frac{n_{\mathrm{ul}}}{n}\right)
\hat{\mathbb{E}}_{\boldsymbol{D}_{\mathrm{ul}}}\left\{
\mathbb{H}\left(S_{\vert\boldsymbol{X}}\right)
\right\}
\right]
\nonumber
\\
=~&
\inf_{\gamma\ge0}\left[
\gamma\epsilon +
\inf_{S\in\hat{\mathcal{P}}\left(\boldsymbol{D}\right)}
\left(
\mathbb{E}_S\left\{
\sup_{\boldsymbol{z}'\in\mathcal{Z}}
\ell\left(\boldsymbol{z}';\theta\right)
-\gamma
c\left(\boldsymbol{z}',\left(\boldsymbol{X},y\right)\right)\right\}
+
\frac{1}{\lambda}
\left(\frac{n_{\mathrm{ul}}}{n}\right)
\hat{\mathbb{E}}_{\boldsymbol{D}_{\mathrm{ul}}}\left\{
\mathbb{H}\left(S_{\vert\boldsymbol{X}}\right)
\right\}
\right)
\right]
\nonumber \\
=~&
\inf_{\gamma\ge0}\left[
\gamma\epsilon +
\left(\frac{n_{\mathrm{l}}}{n}\right)
\frac{1}{n_{\mathrm{l}}}\sum_{i\in\mathcal{I}_{\mathrm{l}}}\left(
\sup_{\boldsymbol{z}'\in\mathcal{Z}}
\ell\left(\boldsymbol{z}';\theta\right)
-\gamma
c\left(\boldsymbol{z}',\left(\boldsymbol{X}_i,y_i\right)\right)
\right)
\right.
\nonumber \\
&\left.
\hspace{1.3cm}
+
\left(\frac{n_{\mathrm{ul}}}{n}\right)
\frac{1}{n_{\mathrm{ul}}}
\sum_{i\in\mathcal{I}_{\mathrm{ul}}}\inf_{\Omega\in M\left(\mathcal{Y}\right)}\left(
\mathbb{E}_{\Omega}\left\{
\sup_{\boldsymbol{z}'\in\mathcal{Z}}
\ell\left(\boldsymbol{z}';\theta\right)
-\gamma
c\left(\boldsymbol{z}',\left(\boldsymbol{X}_i,y\right)\right)\right\}
+
\frac{1}{\lambda}
\mathbb{H}\left(\Omega\right)
\right)
\right],
\nonumber
\end{align}
where the last inequality is a direct result of defining $\hat{\mathcal{P}}\left(\boldsymbol{D}\right)$ in Definition \ref{def:consistent}. Therefore, each $S\in\hat{\mathcal{P}}\left(\boldsymbol{D}\right)$ can be regarded as a weighted (with weights $n_{\mathrm{l}}/n$ and $n_{\mathrm{ul}}/n$, respectively) mixture of $\mathbb{P}_{\boldsymbol{D}_{\mathrm{l}}}$, i.e. delta-spikes over the labeled samples, and $\mathbb{P}_{\boldsymbol{D}_{\mathrm{ul}}}\tilde{\Omega}$, i.e. the same for unlabeled feature vectors which are multiplied by arbitrary conditional distributions of the form ${\Omega}\in M^{\mathcal{X}}\left(\mathcal{Y}\right)$. The two summations above which are over labeled and unlabeled samples, respectively, correspond to this bi-mixture formalism. Thus, the chain of relations in \eqref{eq:thmDROSSLproof} can be continued as 
\begin{align}
=~&
\inf_{\gamma\ge0}\left[
\gamma\epsilon+
\frac{1}{n}\sum_{i\in\mathcal{I}_{\mathrm{l}}}\phi_{\gamma}\left(\boldsymbol{X}_i,y_i\vert\theta\right)+
\frac{1}{n}\sum_{i\in\mathcal{I}_{\mathrm{ul}}}
\left(
\inf_{\Omega\in M\left(\mathcal{Y}\right)}
\sum_{y\in\mathcal{Y}}
\Omega_y
\phi_{\gamma}\left(\boldsymbol{X}_i,y;\theta\right)+
\frac{1}{\lambda}\mathbb{H}\left(\Omega\right)
\right)
\right]
\nonumber\\
=~&
\inf_{\gamma\ge0}\left[
\gamma\epsilon+
\frac{1}{n}\sum_{i\in\mathcal{I}_{\mathrm{l}}}\phi_{\gamma}\left(\boldsymbol{X}_i,y_i;\theta\right)+
\frac{1}{n}\sum_{i\in\mathcal{I}_{\mathrm{ul}}}
\softmin_{y\in\mathcal{Y}}^{\left(\lambda\right)}\left\{
\phi_{\gamma}\left(\boldsymbol{X}_i,y;\theta\right)
\right\}
\right]+const,
\end{align}
where $const$ deos not depend on $\gamma$ or $\theta$, and the last equality is due to the following lemma:
\\[-3mm]
\begin{lemma}
Assume an arbitrary vector $\boldsymbol{b}\in\mathbb{R}^d$ for $d\in\mathbb{N}$, and also let $\mathcal{F}\triangleq\left\{1,\ldots,d\right\}$. Then the following relation holds for all $\lambda\in\mathbb{R}\cup\left\{\pm\infty\right\}$:
\begin{equation}
\softmin_{i\in\mathcal{F}}^{\left(\lambda\right)}\left(b_1,\ldots,b_d\right)=
\inf_{\boldsymbol{q}\in M\left(\mathcal{F}\right)}~
\boldsymbol{q}^T\boldsymbol{b}+\frac{1}{\lambda}\mathbb{H}\left(\boldsymbol{q}\right)-
\frac{1}{\lambda}\log d,
\end{equation}
where $\mathbb{H}\left(\cdot\right)$ denotes the Shannon entropy of distribution $\boldsymbol{q}$ over $\mathcal{F}$.
\label{lemma:softmin}
\end{lemma}
\begin{proof}
The main idea is to replace the term $\boldsymbol{q}^T\boldsymbol{b}$ with
\begin{equation}
\boldsymbol{q}^T\boldsymbol{b}=
\sum_{i\in\mathcal{F}}q_ib_i=\frac{1}{\lambda}\sum_{i\in\mathcal{F}}q_i\log e^{\lambda b_i}.
\end{equation}
Also, note that $\frac{1}{\lambda}\mathbb{H}\left(\boldsymbol{q}\right)-\frac{1}{\lambda}\log d=\frac{-1}{\lambda}\mathcal{D}_{\mathrm{KL}}\left(\boldsymbol{q}\Vert\mathcal{U}\right)$, where
$\mathcal{D}_{\mathrm{KL}}$ is the Kullback–Leibler divergence between two probability measures and
$\mathcal{U}\in M\left(\mathcal{F}\right)$ denotes the uniform measure on $\mathcal{F}$. As a result, the overall objective function can be rewritten as
\begin{align*}
&\boldsymbol{q}^T\boldsymbol{b}-\frac{1}{\lambda}\mathcal{D}_{\mathrm{KL}}\left(\boldsymbol{q}\Vert\mathcal{U}\right)
\\
=~&
-\frac{1}{\lambda}\mathcal{D}_{\mathrm{KL}}\left(\boldsymbol{q}\Vert\mathcal{U}\right)
+\frac{1}{\lambda}\sum_{i\in\mathcal{F}}q_i\log e^{\lambda b_i}
\\
=~&
-\frac{1}{\lambda}\sum_{i\in\mathcal{F}}q_i\log\left(dq_i\right)
+\frac{1}{\lambda}\sum_{i\in\mathcal{F}}q_i\log e^{\lambda b_i}
\\
=~&
-\frac{1}{\lambda}\sum_{i\in\mathcal{F}}q_i\log\left(\frac{q_i}{\frac{1}{d}e^{\lambda b_i}}\right)
=
-\frac{1}{\lambda}\sum_{i\in\mathcal{F}}q_i\log\left(\frac{q_i}{\frac{\alpha}{d}e^{\lambda b_i}}\right)
-\frac{1}{\lambda}\log\alpha,
\end{align*}
for all $\alpha>0$. Then, it can be readily verified that by setting $\alpha^{-1}\triangleq \frac{1}{d}\sum_{i\in\mathcal{F}}e^{\lambda b_i}$, the optimization problem in lemma becomes
\begin{equation}
\inf_{\boldsymbol{q}\in M\left(\mathcal{F}\right)}~-\frac{1}{\lambda}\mathcal{D}_{\mathrm{KL}}\left(
q_i\big\Vert \frac{\alpha}{d}e^{\lambda b_i}
\right)+
\frac{1}{\lambda}\log\left(\frac{1}{d}\sum_{i\in\mathcal{F}}e^{\lambda b_i}\right),
\end{equation}
whose solution always happens to be $q^*_i=\frac{\alpha}{d}e^{-b_i/\lambda}$, regardless of the sign of $\lambda$. Therefore, the solution of the primary optimization problem in lemma would be
\begin{equation}
\frac{1}{\lambda}\log\left(\frac{1}{d}\sum_{i\in\mathcal{F}}e^{\lambda b_i}\right)=
\softmin_{i\in\mathcal{F}}^{\left(\lambda\right)}\left(b_1,\ldots,b_d\right),
\end{equation}
which completes the proof.
\end{proof}


According to the duality relation between $\gamma$ and $\epsilon$, the minimization over $\gamma$ is not necessary in almost all practical situations, where the same methodologies for evaluating a {\it {practically good}} value for $\epsilon$, such as cross-validation, can be used for $\gamma$ as well.
\end{proof}

\begin{proof}[Proof of Theorem \ref{thm:sgdConv}]
The proof is based on a number of techniques used in \cite{ghadimi2012optimal}, and can be considered as a generalization of Theorem $2$ of \cite{sinha2018certifying} for the semi-supervised settings. Similarly, let us define the following set of Lipschitz constants, based on the smoothness constraints assumed in Theorem \ref{thm:sgdConv}:
\begin{align}
\label{eq:lipConstantLoss}
&\left\Vert
\nabla_{\theta}
\ell\left(\boldsymbol{z};\theta\right)-
\nabla_{\theta}
\ell\left(\boldsymbol{z};\theta'\right)
\right\Vert_*\leq
L_{\theta\theta}
\left\Vert
\theta-\theta'
\right\Vert,
~
&\left\Vert
\nabla_{\theta}
\ell\left(\boldsymbol{z};\theta\right)-
\nabla_{\theta}
\ell\left(\boldsymbol{z}';\theta\right)
\right\Vert_*\leq
L_{\theta\boldsymbol{z}}
\left\Vert
\boldsymbol{z}-\boldsymbol{z}'
\right\Vert,
\nonumber\\
&\left\Vert
\nabla_{\boldsymbol{z}}
\ell\left(\boldsymbol{z};\theta\right)-
\nabla_{\boldsymbol{z}}
\ell\left(\boldsymbol{z};\theta'\right)
\right\Vert_*\leq
L_{\boldsymbol{z}\theta}
\left\Vert
\theta-\theta'
\right\Vert,
~
&\left\Vert
\nabla_{\boldsymbol{z}}
\ell\left(\boldsymbol{z};\theta\right)-
\nabla_{\boldsymbol{z}}
\ell\left(\boldsymbol{z}';\theta\right)
\right\Vert_*\leq
L_{\boldsymbol{z}\boldsymbol{z}}
\left\Vert
\boldsymbol{z}-\boldsymbol{z}'
\right\Vert,
\nonumber
\end{align}
where $\left\{L_{\theta\theta},L_{\theta\boldsymbol{z}},L_{\boldsymbol{z}\theta},L_{\boldsymbol{z}\boldsymbol{z}}\right\}$ are a set of Lipschitz constants, $\left\Vert\cdot\right\Vert$ can be any valid norm (generally different norms should be used for $\mathcal{Z}$ and $\Theta$) and $\left\Vert\cdot\right\Vert_*$ denotes the corresponding dual norm(s). Also, the inequalities should hold for all $\boldsymbol{z},\boldsymbol{z}'\in\mathcal{Z}$ and all $\theta,\theta'\in\Theta$. 

In our case, i.e. a semi-supervised setting, one also needs to show that
$\nabla_{\theta}\softmin_{y\in\mathcal{Y}}^{\left(\lambda\right)}\left\{\phi_{\gamma}\left(\boldsymbol{z};\cdot\right)\right\}$ is Lipschitz with respect to $\theta$, for all $\boldsymbol{z}\in\mathcal{Z}$. Before that, Lemma \ref{lemma:supervisedLip} shows that under the above-mentioned constraints on the Lipschitz-ness of gradients of $\ell$, $\phi_{\gamma}\left(\boldsymbol{z};\theta\right)$ also has Lipschitz gradients.
\\[-3mm]
\begin{lemma}
Assume $\ell:\mathcal{Z}\times\Theta\rightarrow\mathbb{R}_{\ge0}$ is smooth and universally differentiable w.r.t. its input arguments. Also assume $\ell$ has Lipschitz gradients with constants $\left\{L_{\theta\theta},L_{\theta\boldsymbol{z}},L_{\boldsymbol{z}\theta},L_{\boldsymbol{z}\boldsymbol{z}}\right\}$, for any fixed norm $\left\Vert\cdot\right\Vert$. 
Also, assume a transportation cost $c$, which has the properties of Lemma \ref{lemma:innerMAxConcave}.
Then, the following Lipschit-ness property holds for gradients of $\phi_{\gamma}\left(\boldsymbol{z};\theta\right)=\sup_{\boldsymbol{z}'\in\mathcal{Z}}\ell\left(\boldsymbol{z}';\theta\right)-\gamma c\left(\boldsymbol{z}',\boldsymbol{z}\right)$:
\begin{equation}
\left\Vert
\nabla_{\theta}\phi_{\gamma}\left(\boldsymbol{z};\theta\right)
-
\nabla_{\theta}\phi_{\gamma}\left(\boldsymbol{z};\theta'\right)
\right\Vert_*
\leq
\left(L_{\theta\theta} + \frac{L_{z\theta}L_{\theta z}}{\gamma-L_{zz}}\right)
\left\Vert
\theta - \theta'
\right\Vert~,\quad
\forall\theta,\theta'\in\Theta,
\end{equation}
for all $\gamma> L_{zz}$.
\label{lemma:supervisedLip}
\end{lemma}
For proof of Lemma \ref{lemma:supervisedLip}, see Lemma $1$ of \cite{sinha2018certifying}. Based on this result, the following lemma provides Lipschitz constants for the $\softmin$ operator over a finite number of $\phi_{\gamma}\left(\cdot;\cdot\right)$ functions, for any $\lambda\in\mathbb{R}$.
\\[-3mm]
\begin{lemma}
For a feature-label space $\mathcal{Z}=\mathcal{X}\times\mathcal{Y}$, assume loss function $\ell:\mathcal{Z}\times\Theta\rightarrow\mathbb{R}_{\ge0}$, transportation cost $c$ and the resulting adversarial loss $\phi_{\gamma}\left(\cdot;\cdot\right):\mathcal{Z}\times\Theta\rightarrow\mathbb{R}$ with $\gamma>L_{zz}$, such that all satisfy the constraints of Lemma \ref{lemma:supervisedLip}. Also, assume there exists $\sigma\ge0$ such that $\left\Vert\nabla_{\theta}\ell\left(\boldsymbol{z};\theta\right)\right\Vert\leq\sigma$ for all $\theta\in\Theta$. Then, for all $\lambda\in\mathbb{R}$, the following Lipschitz-ness property holds:
\begin{equation}
\left\Vert
\nabla_{\theta}\softmin^{\left(\lambda\right)}_{y\in\mathcal{Y}}\left\{
\phi_{\gamma}\left(\boldsymbol{Z};\theta\right)
\right\}
-
\nabla_{\theta}\softmin^{\left(\lambda\right)}_{y\in\mathcal{Y}}\left\{
\phi_{\gamma}\left(\boldsymbol{Z};\theta'\right)
\right\}
\right\Vert_*
\leq
\left(L_{\theta\theta} + \frac{L_{z\theta}L_{\theta z}}{\gamma-L_{zz}}
+{2\sigma^2}{\left\vert\lambda\right\vert}\left\vert\mathcal{Y}\right\vert
\right)
\left\Vert
\theta - \theta'
\right\Vert,
\end{equation}
for all $\boldsymbol{Z}\in\mathcal{Z}$ and $\theta,\theta'\in\Theta$.
\label{lemma:softminLip}
\end{lemma}
In order to avoid discontinuity in the proof, the proof of Lemma \ref{lemma:softminLip} is presented in Appendix \ref{sec:appendix:lemma} instead of here. Also, let $B\triangleq\frac{1}{2}\left(L_{\theta\theta} + \frac{L_{z\theta}L_{\theta z}}{\gamma-L_{zz}}\right)$, where $B$ represents one of the constants mentioned in Theorem \ref{thm:sgdConv}. 

The last lemma which is needed to finalize the proof of Theorem \ref{thm:sgdConv} aims to bound the maximum discrepancy that one might observe, given that the inner maximization in \eqref{eq:SSLDROinnermax} (corresponds to line $6$ of Algorithm \ref{alg:SSLDROsgd}) is solved up to an approximation error of $\delta>0$.
\\[-3mm]
\begin{lemma}
\label{lemma:sinhaDeltaApprox}
Assume $\hat{\boldsymbol{z}}^*\in\mathcal{Z}$ to be a $\delta$-approximate maximizer of \eqref{eq:SSLDROinnermax} for the input $\boldsymbol{z}_0\in\mathcal{Z}$, loss function $\ell$, and transportation cost $c$. Let the consequent adversarial loss function $\phi_{\gamma}$ to satisfy all the constraints mentioned in Lemma \ref{lemma:supervisedLip} in addition to $\gamma>L_{zz}$. Then, the following upper-bound holds for all $\boldsymbol{z}_0\in\mathcal{Z}$:
\begin{equation}
\left\Vert
\nabla_{\theta}
\phi_{\gamma}\left(\boldsymbol{z}_0;\theta\right)
-
\nabla_{\theta}
\ell\left(\hat{\boldsymbol{z}}^*;\theta\right)
\right\Vert^2_*
\leq
\frac{L_{z\theta}L_{\theta z}}{\gamma-L_{zz}}\delta.
\end{equation}
\end{lemma}
Proof of Lemma \ref{lemma:sinhaDeltaApprox} is given in Appendix \ref{sec:appendix:lemma}. Also, Let $C\triangleq\frac{L_{z\theta}L_{\theta z}}{\gamma-L_{zz}}$, recalling $C$ as another constant mentioned in Theorem \ref{thm:sgdConv}.

Algorithm \ref{alg:SSLDROsgd} for a mini-batch size of $k=1$ picks one data-point randomly from $\boldsymbol{D}$ at each iteration. Also, data points at $\boldsymbol{D}$ are assumed to be drawn independently from an unknown but fixed distribution $P_0$. Therefore, one can consider a two-step data generation model in order to analyze the semi-supervised stochastic gradient descent as follows:
\begin{itemize}
\item
$\mathcal{O}$~({\it {Observation step}}): Draw a bi-categorical random variable (denoted as observation variable) $h\in\mathcal{H}\triangleq\left\{\mathrm{l},\mathrm{ul}\right\}$, with probabilities $n_{\mathrm{l}}/n$ and $n_{\mathrm{ul}}/n$ for labeled and unlabeled categories, respectively.  
\item
$\mathcal{G}$~({\it {Generation step}}): Conditioned on $h$, draw a sample from $P_0$ if $h=\mathrm{l}$, and from $P_{0,\boldsymbol{X}}$ if $h=\mathrm{ul}$.
\end{itemize}
Consider a coupled first-order {\it {Markov stochastic process}} defined as $\left(h_0,\theta_0\right),\ldots,\left(h_T,\theta_T\right)$, where $h_i$s denote the observation variables and $\theta_i$s are the consequent outputs of Algorithm \ref{alg:SSLDROsgd} after $T$ iterations. Here, $\theta_0$ can have any initial distribution over $\Theta$. Using the techniques reviewed in \cite{ghadimi2012optimal} (also similar to Theorem $2$ of \cite{sinha2018certifying}), the following result holds for for $1<t\leq T$:
\begin{align}
\mathbb{E}_{\mathcal{G}}\left\{
\hat{R}_{\mathrm{SSAR}}\left(\theta_{t+1};\boldsymbol{D}\right)-
\hat{R}_{\mathrm{SSAR}}\left(\theta_{t};\boldsymbol{D}\right)
\big\vert\theta_t,h_t
\right\}
\leq&
-\alpha\left(\frac{1}{2}-\alpha L_{h_t}\right)\left\Vert
\nabla_{\theta}
\hat{R}_{\mathrm{SSAR}}\left(\theta^t\vert\boldsymbol{D}\right)
\right\Vert^2_2
\nonumber \\
&+
\frac{1}{2}\left(\alpha+5\alpha^2L_{h_t}\right)C\delta+
\frac{1}{2}\alpha^2\sigma^2L_{h_t},
\label{eq:mohem}
\end{align}
where $\mathbb{E}_{\mathcal{G}}$ refers to expectation w.r.t. the randomness of dataset $\boldsymbol{D}$, and given that the information about each sample is labeled or not is known. Also, $L_h\in\mathbb{R}^{\mathcal{H}}_{\ge0}$ denotes the Lipschitz constants for the gradients (w.r.t. $\theta\in\Theta$) of the loss summands in \eqref{eq:SSLmainMin}. Based on Lemma \ref{lemma:softminLip}, we have
\begin{equation}
L_h\leq\left\{\begin{array}{ll}
2\left(B+{\sigma^2}{\left\vert\lambda\right\vert}\left\vert\mathcal{Y}\right\vert\right) ~& h=\mathrm{ul}
\\
2B ~& h=\mathrm{l}
\end{array}\right..
\end{equation}
Now, it should be noted that $\mathbb{E}_{\mathrm{total}}\left\{\cdot\right\}=\mathbb{E}_{\mathcal{O}}\left\{\mathbb{E}_{\mathcal{G}}\left\{\cdot\vert h\in\mathcal{H}\right\}\right\}$, where $\mathbb{E}_{\mathrm{total}}$ denotes the total expectation which is w.r.t. the dataset $\boldsymbol{D}$ whose samples are drawn i.i.d. from $P_0$ and also the randomness of SGD used in Algorithm \ref{alg:SSLDROsgd}. Also, due to the independence assumption on observing each label with probability $\eta$, we have
\begin{equation}
\mathbb{E}_\mathcal{O}\left\{L_{h_t}\right\}=2\left(B+{\bar{\eta}\sigma^2}{\left\vert\lambda\right\vert}\left\vert\mathcal{Y}\right\vert\right)~,\forall t.
\end{equation}
Combining the above arguments with \eqref{eq:mohem} directly leads us to the claims in Theorem \ref{thm:sgdConv} and completes the proof.
\end{proof}
\begin{thm2}[Convergence of hard decisions, $\lambda=\pm\infty$]
Consider the setting described in Theorem \ref{thm:sgdConv}, where $\ell$ is twice differentiable w.r.t. $\theta$ all over $\mathcal{Z}\times\Theta$.  Assume one sets $\lambda=+\infty$ or $\lambda=-\infty$. Also, assume step-size $\alpha$ and approximation interval $\delta$ in Algorithm \ref{alg:SSLDROsgd} can change during the iterations. Then, there exist a sequence of step-sizes $\alpha_1,\alpha_2,\ldots$ and a sequence of approximation intervals $\delta_1,\delta_2,\ldots$ for which Algorithm \ref{alg:SSLDROsgd} converges to a local minimizer of $\hat{R}_{\mathrm{SSAR}}\left(\theta;\boldsymbol{D}\right)$, as $T\rightarrow\infty$ where $T$ is the number of iterations. 
\label{thm:hardConv}
\end{thm2}
\begin{proof}
Problem setting for $\lambda=+\infty$ results into a minimax problem, i.e. minimizing over $\theta\in\Theta$ while maximizing over $y_i\in\mathcal{Y},~i\in\mathcal{I}_{\mathrm{ul}}$ for any given $\theta$. Thus, the solution is a local saddle point in $\Theta\times\mathcal{Y}^{\left\vert\mathcal{I}_{\mathrm{ul}}\right\vert}$. Convergence of combinatoric optimization schemes for such problems are already established (see \cite{loog2016contrastive} and \cite{dresher1961games}), and we avoid to repeat them here.

For the case of $\lambda=-\infty$, we show that by choosing sufficiently small values for $\alpha_i$ and $\delta_i$ for $i=1,2,\ldots$, the objective of the optimization always decreases, and thus convergence to a stable point is guaranteed. First, let us define
\begin{equation}
y^*_i\left(\theta\right)\triangleq\argmin_{y\in\mathcal{Y}}\phi_{\gamma}\left(\boldsymbol{X}_i,y;\theta\right),
\end{equation}
for $i\in\mathcal{I}_{\mathrm{ul}}$. Whenever there are more than one minimizers, one of them is chosen at random. Assume iteration steps $t_s$ and $t_f$ (with $t_s\leq t_f$), such that $y^*_i\left(\theta_t\right)$ for $t_s\leq t\leq t_f$ does not change for any $i\in\mathcal{I}_{\mathrm{ul}}$. Then, Algorithm \ref{alg:SSLDROsgd} for this period acts exactly like a fully-supervised Stochastic Gradient Descent method on the dataset $\left\{\left(\boldsymbol{X}_i,y_i\right),~i\in\mathcal{I}_{\mathrm{l}}\right\}\cup\left\{\left(\boldsymbol{X}_i,y^*_i\left(\theta_t\right)\right),~i\in\mathcal{I}_{\mathrm{ul}}\right\}$. Consider the set of Lipschitz constants from Theorem \ref{thm:sgdConv} (refer to its proof in Appendix \ref{sec:appendix:thm}), i.e. $\left\{L_{\theta\theta},L_{\theta\boldsymbol{z}},L_{\boldsymbol{z}\theta},L_{\boldsymbol{z}\boldsymbol{z}}\right\}$. Let
\begin{equation}
\delta_t\leq \frac{\gamma-L_{\boldsymbol{z}\boldsymbol{z}}}{2nL_{\theta\boldsymbol{z}}L_{\boldsymbol{z}\theta}}
\min_{i=1,2,\ldots,n}
\left\Vert
\nabla_{\theta}\phi_{\gamma}\left(\boldsymbol{Z}_i;\theta_{t-1}\right)
\right\Vert_2,
\end{equation}
where $\boldsymbol{Z}_i=\left(\boldsymbol{X}_i,y_i\right),~i\in\mathcal{I}_{\mathrm{l}}$ and $\boldsymbol{Z}_i=\left(\boldsymbol{X}_i,y^*_i\left(\theta_t\right)\right),~i\in\mathcal{I}_{\mathrm{ul}}$.
Also assume
\begin{equation}
\alpha_t\leq 
\min_{i=1,2,\ldots,n}~\inf_{\theta\in\Theta}~
\frac{4}{9}\left\vert\lambda^{-1}_{\max}\left\{
\nabla^2_{\theta\theta}\phi_{\gamma}\left(\boldsymbol{Z}_i;\theta\right)
\right\}
\right\vert,
\end{equation}
where $\nabla^2_{\theta\theta}$ indicates the Hessian matrix operator, and $\lambda_{\max}\left\{\cdot\right\}$ extracts the maximum eigenvalue. Then, it can be easily checked that $\phi_{\gamma}\left(\boldsymbol{Z}_i;\theta_{t}\right)\leq\phi_{\gamma}\left(\boldsymbol{Z}_i;\theta_{t-1}\right)$ for all $i=1,2,\ldots,n$. This result is due to the fact that for any twice differentiable function $f:\mathbb{R}^d\rightarrow\mathbb{R}$, with $\boldsymbol{x},\boldsymbol{v}\in\mathbb{R}^d$ and $d\in\mathbb{N}$, we have
\begin{equation}
f\left(\boldsymbol{x}+\boldsymbol{v}\right)-f\left(\boldsymbol{x}\right)=
\boldsymbol{v}^T\nabla f\left(\boldsymbol{x}\right)+\frac{1}{2}\boldsymbol{v}^T\nabla^2f\left(\tilde{\boldsymbol{x}}\right)\boldsymbol{v},
\label{eq:badihiTaylor}
\end{equation} 
with $\tilde{\boldsymbol{x}}\in \left\{\boldsymbol{x}+\mu\boldsymbol{v}\vert~0\leq\mu\leq1\right\}$. Also, based on Lemma \ref{lemma:sinhaDeltaApprox} and given the condition on $\delta_t$, we have
\begin{equation}
\Delta\triangleq 
\frac{\left\Vert
\hat{\partial}_{\theta}\hat{R}_{\mathrm{SSAR}}\left(\theta_{t-1};\boldsymbol{D}\right)-
\partial^*_{\theta}\hat{R}_{\mathrm{SSAR}}\left(\theta_{t-1};\boldsymbol{D}\right)
\right\Vert_2}
{\left\Vert
\partial^*_{\theta}\hat{R}_{\mathrm{SSAR}}\left(\theta_{t-1};\boldsymbol{D}\right)
\right\Vert_2}
\leq\frac{1}{2},
\end{equation}
where $\hat{\partial}_{\theta}\hat{R}_{\mathrm{SSAR}}\left(\theta_{t-1};\boldsymbol{D}\right)$ represents the sub-gradient of $\hat{R}_{\mathrm{SSAR}}\left(\theta_{t-1};\boldsymbol{D}\right)$ with the inexact solution of \eqref{eq:SSLDROinnermax} (a $\delta_t$-approximate solution), while $\partial^*_{\theta}\hat{R}_{\mathrm{SSAR}}\left(\theta_{t-1};\boldsymbol{D}\right)$ denotes the exact sub-gradient corresponding to the same data point chosen for iteration $t$. This result holds regardless of the randomness of Algorithm \ref{alg:SSLDROsgd} in choosing a sample for computing the sub-gradient. Using \eqref{eq:badihiTaylor},
it can be easily checked that
\begin{align}
&\hat{R}_{\mathrm{SSAR}}\left(\theta_{t};\boldsymbol{D}\right)-
\hat{R}_{\mathrm{SSAR}}\left(\theta_{t-1};\boldsymbol{D}\right)
\leq
\nonumber\\
&{\left\Vert
\partial^*_{\theta}\hat{R}_{\mathrm{SSAR}}\left(\theta_{t-1};\boldsymbol{D}\right)
\right\Vert^2_2}\left(
-\alpha_t\left(1-\Delta\right)+
\frac{\alpha^2_t}{2}
\left\vert\lambda_{\max}\left\{
\nabla^2_{\theta\theta}\phi_{\gamma}\left(\boldsymbol{Z}^{\left(t\right)}_{\mathrm{chosen}};\tilde{\theta}\right)
\right\}
\right\vert
\left(1+\Delta\right)^2
\right),
\end{align}
where $\boldsymbol{Z}^{\left(t\right)}_{\mathrm{chosen}}$ represents that particular $\boldsymbol{Z}_i,~i=1,2,\ldots,n$ that is chosen for computing the sub-gradient at interation $t_s\leq t\leq t_f$. Also, we have $\tilde{\theta}\in\left\{\mu\theta_{t-1}+\left(1-\mu\right)\theta_t\vert~0\leq \mu\leq1\right\}$. It is straightforward to check that due to the mentioned condition on $\alpha_t$, we have
\begin{equation}
\hat{R}_{\mathrm{SSAR}}\left(\theta_{t_f};\boldsymbol{D}\right)
\leq
\hat{R}_{\mathrm{SSAR}}\left(\theta_{t_s};\boldsymbol{D}\right).
\end{equation}
On the other hand, while transitioning from the $t_f$th to $\left(t_f+1\right)$th iteration, where at least one $y^*_i\left(\theta\right)$ changes by assumption, again we have
\begin{equation}
\hat{R}_{\mathrm{SSAR}}\left(\theta_{t_f+1};\boldsymbol{D}\right)
\leq
\hat{R}_{\mathrm{SSAR}}\left(\theta_{t_f};\boldsymbol{D}\right),
\end{equation}
due to the definition of $y^*_i\left(\theta_{t_f+1}\right)$ for $i\in\mathcal{I}_{\mathrm{ul}}$. This way, Algorithm \ref{alg:SSLDROsgd} never increases the optimization objective and convergence to a stable point is guaranteed as $T\rightarrow\infty$.

Obviously, the arguments of Theorem \ref{thm:hardConv} still hold for $\delta=0$. However, it is not practical since \eqref{eq:SSLDROinnermax} cannot be solved with an infinitesimally small error in reality. On the other hand, giving a convergence rate for the two scenarios considered in this theorem, i.e. $\lambda=\pm\infty$, falls out of the scope of this paper. A trivial upper-bound on the number of iterations increases exponentially w.r.t. the number of unlabeled samples $n_{\mathrm{ul}}$, which is based on the worst-case assumption that the combinatoric part of the optimization walks through all the possible labels for the unlabeled data. However, \cite{wu2016convergence} has experimentally shown that the convergence rate (at least for a class of similar problems) is much faster. It should be noted that solving for the exact convergence rate of Theorem \ref{thm:hardConv} is equivalent to assessing the convergence rate of {\it {self-training}}, which (to the best of our knowledge) is still an open area of research.
\end{proof}
\begin{thm2}[Convexity]
\label{corl:Convexity}
Assume the setting of Theorem \ref{thm:sgdConv} with $\Theta\subseteq\mathbb{R}^d$, for some $d\in\mathbb{N}$. Let the loss function $\ell:\mathcal{Z}\times\Theta\rightarrow\mathbb{R}_{\ge0}$ to be twice differentiable and strictly convex with respect to $\theta$, for all $\left(\boldsymbol{z},\theta\right)\in\mathcal{Z}\times\Theta$. Also, assume $\lambda$ satisfies the following property
\begin{equation}
\lambda \ge
-
\inf_{\left(\boldsymbol{z},\theta\right)\in\mathcal{Z}\times\Theta}
\frac{\lambda_{\min}\left\{
\nabla^2_{\theta\theta}\phi_{\gamma}\left(\boldsymbol{z};\theta\right)
\right\}}
{\sigma^2\left(1-\left\vert\mathcal{Y}\right\vert^{-1}\right)},
\label{eq:convexityGuarant}
\end{equation}
where $\nabla^2_{\theta\theta}$ is the Hessian matrix operator w.r.t. $\theta$, and $\lambda_{\min}\left\{\cdot\right\}:\mathbb{R}^{d\times d}\rightarrow\mathbb{R}$ denotes the minimum eigenvalue operator. Then, the optimization programs in \eqref{eq:main2} and \eqref{eq:SSLmainMin} w.r.t. $\theta$ are convex.
\end{thm2} 
\begin{proof}
For $\boldsymbol{z}_0\in\mathcal{Z}$, let us define the function $f_{\boldsymbol{z}_0}\left(\theta,\boldsymbol{z}\right):\Theta\times\mathcal{Z}\rightarrow\mathbb{R}$ as
\begin{equation}
f_{\boldsymbol{z}_0}\left(\theta,\boldsymbol{z}\right)\triangleq
\ell\left(\boldsymbol{z};\theta\right)-\gamma c\left(\boldsymbol{z},\boldsymbol{z}_0\right),
\end{equation}
then we have $\phi_{\gamma}\left(\boldsymbol{z}_0;\theta\right)=\max_{\boldsymbol{z}}f_{\boldsymbol{z}_0}\left(\theta,\boldsymbol{z}\right)$. Since $f$ is twice differentiable and convex w.r.t. $\theta$, $\phi_{\gamma}$ also shares these two properties based on Danskin's theorem \cite{bonnans2013perturbation}. Thus, the $d\times d$ hessian matrix $\nabla^2_{\theta\theta}\phi_{\gamma}$ is well-defined and positive definite for all $\left(\boldsymbol{z}_0,\theta\right)\in\mathcal{Z}\times\Theta$.

By looking at \eqref{eq:SSLmainMin}, the first summation over labeled samples, i.e. $i\in\mathcal{I}_{\mathrm{l}}$, is again a convex function w.r.t. $\theta$. However, the second summand might not be convex due to the usage of $\softmin$. Therefore, it is sufficient to provide conditions under which $\softmin^{\left(\lambda\right)}_{y\in\mathcal{Y}}\left\{\phi_{\gamma}\right\}$ becomes convex for all $\theta\in\Theta$. This will also prove the convexity of \eqref{eq:SSLmainMin}. Obviously, each $\softmin$ summand in the equation is twice differentiable and hence, for any $\boldsymbol{X}\in\mathcal{X}$, we have
\begin{align}
\nabla^2_{\theta\theta}\left(
\softmin^{\left(\lambda\right)}_{y\in\mathcal{Y}}\left\{
\phi_{\gamma}\left(\boldsymbol{X},y;\theta\right)
\right\}\right)&=
\nabla_{\theta}\left(
\sum_{y\in\mathcal{Y}}\beta_y\left(\theta\right)
\nabla_{\theta}\phi_{\gamma}\left(\boldsymbol{X},y;\theta\right)
\right)
\\
&=\sum_{y\in\mathcal{Y}}\left(
\beta_y\left(\theta\right)
\nabla^2_{\theta\theta}\phi_{\gamma}\left(\boldsymbol{X},y;\theta\right)+
\nabla_{\theta}\beta_y\left(\theta\right)
\nabla^T_{\theta}\phi_{\gamma}\left(\boldsymbol{X},y;\theta\right)
\right),
\nonumber
\end{align}  
where $\beta_y\left(\theta\right)$ (with $0\leq\beta_y\left(\theta\right)\leq1$ for $y\in\mathcal{Y}$ and $\theta\in\Theta$) is defined as
\begin{equation}
\beta_y\left(\theta\right)\triangleq
\frac{
e^{\lambda\phi_{\gamma}\left(\boldsymbol{X},y;\theta\right)}
}{
\sum_{y'\in\mathcal{Y}}
e^{\lambda\phi_{\gamma}\left(\boldsymbol{X},y';\theta\right)}
}
\quad,\mathrm{and~we~have}\quad
\sum_{y\in\mathcal{Y}}\beta_y\left(\theta\right)=1.
\end{equation}
Some mathematical simplifications reveal that
\begin{equation}
\nabla_{\theta}\beta_y\left(\theta\right)=
\lambda
\beta_y\left(\theta\right)\left(1-\beta_y\left(\theta\right)\right)
\nabla_{\theta}\phi_{\gamma}\left(\boldsymbol{X},y;\theta\right),
\end{equation}
and as a result we have the following formula for the Hessian matrix of each $\softmin$ summand:
\begin{align}
\label{eq:corlEq111}
\nabla^2_{\theta\theta}\left(
\softmin^{\left(\lambda\right)}_{y\in\mathcal{Y}}\left\{
\phi_{\gamma}\left(\boldsymbol{X},y;\theta\right)
\right\}\right)=&
\sum_{y\in\mathcal{Y}}
\beta_y\left(\theta\right)
\nabla^2_{\theta\theta}\phi_{\gamma}\left(\boldsymbol{X},y;\theta\right)
\\
+&{\lambda}
\sum_{y\in\mathcal{Y}}
\beta_y\left(\theta\right)\left(1-\beta_y\left(\theta\right)\right)
\nabla_{\theta}\phi_{\gamma}\left(\boldsymbol{X},y;\theta\right)
\nabla_{\theta}^T\phi_{\gamma}\left(\boldsymbol{X},y;\theta\right).
\nonumber
\end{align}
Note that for each $y\in\mathcal{Y}$, the $d\times d$ matrix $\nabla_{\theta}\phi_{\gamma}\left(\boldsymbol{X},y;\theta\right)
\nabla_{\theta}^T\phi_{\gamma}\left(\boldsymbol{X},y;\theta\right)$ is rank-one, positive semi-definite and its only non-zero eigenvalue equals to $\left\Vert\nabla_{\theta}\phi_{\gamma}\left(\boldsymbol{X},y;\theta\right)\right\Vert^2_2\leq\sigma^2$. Therefore, the matrix corresponding to the second summand in the r.h.s. of \eqref{eq:corlEq111} is negative semi-definite only if $\lambda<0$. In this case, i.e. having a negative $\lambda$, the following upper-bound holds for the magnitude of its largest eigenvalue:
\begin{equation}
\leq {\sigma^2}{\left\vert\lambda\right\vert}\max_{\boldsymbol{\beta}\in M\left(\mathcal{Y}\right)}\boldsymbol{\beta}^T\left(\boldsymbol{1} - \boldsymbol{\beta}\right)
=
{\sigma^2}{\left\vert\lambda\right\vert}\left(1-\left\vert\mathcal{Y}\right\vert^{-1}\right).
\label{eq:corlUpperB}
\end{equation} 
On the other hand, the first summand in the r.h.s. of \eqref{eq:corlEq111} is always positive definite and (since $\beta_y\left(\theta\right)$s sum up to $1$) its smallest eigenvalue satisfies the following lower-bound:
\begin{equation}
\ge \inf_{\left(\boldsymbol{z},\theta\right)\in\mathcal{Z}\times\Theta}
\lambda_{\min}\left\{
\nabla^2_{\theta\theta}\phi_{\gamma}\left(\boldsymbol{z}\vert\theta\right)
\right\}.
\label{eq:corlLowerB}
\end{equation}
Therefore, as long as i) $\lambda$ is non-negative, or ii) the upper-bound in \eqref{eq:corlUpperB} is strictly smaller than the lower-bound in \eqref{eq:corlLowerB}, which is the condition of the Theorem on $\lambda$, the Hessian of $\softmin^{\left(\lambda\right)}_{y\in\mathcal{Y}}\left\{\phi_{\gamma}\left(\boldsymbol{X},y;\theta\right)\right\}$ remains positive definite for all $\boldsymbol{z}\in\mathcal{Z}$ and $\theta\in\Theta$, and the proof is complete.

Note that due to assuming strict convexity and twice differentiability for $\ell$, $\nabla^2_{\theta\theta}\phi_{\gamma}$ is universally positive-definite and hence, the r.h.s. of \eqref{eq:convexityGuarant} is negative. This argument is a direct consequence of Danskin's theorem. However, there are no general ways to directly relate eigenvalues of $\nabla^2_{\theta\theta}\ell$ to those of $\nabla^2_{\theta\theta}\phi_{\gamma}$, since such relations extremely depend on the properties of function $\ell$.
\end{proof}


\begin{proof}[Proof of Theorem \ref{thm:generalBound1}]
We prove the Theorem in two steps. In the first step, we show that the empirical value of the proposed semi-supervised adversarial risk, i.e.  $\hat{R}_{\mathrm{SSAR}}\left(\theta;\boldsymbol{D}\right)$, {\it {uniformly}} converges to its expected value all over $\Theta$. In the second step, we use the asymptotic results of Theorem \ref{thm:statGeneral} to finalize the bounds. For the first step, a similar technique to the ones used in classical learning theory, e.g. \cite{mohri2012foundations}, is employed. In this regard, let the random variable $J\left(\boldsymbol{D}\right)$ to be defined as
\begin{equation}
J\left(\boldsymbol{D}\right)~\triangleq~
\sup_{\theta\in\Theta}~
\left\vert
\hat{R}_{\mathrm{SSAR}}\left(\theta;\boldsymbol{D}\right)-
\mathbb{E}_{P_0}\left\{
\hat{R}_{\mathrm{SSAR}}\left(\theta;\boldsymbol{D}\right)\right\}
\right\vert.
\end{equation}
On the other hand, we have $\left\vert\phi_{\gamma}\left(\boldsymbol{z};\theta\right)\right\vert\leq B$, for all $\boldsymbol{z}\in\mathcal{Z}$ and $\theta\in\Theta$. This can be deduced from the definition of adversarial loss $\phi_{\gamma}$ as follows:
\begin{align}
\phi_{\gamma}\left(\boldsymbol{z};\theta\right)&\triangleq
\sup_{\boldsymbol{z}'\in\mathcal{Z}}\ell\left(\boldsymbol{z}';\theta\right)-\gamma c\left(\boldsymbol{z}',\boldsymbol{z}\right)\leq
\sup_{\boldsymbol{z}'\in\mathcal{Z}}\ell\left(\boldsymbol{z}';\theta\right)\leq B,
\nonumber\\
\phi_{\gamma}\left(\boldsymbol{z};\theta\right)&\ge
\ell\left(\boldsymbol{z};\theta\right)-\gamma c\left(\boldsymbol{z},\boldsymbol{z}\right)=
\ell\left(\boldsymbol{z};\theta\right)\ge -B.
\end{align}
Also, note that
\begin{equation}
\left\vert
\softmin^{\left(\lambda\right)}_{y\in\mathcal{Y}}\left\{\phi_{\gamma}\left(\boldsymbol{X},y;\theta\right)\right\}
\right\vert
\leq B,\quad
\forall \lambda\in\mathbb{R}\cup\left\{\pm\infty\right\},
\end{equation}
for all $\boldsymbol{X}\in\mathcal{X}$ and $\theta\in\Theta$. Now, assume the two partially observed data sets $\boldsymbol{D}$ and $\boldsymbol{D}'$, both with size $n$, where the only difference between them is a single data point. Then, it can be readily deduced that
\begin{equation}
\left\vert
\hat{R}_{\mathrm{SSAR}}\left(\theta;\boldsymbol{D}\right)-
\hat{R}_{\mathrm{SSAR}}\left(\theta;\boldsymbol{D}'\right)
\right\vert
\leq
\frac{2B}{n}~~~\Rightarrow~~~
\left\vert
J\left(\boldsymbol{D}\right)-
J\left(\boldsymbol{D}'\right)
\right\vert
\leq
\frac{2B}{n}.
\end{equation}
In this regard, one can use the McDiarmid's inequality and show that: For all $0<\delta\leq 1$, with probability at least $1-\delta$, the following inequality holds:
\begin{equation}
J\left(\boldsymbol{D}\right)
\leq
\mathbb{E}_{P_0}\left\{J\left(\boldsymbol{D}\right)\right\}+
B\sqrt{\frac{2}{n}\log\frac{1}{\delta}},
\end{equation}
which also implies that the following uniform upper-bound exists for all $\theta\in\Theta$:
\begin{equation}
\left\vert
\hat{R}_{\mathrm{SSAR}}\left(\theta;\boldsymbol{D}\right)-
\mathbb{E}_{P_0}\left\{
\hat{R}_{\mathrm{SSAR}}\left(\theta;\boldsymbol{D}\right)\right\}
\right\vert
\leq
\mathbb{E}_{P_0}\left\{J\left(\boldsymbol{D}\right)\right\}+
B\sqrt{\frac{2}{n}\log\frac{1}{\delta}}.
\end{equation}
The term
$\mathbb{E}_{P_0}\left\{J\left(\boldsymbol{D}\right)\right\}$
does not depend on the randomness of the chosen dataset and is a function of the hypothesis set $\mathcal{L}$ (or more precisely, its adversarial counterpart $\Phi$), and distribution $P_0$. It plays the role of {\it {Rademacher complexity}} in classical learning theory. In order to express this term in a more intuitive formulation, first let us introduce the function $f\left(\boldsymbol{z},h;\theta\right)$ for $\boldsymbol{z}=\left(\boldsymbol{X},y\right)$ and $h\in\mathcal{H}\triangleq\left\{\mathrm{l},\mathrm{ul}\right\}$ as follows:
\begin{equation}
f\left(\boldsymbol{z},h;\theta\right)\triangleq
\left\{\begin{array}{lc}
\phi_{\gamma}\left(\boldsymbol{X},y;\theta\right) & h=\mathrm{l}
\\[2mm]
\softmin^{\left(\lambda\right)}_{y\in\mathcal{Y}}\left\{
\phi_{\gamma}\left(\boldsymbol{X},y;\theta\right)
\right\}
&
h=\mathrm{ul}
\end{array}\right.,
\end{equation}
where the rest of parameters are omitted from the input arguments of $f$ for the sake of simplicity in notation. It should be noted that we can write:
\begin{equation}
\mathbb{E}_{\boldsymbol{D}\sim P_0}\left\{\cdot\right\}
=
\mathbb{E}_{h_1,\ldots,h_n\in\mathcal{H}}\left\{
\mathbb{E}_{\boldsymbol{z}_1,\ldots,\boldsymbol{z}_n\sim P_0}\left\{
\cdot
\right\}\right\}~\Rightarrow~
\mathbb{E}_{P_0}\left\{
\hat{R}_{\mathrm{SSAR}}\left(\theta;\boldsymbol{D}\right)
\right\}=
\mathbb{E}_h\left\{
\mathbb{E}_{\boldsymbol{z}}\left\{
f\left(\boldsymbol{z},h;\theta\right)
\right\}
\right\}
\end{equation}
where $h_1,\ldots,h_n$ are i.i.d. bi-categorical random variables in $\mathcal{H}$, with probabilities of $\eta$ and $1-\eta$ for $h=\mathrm{l}$ and $h=\mathrm{ul}$, respectively. Then, Similar to \cite{mohri2012foundations}, one can write the following set of relations:
\begin{align}
\mathbb{E}_{P_0}\left\{J\left(\boldsymbol{D}\right)\right\}&=
\mathbb{E}_{\boldsymbol{D}\sim P_0}\left\{
\sup_{\theta\in\Theta}\left\vert
\hat{R}_{\mathrm{SSAR}}\left(\theta;\boldsymbol{D}\right)-
\mathbb{E}_{\boldsymbol{D}'\sim P_0}\left\{
\hat{R}_{\mathrm{SSAR}}\left(\theta;\boldsymbol{D}'\right)
\right\}
\right\vert
\right\}
\nonumber\\
&=
\mathbb{E}_{\boldsymbol{D}\sim P_0}\left\{
\sup_{\theta\in\Theta}\left\vert
\mathbb{E}_{\boldsymbol{D}'\sim P_0}\left\{
\hat{R}_{\mathrm{SSAR}}\left(\theta;\boldsymbol{D}\right)-
\hat{R}_{\mathrm{SSAR}}\left(\theta;\boldsymbol{D}'\right)
\right\}
\right\vert
\right\}
\nonumber\\
&\leq
\mathbb{E}_{\boldsymbol{D},\boldsymbol{D}'\sim P_0}\left\{
\sup_{\theta\in\Theta}\left\vert
\hat{R}_{\mathrm{SSAR}}\left(\theta;\boldsymbol{D}\right)-
\hat{R}_{\mathrm{SSAR}}\left(\theta;\boldsymbol{D}'\right)
\right\vert
\right\}
\nonumber\\
&=
\mathbb{E}_{\boldsymbol{h}_{1:n},\boldsymbol{h}'_{1:n}\in\mathcal{H}}\left\{
\mathbb{E}_{\boldsymbol{z}_{1:n},\boldsymbol{z}'_{1:n}\sim P_0}\left\{
\sup_{\theta\in\Theta}\left\vert
\frac{1}{n}\sum_{i=1}^{n}
f\left(\boldsymbol{z}_i,h_i;\theta\right)-
f\left(\boldsymbol{z}'_i,h'_i;\theta\right)
\right\vert
\right\}\right\}
\nonumber\\
&=
\mathbb{E}_{\boldsymbol{h}_{1:n},\boldsymbol{h}'_{1:n}\in\mathcal{H}}\left\{
\mathbb{E}_{\boldsymbol{z}_{1:n},\boldsymbol{z}'_{1:n}\sim P_0,~\boldsymbol{\sigma}}\left\{
\sup_{\theta\in\Theta}\left\vert
\frac{1}{n}\sum_{i=1}^{n}\sigma_i\left(
f\left(\boldsymbol{z}_i,h_i;\theta\right)-
f\left(\boldsymbol{z}'_i,h'_i;\theta\right)
\right)
\right\vert
\right\}\right\}
\nonumber\\
&\leq
2\mathbb{E}_{\boldsymbol{h}_{1:n}\in\mathcal{H}}\left\{
\mathbb{E}_{\boldsymbol{z}_{1:n}\sim P_0,~\boldsymbol{\sigma}}\left\{
\sup_{\theta\in\Theta}\left\vert
\frac{1}{n}\sum_{i=1}^{n}\sigma_i
f\left(\boldsymbol{z}_i,h_i;\theta\right)
\right\vert
\right\}\right\},
\end{align}
where $\boldsymbol{\sigma}\in\left\{-1,+1\right\}^n$ represents a vector of $n$ i.i.d. Rademacher random variables. Based on this result and its preceding discussions, one can write:
\begin{align}
\label{eq:RadBreak}
\frac{1}{2}\mathbb{E}_{P_0}\left\{J\left(\boldsymbol{D}\right)\right\}=~
&\eta\mathbb{E}_{\boldsymbol{z}_{1:n\eta}\sim P_0,~\boldsymbol{\sigma}}\left\{
\sup_{\theta\in\Theta}\frac{1}{n\eta}\sum_{i=1}^{n\eta}
\sigma_i\phi_{\gamma}\left(\boldsymbol{z}_i;\theta\right)
\right\}
\\
+&\left(1-\eta\right)\mathbb{E}_{\boldsymbol{X}_{1:n\left(1-\eta\right)}\sim P_{0_{\boldsymbol{X}}},~\boldsymbol{\sigma}}\left\{
\sup_{\theta\in\Theta}\frac{1}{n\left(1-\eta\right)}\sum_{i=1}^{n\left(1-\eta\right)}
\sigma_i\softmin^{\left(\lambda\right)}_{y\in\mathcal{Y}}
\left\{
\phi_{\gamma}\left(\boldsymbol{X}_i,y;\theta\right)
\right\}
\right\}.
\nonumber
\end{align}
The first term in the r.h.s. of \eqref{eq:RadBreak} can be more analytically investigated. In order to do so, let us define the $\epsilon$-neighborhood around $\boldsymbol{z}_0$ as 
$\mathcal{N}_{\epsilon}\left(\boldsymbol{z}_0\right)\triangleq\left\{\boldsymbol{z}\in\mathcal{Z}\vert c\left(\boldsymbol{z},\boldsymbol{z}_0\right)\leq\epsilon\right\}$, for $\epsilon\ge0$. Then, there exists $\epsilon\ge0$ such that
\begin{align}
\mathbb{E}_{\boldsymbol{z}_{1:n}\sim P_0,~\boldsymbol{\sigma}}\left\{
\sup_{\theta\in\Theta}\frac{1}{n}\sum_{i=1}^{n}
\sigma_i\phi_{\gamma}\left(\boldsymbol{z}_i;\theta\right)
\right\}&=
\mathbb{E}_{\boldsymbol{z}_{1:n}\sim P_0,~\boldsymbol{\sigma}}\left\{
\sup_{\theta\in\Theta}\frac{1}{n}\sum_{i=1}^{n}
\sigma_i
\sup_{\boldsymbol{z}'_i\in\mathcal{Z}}
\ell\left(\boldsymbol{z}'_i;\theta\right)-\gamma c\left(\boldsymbol{z}'_i,\boldsymbol{z}_i\right)
\right\}
\nonumber\\
&=
\mathbb{E}_{\boldsymbol{z}_{1:n}\sim P_0,~\boldsymbol{\sigma}}\left\{
\sup_{\theta\in\Theta}\frac{1}{n}\sum_{i=1}^{n}
\sigma_i
\left[
\sup_{\boldsymbol{z}'_i\in\mathcal{N}_{\epsilon}\left(\boldsymbol{z}_i\right)}
\ell\left(\boldsymbol{z}'_i;\theta\right)-\gamma\epsilon
\right]
\right\}
\nonumber\\
&=
\mathbb{E}_{\boldsymbol{z}_{1:n}\sim P_0,~\boldsymbol{\sigma}}\left\{
\sup_{\theta\in\Theta}\frac{1}{n}\sum_{i=1}^{n}
\sigma_i
\sup_{\boldsymbol{z}'_i\in\mathcal{N}_{\epsilon}\left(\boldsymbol{z}_i\right)}
\ell\left(\boldsymbol{z}'_i;\theta\right)
\right\}
\nonumber\\
&=g_{\mathrm{l}}\left(n\right),
\end{align}
where $g_{\mathrm{l}}\left(n\right)$ can be found in Definition \ref{def:SSM_main}, with the function set $\mathcal{F}$ representing the loss function set $\mathcal{L}$ in the above relations. For the second term on the r.h.s. of \eqref{eq:RadBreak}, the following inequality holds for all $\lambda\in\mathbb{R}\cup\left\{\pm\infty\right\}$:
\begin{align}
&\mathbb{E}_{\boldsymbol{X}_{1:n},\ldots,\boldsymbol{X}_n\sim P_{0_{\boldsymbol{X}}},~\boldsymbol{\sigma}}\left\{
\sup_{\theta\in\Theta}\frac{1}{n}\sum_{i=1}^{n}
\sigma_i\softmin^{\left(\lambda\right)}_{y\in\mathcal{Y}}
\left\{
\phi_{\gamma}\left(\boldsymbol{X}_i,y;\theta\right)
\right\}
\right\}
\nonumber\\
\leq~&
\mathbb{E}_{\boldsymbol{X}_{1:n}\sim P_{0_{\boldsymbol{X}}},~\boldsymbol{\sigma}}\left\{
\left(
\Pi_{y\in\mathcal{Y}}
\sup_{\theta_y\in\Theta}
\right)
\frac{1}{n}\sum_{i=1}^{n}
\sigma_i\softmin^{\left(\lambda\right)}_{y\in\mathcal{Y}}
\left\{
\phi_{\gamma}\left(\boldsymbol{X}_i,y;\theta_y\right)
\right\}
\right\}
\nonumber\\
\leq~&
\sum_{y\in\mathcal{Y}}
\mathbb{E}_{\boldsymbol{z}_{1:n}\sim\left(P_{0_{\boldsymbol{X}}}\delta_y\right),~\boldsymbol{\sigma}}
\left\{
\sup_{\theta\in\Theta}\frac{1}{n}\sum_{i=1}^{n}\sigma_i \sup_{\boldsymbol{z}'_i\in\mathcal{N}_{\epsilon}\left(\boldsymbol{z}_i\right)}\ell\left(\boldsymbol{z}'_i;\theta\right)
\right\}=g_{\mathrm{ul}}\left(n\right).
\end{align}
The last two inequalities above are the results of Lemma \ref{lemma:RadSumSoft} (see below), and Definition \ref{def:SSM_main}, respectively. The following lemma helps us to resolve the presence of $\softmin$ operator in the formulation of $\mathbb{E}_{P_0}\left\{J\left(\boldsymbol{D}\right)\right\}$.
\\[-3mm]
\begin{lemma}
\label{lemma:RadSumSoft}
Assume the function sets
$\mathcal{F}_j\subseteq\mathbb{R}^{\mathcal{Z}},~j=1,\ldots,d$, where $\mathcal{Z}$ denotes a vector domain and $d\in\mathbb{N}$. Also, assume $\boldsymbol{\sigma}=\left(\sigma_1,\ldots,\sigma_d\right)$ to be a vector of i.i.d. Rademacher variables, and $\boldsymbol{Z}=\left\{\boldsymbol{z}_1,\ldots,\boldsymbol{z}_n\right\}$ are i.i.d. generated data points in domain $\mathcal{Z}$, according to some probability measure. Then, the following upper-bound holds for all $\lambda\in\mathbb{R}\cup\left\{\pm\infty\right\}$:
\begin{equation}
\mathbb{E}_{\boldsymbol{Z},\boldsymbol{\sigma}}\left\{
\left(\prod_{j=1}^{d}\sup_{f_j\in\mathcal{F}_j}\right)~\frac{1}{n}\sum_{i=1}^{n}
\sigma_i\softmin^{\left(\lambda\right)}_{j=1,\ldots,d}\left(
f_j\left(\boldsymbol{z}_i;\theta\right)
\right)
\right\}
\leq
\sum_{j=1}^{d}\mathcal{R}_n\left(\mathcal{F}_j\right),
\end{equation}
where $\mathcal{R}_n\left(\cdot\right)$ denotes the $n$-point expected Rademacher complexity w.r.t. to the same distribution that generates the samples in $\boldsymbol{Z}$. 
\end{lemma}
\begin{proof}
Looking at the definition of $\softmin$ in \eqref{eq:phiGammaDef}, first let us consider the following function: For $a,b\in\mathbb{R}$ and non-negative parameters $\alpha$ and $\beta$, with $\alpha+\beta=1$, define
\begin{equation}
H_{\lambda,\alpha,\beta}\left(a,b\right)
\triangleq
\frac{1}{\lambda}
\log\left(
\alpha e^{\lambda a}+
\beta e^{\lambda b}
\right).
\end{equation}
Then, the following relations hold:
\begin{align}
H_{\lambda,\alpha,\beta}\left(a,b\right) &= 
a + \frac{1}{\lambda}\log\left(\alpha+\beta e^{\lambda\left(b-a\right)}\right)
\nonumber\\
&=b + \frac{1}{\lambda}\log\left(\beta+\alpha e^{\lambda\left(a-b\right)}\right)
\end{align}
and, as a result
\begin{align}
H_{\lambda,\alpha,\beta}\left(a,b\right) &= 
{\frac{a}{2}+\frac{b}{2}}+\frac{1}{2\lambda}\left[
\log\left(\alpha + \beta e^{\lambda \left(b-a\right)}\right) +
\log\left(\beta + \alpha e^{\lambda \left(a-b\right)}\right)
\right]
\nonumber \\
&\triangleq {\frac{a+b}{2}}+h_{\lambda,\alpha,\beta}\left(b-a\right),
\end{align}
where the last equality is in fact the definition of  $h_{\lambda,\alpha,\beta}:\mathbb{R}\rightarrow\mathbb{R}$. It should be noted that $h_{\lambda,\alpha,\beta}\left(0\right)=0$. Also, the following holds for the derivative of $h_{\lambda,\alpha,\beta}\left(\cdot\right)$:
\begin{equation}
h'_{\lambda,\alpha,\beta}\left(u\right)=\frac{
\beta^2 e^{\lambda u} - \alpha^2e^{-\lambda u}
}{
2\alpha\beta + \beta^2 e^{\lambda u} + \alpha^2e^{-\lambda u}
}=
\frac{
\beta e^{\left(\lambda u\right)/2} - \alpha e^{\left(-\lambda u\right)/2}
}{
\beta e^{\left(\lambda u\right)/2} + \alpha e^{\left(-\lambda u\right)/2}
},
\end{equation}
which indicates $\left\vert h'_{\lambda,\alpha,\beta}\left(u\right)\right\vert\leq 1$, for all $u\in\mathbb{R}$ and the legitimate set of parameters $\left(\lambda,\alpha,\beta\right)$. Therefore, $h'_{\lambda,\alpha,\beta}$ is a 1-Lipschitz continuous function. In this regard, for any two real-valued function sets $\mathcal{A}$ and $\mathcal{B}$ whose domain is $\mathcal{Z}$, the following relation holds due to the {\it {sum inequality}} of Rademacher complexity:
\begin{align}
\mathbb{E}_{\boldsymbol{Z},\boldsymbol{\sigma}}\left\{
\sup_{a\in\mathcal{A},~b\in\mathcal{B}}~\frac{1}{n}\sum_{i=1}^{n}
\sigma_i H_{\lambda,\alpha,\beta}\left(
a\left(\boldsymbol{z}_i\right),
b\left(\boldsymbol{z}_i\right)\right)
\right\}&=
\mathcal{R}_n\left(\left\{
\frac{a+b}{2}+\frac{1}{2}h_{\lambda,\alpha,\beta}\left(b-a\right)
\bigg\vert
~a\in\mathcal{A},~b\in\mathcal{B}
\right\}
\right)
\nonumber \\
&\leq
\frac{1}{2}\left[\mathcal{R}_n\left(\mathcal{A}\right)+
\mathcal{R}_n\left(\mathcal{B}\right)+
\mathcal{R}_n\left(h_{\lambda,\alpha,\beta}\circ\mathcal{C}\right)
\right],
\end{align}
where $\mathcal{C}\triangleq\left\{a-b\vert~a\in\mathcal{A},~b\in\mathcal{B}\right\}$. It can be readily verified that $\mathcal{R}_n\left(\mathcal{C}\right)\leq\mathcal{R}_n\left(\mathcal{A}\right)+\mathcal{R}_n\left(\mathcal{B}\right)$. Also, {\it {Talagrand's contraction lemma}} in statistical learning theory \cite{mohri2012foundations} states that given the above properties for a 1-Lipschitz function $h_{\lambda,\alpha,\beta}\left(\cdot\right)$, we have $\mathcal{R}_n\left(h_{\lambda,\alpha,\beta}\circ\mathcal{C}\right)\leq\mathcal{R}_n\left(\mathcal{C}\right)$. Therefore, the previous chain of inequalities can be concluded as
\begin{equation}
\mathcal{R}_n\left(
H_{\lambda,\alpha,\beta}\left(a,b\right)\big\vert~
a\in\mathcal{A},b\in\mathcal{B}
\right)
\leq\mathcal{R}_n\left(\mathcal{A}\right)+\mathcal{R}_n\left(\mathcal{B}\right).
\label{eq:Hradineq}
\end{equation}
for all $\lambda\in\mathbb{R}\cup\left\{\pm\infty\right\}$, and all $\alpha,\beta\ge0$ with $\alpha+\beta=1$. For the remainder of the proof, one should consider the following recursive relation for all $f_1,\ldots,f_d$:
\begin{equation}
\softmin^{\left(\lambda\right)}_{j=1,\ldots,d}\left(f_j\right)=
H_{\lambda,\frac{d-1}{d},\frac{1}{d}}\left(
\softmin^{\left(\lambda\right)}_{j=1,\ldots,d-1}\left(f_j\right),
f_d
\right),
\end{equation}
which can be verified through a simple substitution of parameters. By using \eqref{eq:Hradineq}, we have
\begin{align}
\mathcal{R}_n\left(
\softmin^{\left(\lambda\right)}_{j=1,\ldots,d}\left(f_j\right)
\bigg\vert
f_j\in\mathcal{F}_j
\right)
&\leq
\mathcal{R}_n\left(
\softmin^{\left(\lambda\right)}_{j=1,\ldots,d-1}\left(f_j\right)
\bigg\vert
f_j\in\mathcal{F}_j
\right)+
\mathcal{R}_n\left(\mathcal{F}_d\right).
\end{align}
Repeating the above inequality for $d$ consecutive times gives us the desired result and completes the proof.
\end{proof}
According to Definition \ref{def:SSM_main}, the previous upper-bounds can be simplified into the following statement: With probability at least $1-\delta$, and for all $\theta\in\Theta$, we have
\begin{equation}
\left\vert
\hat{R}_{\mathrm{SSAR}}\left(\theta;\boldsymbol{D}\right)-
\mathbb{E}_{P_0}
\left\{\hat{R}_{\mathrm{SSAR}}\left(\theta;\boldsymbol{D}\right)\right\}
\right\vert
\leq
2\left(\mathcal{R}^{\left(\mathrm{SSM}\right)}_{n,\left(\epsilon,\eta\right)}\left(\mathcal{L}\right)+
B\sqrt{\frac{\log\frac{1}{\delta}}{2n}}\right),
\label{eq:combine2}
\end{equation}
where $\epsilon\ge0$ is the dual counterpart of $\gamma$ in \eqref{eq:SSLmainMin}. Therefore, the empirical values of $R_{\mathrm{SSAR}}$ are always close (and asymptotically convergent) to their corresponding expected values. Next, we have to show that the expected value of $R_{\mathrm{SSAR}}$ legitimately upper-bounds the true risk at the solution point, i.e. $\theta^*\in\Theta$.

Let $\theta^*_{\mathrm{true}}$ to represent the true minimizer of the expected adversarial risk, i.e. $\theta^*_{\mathrm{true}}\triangleq\argmin_{\theta\in\Theta}\mathbb{E}_{P_0}\left\{\phi_{\gamma}\left(\boldsymbol{Z};\theta\right)\right\}$. Then, based on Theorem \ref{thm:statGeneral} and for any $\zeta\ge0$, there exists a neighborhood around $\theta^*_{\mathrm{true}}$, denoted by $\Theta_{\mathrm{local}}\subset\Theta$, such that the following gap is guaranteed to exist for all $\theta\notin\Theta_{\mathrm{local}}$:
\begin{equation}
\mathbb{E}_{P_0}\left\{
\hat{R}_{\mathrm{SSAR}}\left(\theta;\boldsymbol{D}\right)-
\hat{R}_{\mathrm{SSAR}}\left(\theta^*_{\mathrm{true}};\boldsymbol{D}\right)
\right\}\ge\zeta,
\end{equation}
given that the condition $\eta\ge\mathrm{MSR}_{\left(\Phi,P_0\right)}\left(\lambda,\zeta\right)$ is satisfied. According to the assumption on $\eta$ in the current theorem, it can be readily deduced that with probability at least $1-\delta$, the following relation holds for all $\theta\notin\Theta_{\mathrm{local}}$:
\begin{equation}
\hat{R}_{\mathrm{SSAR}}\left(\theta;\boldsymbol{D}\right)-
\hat{R}_{\mathrm{SSAR}}\left(\theta^*_{\mathrm{true}};\boldsymbol{D}\right)
>0
~~~\Rightarrow~~~
\theta^*\triangleq\argmin_{\theta\in\Theta}\hat{R}_{\mathrm{SSAR}}\left(\theta;\boldsymbol{D}\right)\in\Theta_{\mathrm{local}},
\label{eq:combine3}
\end{equation}
i.e. the minimizer of $\hat{R}_{\mathrm{SSAR}}\left(\theta;\boldsymbol{D}\right)$ also falls in $\Theta_{\mathrm{local}}$. Also, for all $\theta\in\Theta_{\mathrm{local}}$ and any $\epsilon\ge0$ we have
\begin{equation}
\mathbb{E}_{P_0}\left\{
\hat{R}_{\mathrm{SSAR}}\left(\theta;\boldsymbol{D}\right)
\right\}
~\ge~
\mathbb{E}_{P_0}\left\{
\phi_{\gamma}\left(\boldsymbol{Z};\theta\right)
\right\}+\gamma\epsilon
~\ge~
\sup_{P\in\mathcal{B}_{\epsilon}\left(P_0\right)}\mathbb{E}_{P}\left\{\ell\left(\boldsymbol{Z};\theta\right)\right\}.
\label{eq:combine4}
\end{equation}
Combining relations given in \eqref{eq:combine2}, \eqref{eq:combine3} and \eqref{eq:combine4} gives the desired result and completes the proof.
\end{proof}

\section{Auxiliary Lemmas and Proofs}
\label{sec:appendix:lemma}


\begin{proof}[Proof of Lemma \ref{lemma:softminLip}]
For simplicity, let us consider the following change of notation: for a fixed $\boldsymbol{X}\in\mathcal{X}$ and $\lambda\in\mathbb{R}$, define:
\begin{equation}
f\left(\theta\right)
\triangleq
\softmin_{y\in\mathcal{Y}}^{\left(\lambda\right)}\left\{
\phi_{\gamma}\left(\boldsymbol{X},y;\theta\right)
\right\},
\end{equation}
where $\boldsymbol{X}$ and $\lambda$ are hidden from $f$. Then, based on the definition of $\softmin$, it can be easily verified that we have the following formulation for $\nabla_{\theta}f$:
\begin{equation}
\nabla_{\theta}f=
\sum_{y\in\mathcal{Y}}\beta_{y}\left(\theta\right)
\nabla_{\theta}\phi_{\gamma}\left(\boldsymbol{X},y;\theta\right)
\quad\mathrm{with}\quad
\beta_{y}\left(\theta\right)\triangleq
\frac{
e^{\lambda\phi_{\gamma}\left(\boldsymbol{X},y;\theta\right)}
}{
\sum_{y'\in\mathcal{Y}}
e^{\lambda\phi_{\gamma}\left(\boldsymbol{X},y';\theta\right)}},~
y\in\mathcal{Y},
\end{equation}
where $\sum_{y\in\mathcal{Y}}\beta_y\left(\theta\right)=1$, for all $\theta\in\Theta$. Hence, the following inequalities hold:
\begin{align}
\left\Vert
\nabla_{\theta}f\left(\theta\right)-
\nabla_{\theta}f\left(\theta'\right)
\right\Vert_*&=
\left\Vert
\sum_{y\in\mathcal{Y}}\beta_y\left(\theta\right)
\nabla_{\theta}\phi_{\gamma}\left(\boldsymbol{X},y;\theta\right)-
\sum_{y\in\mathcal{Y}}\beta_y\left(\theta'\right)
\nabla_{\theta}\phi_{\gamma}\left(\boldsymbol{X},y;\theta'\right)
\right\Vert_*
\nonumber \\
&
\leq
\sum_{y\in\mathcal{Y}}\beta_y\left(\theta\right)\left\Vert
\nabla_{\theta}\phi_{\gamma}\left(\boldsymbol{X},y;\theta\right)-\nabla_{\theta}\phi_{\gamma}\left(\boldsymbol{X},y;\theta'\right)
\right\Vert_*
\nonumber\\
&~~~+
\sum_{y\in\mathcal{Y}}\left\Vert
\nabla_{\theta}\phi_{\gamma}\left(\boldsymbol{X},y;\theta'\right)
\right\Vert_*\left\vert
\beta_y\left(\theta\right)-\beta_y\left(\theta'\right)
\right\vert
\nonumber \\
&\leq
\sum_{y\in\mathcal{Y}}\beta_y\left(\theta\right)
\left(L_{\theta\theta}+\frac{L_{\boldsymbol{z}\theta}L_{\theta\boldsymbol{z}}}{\gamma-L_{\boldsymbol{z}\boldsymbol{z}}}\right)\left\Vert
\theta-\theta'
\right\Vert+
\sigma\omega\left\vert\mathcal{Y}\right\vert
\left\Vert
\theta-\theta'
\right\Vert
\nonumber\\
&=\left(L_{\theta\theta}+\frac{L_{\boldsymbol{z}\theta}L_{\theta\boldsymbol{z}}}{\gamma-L_{\boldsymbol{z}\boldsymbol{z}}}+\sigma\omega\left\vert\mathcal{Y}\right\vert\right)\left\Vert
\theta-\theta'
\right\Vert,
\end{align}
where $\omega$ denotes the Lipschitz constant of $\beta_y\left(\theta\right)$ w.r.t. $\theta$, for all $y\in\mathcal{Y}$. The last inequality is a direct consequence of assuming $\left\Vert\nabla_{\theta}\phi_{\gamma}\left(\boldsymbol{X},y;\theta\right)\right\Vert_*\leq\sigma$, which can be validated through the following mathematical argument: There exists $\epsilon\ge0$, such that
\begin{align}
\left\Vert\nabla_{\theta}\phi_{\gamma}\left(\boldsymbol{X},y;\theta\right)\right\Vert_*
&=
\left\Vert\nabla_{\theta}\left(\sup_{\boldsymbol{z}'\in\mathcal{Z}}\ell\left(\boldsymbol{X},y;\theta\right)
-\gamma c\left(\boldsymbol{z}',\left(\boldsymbol{X},y\right)\right)\right)\right\Vert_*
\nonumber \\
&=
\left\Vert\nabla_{\theta}
\ell\left(\argmax_{\boldsymbol{z}'\in\mathcal{Z}}~\ell\left(\boldsymbol{z}';\theta\right)-\gamma c\left(\boldsymbol{z}',\left(\boldsymbol{X},y\right)\right);\theta\right)\right\Vert_*
\leq
\sigma,
\label{eq:infcompact1}
\end{align}
where the last inequality is due to the assumption of Lemma \ref{lemma:supervisedLip} under an appropriate choice of norm. The middle equality in \eqref{eq:infcompact1} is the result of the extended Danskin's theorem which relaxes convexity into {\it {inf-compactness}} of function $\ell$. For proof of inf-compactness of $\ell$ and the consequent properties, see Section $4$ of \cite{bonnans2013perturbation}. In order to assess $\omega$, which is an indicator of smoothness for $\beta_y\left(\theta\right)$, one can take advantage of the {\it {Mean Value Theorem}} \cite{miller1993introduction}, as follows:
\begin{equation}
\left\vert
\beta_y\left(\theta\right)
-\beta_y\left(\theta'\right)
\right\vert
\leq
\max_{y\in\mathcal{Y}}
\sup_{\theta^*\in\mathcal{T}\left(\theta\rightarrow\theta'\right)}
\left\Vert
\nabla_{\theta}\beta_y\left(\theta^*\right)
\right\Vert
\left\Vert
\theta-\theta'
\right\Vert,\quad
\theta,\theta'\in\Theta,
\label{eq:betaLipEq1}
\end{equation}
where $\mathcal{T}\left(\theta\rightarrow\theta'\right)$ is the set of all continuous paths from $\theta$ to $\theta'$ that entirely lie in $\Theta$. It is not hard to verify that the gradient $\nabla_{\theta}\beta_y\left(\theta\right)$ has the following formulation:
\begin{equation}
\nabla_{\theta}\beta_y\left(\theta\right)=
\lambda{\beta_y\left(\theta\right)}
\sum_{y'\in\mathcal{Y}}
\beta_{y'}\left(\theta\right)\left(
\nabla_{\theta}\phi_{\gamma}\left(\boldsymbol{X},y;\theta\right)
-
\nabla_{\theta}\phi_{\gamma}\left(\boldsymbol{X},y';\theta\right)
\right),
\quad
\left(y,\theta\right)\in\mathcal{Y}\times\Theta,
\end{equation}
and hence satisfies the subsequent inequalities:
\begin{align}
\left\Vert\nabla_{\theta}\beta_y\left(\theta\right)\right\Vert
\leq
{2}{\left\vert\lambda\right\vert}\sum_{y'\in\mathcal{Y}}
\beta_{y'}\left(\theta\right)\max_{h\in\left\{y,y'\right\}}\left\{
\left\Vert
\nabla_{\theta}\phi_{\gamma}\left(\boldsymbol{X},h\vert\theta\right)
\right\Vert
\right\}
\leq
{2\sigma}{\left\vert\lambda\right\vert},\quad\forall\theta\in\Theta.
\label{eq:betaLipEq2}
\end{align}
Combining \eqref{eq:betaLipEq1} with \eqref{eq:betaLipEq2} provides us with the safe choice of $\omega=2\sigma\left\vert\lambda\right\vert$. Therefore, $\nabla_{\theta}f$ is $\left(L_{\theta\theta}+\frac{L_{\boldsymbol{z}\theta}L_{\theta\boldsymbol{z}}}{\gamma-L_{\boldsymbol{z}\boldsymbol{z}}}+{2\sigma^2}{\left\vert\lambda\right\vert}\left\vert\mathcal{Y}\right\vert\right)$-Lipschitz w.r.t. $\theta$, and the proof is complete.
\end{proof}
\begin{proof}[Proof of Lemma \ref{lemma:sinhaDeltaApprox}]
The proof is simple and directly results from the assumptions. According to the differentiablility of $\phi_{\gamma}$ w.r.t. $\theta$ which is a consequence of an extended version of Danskin's theorem (see Lemma \ref{lemma:softminLip}), the following relations hold:
\begin{align}
\left\Vert
\nabla_{\theta}
\phi_{\gamma}\left(\boldsymbol{z}_0;\theta\right)
-
\nabla_{\theta}
\ell\left(\hat{\boldsymbol{z}}^*;\theta\right)
\right\Vert_*
&=
\left\Vert
\nabla_{\theta}
\ell\left(\argmax_{\boldsymbol{z}'\in\mathcal{Z}}~\ell\left(\boldsymbol{z}';\theta\right)-
\gamma c\left(\boldsymbol{z}',\boldsymbol{z}_0\right)
;\theta\right)
-
\nabla_{\theta}
\ell\left(\hat{\boldsymbol{z}}^*;\theta\right)
\right\Vert_*
\nonumber\\
&\leq~
L_{\theta z}\left\Vert
\hat{\boldsymbol{z}}^*-
\argmax_{\boldsymbol{z}'\in\mathcal{Z}}\left(\ell\left(\boldsymbol{z}';\theta\right)
-\gamma c\left(\boldsymbol{z}',\boldsymbol{z}_0\right)
\right)
\right\Vert.
\label{eq:lemmaAux1010}
\end{align}
On the other hand, due to $\left(\gamma-L_{zz}\right)$-strict-concavity of \eqref{eq:SSLDROinnermax}, a $\delta$-approximation maximizer, i.e. $\hat{\boldsymbol{z}}^*$, satisfies
\begin{equation}
\left\Vert
\hat{\boldsymbol{z}}^*-
\argmax_{\boldsymbol{z}'\in\mathcal{Z}}\left(\ell\left(\boldsymbol{z}';\theta\right)
-\gamma c\left(\boldsymbol{z}',\boldsymbol{z}_0\right)
\right)
\right\Vert^2\leq
\frac{L_{z\theta}}{L_{\theta z}\left(\gamma-L_{zz}\right)}.
\end{equation} 
Substituting the above into \eqref{eq:lemmaAux1010} completes the proof.
\end{proof}
\begin{lemma}
\label{lemma:distFreeBounds}
Assume a feature-label space $\mathcal{Z}=\mathcal{X}\times\mathcal{Y}$ and a function class $\mathcal{F}\subseteq\mathbb{R}^{\mathcal{Z}}$, for a feature space $\mathcal{X}$ and a finite label set $\mathcal{Y}$. Also, assume there exists $\Delta:\mathbb{N}\rightarrow\mathbb{R}$, such that
$\mathcal{R}_{n}\left(\mathcal{F}\right)\leq \Delta\left(n\right)$, for all $n\in\mathbb{N}$ and any data distribution $P_0\in M\left(\mathcal{Z}\right)$. Then, the following holds:
\begin{equation}
\mathcal{R}_{n,\left(\epsilon,\eta\right)}^{\left(\mathrm{SSM}\right)}\left(
\mathcal{F}
\right)\leq
\eta \Delta\left(\lceil\eta n\rceil\right)
+
\left(1-\eta\right)\left\vert\mathcal{Y}\right\vert
\Delta\left(\lceil\left(1-\eta\right)n\rceil\right)
\end{equation}
for all distributions in $M\left(\mathcal{Z}\right)$, any $\epsilon\ge0$ and $\eta\in\left[0,1\right]$.
\end{lemma}
\begin{proof}
According to the assumption, $\Delta\left(n\right)$ is an upper-bound for Rademacher complexity of $\mathcal{F}$, regardless of the probability measure that generates the data samples. Therefore, one can write
\begin{equation}
\sup_{P_0\in M\left(\mathcal{Z}\right)}
\mathbb{E}_{\boldsymbol{z}_{1:n}\sim P_0,\boldsymbol{\sigma}}\left\{
\sup_{f\in\mathcal{F}}~\frac{1}{n}\sum_{i=1}^{n}\sigma_i f\left(\boldsymbol{z}_i\right)
\right\}
=
\sup_{\boldsymbol{z}_{1:n}\in\mathcal{Z}}
\mathbb{E}_{\boldsymbol{\sigma}}\left\{
\sup_{f\in\mathcal{F}}~\frac{1}{n}\sum_{i=1}^{n}\sigma_i f\left(\boldsymbol{z}_i\right)
\right\}
\leq
\Delta\left(n\right).
\end{equation}
In this regard, the following relations hold for the function $g_{\mathrm{l}}\left(n\right)$ of Definition \ref{def:SSM_main}:
\begin{align}
g_{\mathrm{l}}\left(n\right)&=
\mathbb{E}_{\boldsymbol{z}_{1:n}\sim P_0,\boldsymbol{\sigma}}\left\{
\sup_{f\in\mathcal{F}}~\frac{1}{n}\sum_{i=1}^{n}\sigma_i\left[
\sup_{a\in\mathcal{A}_{\epsilon}}~
f\left(a\left(\boldsymbol{z}_i\right)\right)
\right]
\right\}
\nonumber\\
&\leq
\mathbb{E}_{\boldsymbol{z}_{1:n}\sim P_0,\boldsymbol{\sigma}}\left\{
\sup_{\boldsymbol{z}'_{1:n}\in\mathcal{Z}\vert~c\left(\boldsymbol{z}_i,\boldsymbol{z}'_i\right)\leq\epsilon}~
\sup_{f\in\mathcal{F}}~\frac{1}{n}\sum_{i=1}^{n}\sigma_i\left[
f\left(\boldsymbol{z}'_i\right)
\right]
\right\}
\nonumber\\
&\leq
\sup_{\boldsymbol{z}'_{1:n}\in\mathcal{Z}}
\mathbb{E}_{\boldsymbol{\sigma}}\left\{
\sup_{f\in\mathcal{F}}~\frac{1}{n}\sum_{i=1}^{n}\sigma_i f\left(\boldsymbol{z}'_i\right)
\right\}
\leq
\Delta\left(n\right).
\end{align}
With some very similar mathematical arguments, one can easily show that $g_{\mathrm{ul}}\left(n\right)\leq\left\vert\mathcal{Y}\right\vert\Delta\left(n\right)$. Therefore, for any distribution $P_0$, any $\epsilon\ge0$ and any $\eta\in\left[0,1\right]$, we always have
\begin{align}
\mathcal{R}^{\left(\mathrm{SSM}\right)}_{n,\left(\epsilon,\eta\right)}
&\triangleq
\eta g_{\mathrm{l}}\left(\lceil n\eta\rceil\right)+
\left(1-\eta\right) g_{\mathrm{ul}}
\left(\lceil n\left(1-\eta\right)\rceil\right)
\nonumber\\
&\leq
\Delta\left(\lceil n\eta\rceil\right)+
\left(1-\eta\right)\left\vert\mathcal{Y}\right\vert \Delta
\left(\lceil n\left(1-\eta\right)\rceil\right).
\end{align}
and the proof is complete.

In particular, assume a $0$-$1$ loss function set $\mathcal{L}=\left\{\ell\left(\cdot;\theta\right)\vert~\theta\in\Theta\right\}$, where $\Theta$ denotes the parameter space of a classifier with a finite VC-dimension of $\mathrm{dim}\left(\Theta\right)$. Then, due to Dudley's entropy bound and Haussler's upper-bound \cite{mohri2012foundations}, there exists constant $C$ such that
\begin{equation}
\Delta\left(n\right)=
C\sqrt{\frac{\mathrm{dim}\left(\Theta\right)}{n}}
\end{equation}
is a valid upper-bound on the Rademacher complexity of $\mathcal{F}$ regardless of $P_0$. Then, one can write
\begin{align}
\mathcal{R}^{\left(\mathrm{SSM}\right)}_{n,\left(\epsilon,\eta\right)}
&\leq
\Delta\left(\lceil n\eta\rceil\right)+
\left(1-\eta\right)\left\vert\mathcal{Y}\right\vert \Delta
\left(\lceil n\left(1-\eta\right)\rceil\right)
\nonumber\\
&=
C\left[
\eta \sqrt{\frac{\mathrm{dim}\left(\Theta\right)}{\lceil n\eta\rceil}} +
\left(1-\eta\right)\left\vert\mathcal{Y}\right\vert
\sqrt{\frac{\mathrm{dim}\left(\Theta\right)}{\lceil n\left(1-\eta\right)\rceil}}
\right]
\nonumber\\
&\leq
C\left[
\eta \sqrt{\frac{\mathrm{dim}\left(\Theta\right)}{n\eta}} +
\left(1-\eta\right)\left\vert\mathcal{Y}\right\vert
\sqrt{\frac{\mathrm{dim}\left(\Theta\right)}{n\left(1-\eta\right)}}
\right]
\nonumber\\
&=
C\sqrt{\frac{\mathrm{dim}\left(\Theta\right)}{n}}
\left[
\sqrt{\eta} +
\left\vert\mathcal{Y}\right\vert
\sqrt{1-\eta}
\right].
\end{align}
This will also prove the claim on SSM Rademacher complexity in Section \ref{sec:proposed:general}.
\end{proof}

\end{document}